\documentclass[11pt]{article}

\usepackage{comment,url,algorithm,algorithmic,graphicx,subcaption,relsize}
\usepackage{amssymb,amsfonts,amsmath,amsthm,amscd,dsfont,mathrsfs,mathtools,nicefrac}
\usepackage{float,psfrag,epsfig,color,xcolor,url,hyperref}
\usepackage{epstopdf,bbm,mathtools,enumitem}
\usepackage[toc,page]{appendix}
\usepackage[mathscr]{euscript}

\usepackage[top=1in, bottom=1in, left=1in, right=1in]{geometry}

\def\balign#1\ealign{\begin{align}#1\end{align}}
\def\baligns#1\ealigns{\begin{align*}#1\end{align*}}
\def\balignat#1\ealign{\begin{alignat}#1\end{alignat}}
\def\balignats#1\ealigns{\begin{alignat*}#1\end{alignat*}}
\def\bitemize#1\eitemize{\begin{itemize}#1\end{itemize}}
\def\benumerate#1\eenumerate{\begin{enumerate}#1\end{enumerate}}

\newenvironment{talign*}
 {\csname align*\endcsname}
 {\endalign}
\newenvironment{talign}
 {\csname align\endcsname}
 {\endalign}

\def\balignst#1\ealignst{\begin{talign*}#1\end{talign*}}
\def\balignt#1\ealignt{\begin{talign}#1\end{talign}}

\let\originalleft\left
\let\originalright\right
\renewcommand{\left}{\mathopen{}\mathclose\bgroup\originalleft}
\renewcommand{\right}{\aftergroup\egroup\originalright}

\def\tinycitep*#1{{\tiny\citep*{#1}}}
\def\tinycitealt*#1{{\tiny\citealt*{#1}}}
\def\tinycite*#1{{\tiny\cite*{#1}}}
\def\smallcitep*#1{{\scriptsize\citep*{#1}}}
\def\smallcitealt*#1{{\scriptsize\citealt*{#1}}}
\def\smallcite*#1{{\scriptsize\cite*{#1}}}

\def\R{\mathbb{R}}

\def\<{\left\langle} %
\def\>{\right\rangle}

\def\E{\mbb{E}} %

\DeclareMathOperator{\Tr}{Tr} %

\DeclareSymbolFont{rsfs}{U}{rsfs}{m}{n}
\DeclareSymbolFontAlphabet{\mathscrsfs}{rsfs}

\ifdefined\nonewproofenvironments\else
\ifdefined\ispres\else
\newtheorem{theorem}{Theorem}
\newtheorem{lemma}[theorem]{Lemma}
\newtheorem{corollary}[theorem]{Corollary}

\renewenvironment{proof}{\noindent\textbf{Proof.}\hspace*{.3em}}{\qed\\}
\newenvironment{proof-sketch}{\noindent\textbf{Proof Sketch}
  \hspace*{1em}}{\qed\bigskip\\}
\newenvironment{proof-idea}{\noindent\textbf{Proof Idea}
  \hspace*{1em}}{\qed\bigskip\\}
\newenvironment{proof-of-lemma}[1][{}]{\noindent\textbf{Proof of Lemma {#1}}
  \hspace*{1em}}{\qed\\}
\newenvironment{proof-of-theorem}[1][{}]{\noindent\textbf{Proof of Theorem {#1}}
  \hspace*{1em}}{\qed\\}
\newenvironment{proof-attempt}{\noindent\textbf{Proof Attempt}
  \hspace*{1em}}{\qed\bigskip\\}

\fi

\newtheorem{proposition}[theorem]{Proposition}

\newtheorem{assumption}{Assumption}
\fi
\makeatletter
\@addtoreset{equation}{section}
\makeatother

\hypersetup{
  colorlinks,
  linkcolor={red!50!black},
  citecolor={blue!50!black},
  urlcolor={blue!80!black}
}

\mathtoolsset{showonlyrefs}

\usepackage{mathrsfs}
\allowdisplaybreaks[3]

\def\Ltwo{\mathbb L_2}
\def\wass{{\sf W}}
\def\law{{\mathcal L}}
\def\rmd{{\,\rm d}}
\def\l|{\left\lVert}
\def\r|{\right\rVert}
\def\E{\mathbb E}
\def\R{\mathbb R}
\newcommand{\inprod}[2]{\ensuremath{\left\langle #1 , \, #2 \right\rangle}}

\definecolor{darkmidnightblue}{rgb}{0.0, 0.2, 0.4}
\definecolor{darkpowderblue}{rgb}{0.0, 0.2, 0.6}
\definecolor{dukeblue}{rgb}{0.0, 0.0, 0.61}

\hypersetup{
    colorlinks = true,
    citecolor= midnightblue,
    urlcolor= black,
    breaklinks=true,
    linkcolor = midnightblue,
    linkbordercolor = {white},
}

\definecolor{darkmidnightblue}{HTML}{003366}    
\definecolor{midnightblue}{HTML}{0059b3}
\definecolor{chromered}{HTML}{f14233}

\begin{document}

\title{Advancing Wasserstein Convergence Analysis of Score-Based Models: Insights from Discretization and Second-Order Acceleration}

 \author{
 Yifeng Yu\thanks{Department of Mathematical Sciences, Tsinghua University \texttt{yyf22@mails.tsinghua.edu.cn}}
 \and 
 Lu Yu\thanks{
  Department of Data Science,
  City University of Hong Kong \texttt{lu.yu@cityu.edu.hk}
 }
}

\maketitle

\begin{abstract}
Score-based diffusion models have emerged as powerful tools in generative modeling, yet their theoretical foundations remain underexplored. 
In this work, we focus on the Wasserstein convergence analysis of score-based diffusion models. Specifically, we investigate the impact of various discretization schemes, including Euler discretization, exponential integrators, and midpoint randomization methods. 
Our analysis provides a quantitative comparison of these discrete approximations, emphasizing their influence on convergence behavior. 
Furthermore, we explore scenarios where Hessian information is available and propose an accelerated sampler based on the local linearization method. 
We demonstrate that this Hessian-based approach achieves faster convergence rates of order 
$\widetilde{\mathcal{O}}\left(\frac{1}{\varepsilon}\right)$ significantly improving upon the standard 
rate $\widetilde{\mathcal{O}}\left(\frac{1}{\varepsilon^2}\right)$ of vanilla diffusion models, where $\varepsilon$ denotes the target accuracy.

\end{abstract}

\section{Introduction}

\textit{Diffusion models} have become a pivotal framework in modern generative modeling, achieving notable success across fields such as image generation~\cite{ramesh2022hierarchical,ho2020denoising,song2020denoising,dhariwal2021diffusion}, natural language processing~\cite{popov2021grad}, and computational biology~\cite{xu2022geodiff,anand2022protein}. These models operate by systematically introducing noise into data through a forward diffusion process and then learning to reverse this process, effectively reconstructing data from noise. This approach enables them to capture the underlying structure of complex, high-dimensional data distributions. %
For a detailed review of diffusion models, we refer the readers to~\cite{yang2023diffusion,tang2024score,chen2024overview}.

A widely adopted formulation of diffusion models is the score-based generative model (SGM), implemented using stochastic differential equations (SDEs)~\cite{song2020score}. Broadly speaking, SGMs rely on two key stochastic processes: a forward process and a backward process. The forward process gradually transforms samples from the data distribution into pure noise, while the backward process reverses this transformation, converting noise back into the target data distribution, thereby enabling generative modeling. 

Despite the remarkable empirical success of diffusion models across various applications, their theoretical understanding remains limited. In recent years, there has been a rapidly expanding body of research on the convergence theory of diffusion models. Broadly, these contributions can be divided into two main approaches, each focusing on different metrics and divergences.
The first category investigates convergence bounds based on $\alpha$-divergence, including the Kullback–Leibler (KL) divergence and the total variation (TV) distance (see e.g.,~\cite{li2024provable,liang2025low,chen2024probability,chen2022sampling,wu2024stochastic,chen2023improved}). 
Among these works, several explore acceleration techniques that leverage higher-order information about the log density (see e.g.,~\cite{li2024accelerating,liang2024broadening,huang2024convergence}).
The second category focuses on convergence bounds in Wasserstein distance, which is often considered more practical and informative for estimation tasks.
One line of work within this category assumes strong log-concavity of the data distribution and access to accurate estimates of the score function~\cite{gao2023wasserstein,gao2023wasserstein,bruno2023diffusion,tang2024contractive,strasman2024analysis}. Another line of work focuses on specific structural assumptions of the data distribution.
For example,~\cite{de2022convergence} establishes Wasserstein-1 convergence with exponential rates under the manifold hypothesis, assuming the data distribution lies on a lower-dimensional manifold or represents an empirical distribution.~\cite{mimikos2024score} provides the Wasserstein-1 convergence analysis when the data distribution is defined on a torus.
Furthermore, recent work~\cite{gentiloni2025beyond} analyzes Wasserstein-2 convergence in SGMs while relaxing log-concavity and score regularity assumptions.

Much of the existing literature on the convergence theory of diffusion models relies on the Euler discretization method. 
Notably,~\cite{chen2023improved} compare the behavior of Euler discretization and exponential integrators~\cite{zhang2022fast,hochbruck2010exponential} in terms of KL divergence.
Additionally,~\cite{de2022convergence} provide a comparative analysis of these two schemes, though without formal theoretical guarantees.
A comprehensive and systematic understanding of how different discretization schemes influence convergence performance in diffusion models remains underexplored. 
Furthermore, while convergence analyses of accelerated diffusion models primarily focus on TV or KL distances, studies investigating Wasserstein convergence for these accelerations remain lacking.

In this work, we address these challenges by analyzing the Wasserstein convergence of score-based diffusion models when the data distribution has a smooth and strongly log-concave density.
Specifically, we investigate the impact of different discretization schemes on convergence behavior. 
Beyond the widely used Euler method and exponential integrator, we explore the midpoint randomization method.
This method was initially introduced in~\cite{shen2019randomized} for discretizing kinetic Langevin diffusion~\cite{cheng2018underdamped} and then has been extensively studied in log-concave sampling complexity theory~\cite{he2020ergodicity,yu2023langevin,yu2024parallelized,yu2024log,kandasamy2024poisson}.
It was later applied to diffusion models~\cite{gupta2024faster,li2024improved}, showing improved KL and TV convergence performance over vanilla models and offering easy parallelization.

We also consider scenarios where accurate estimates of the Hessian of the log density are accessible. Inspired by~\cite{Shoji1998}, we propose a novel sampler based on the local linearization method, which leverages second-order information about the log density. Our analysis shows that this approach significantly improves the upper bounds on the Wasserstein distance between the target data distribution and the generative distribution of the diffusion model.

Our contribution can be summarized as follows.
\begin{itemize}
\setlength\itemsep{0.02em}
    \item In Section~\ref{sec:discretization}, we establish convergence guarantees for SGMs in the Wasserstein-2 distance under various discretization methods, including the Euler method, exponential integrators, the midpoint randomization method, and a hybrid approach combining the latter two. 
    \item In Section~\ref{sec:accleration}, we introduce a novel Hessian-based accelerated sampler for the stochastic diffusion process, leveraging the local linearization method. 
    We then establish its Wasserstein convergence analysis in Theorem~\ref{thm:2order}, achieving state-of-the-art order of
$\widetilde{\mathcal{O}}\left(\frac{1}{\varepsilon}\right)$.
\end{itemize}
In summary, our analysis provides a quantitative comparison of different discrete approximations in terms of the Wasserstein-2 distance, offering practical guidance for choosing discretization points.
Moreover, we present the first Wasserstein convergence analysis of an accelerated sampler that leverages accurate score function estimation and second-order information about log-densities.
This accelerated sampler achieves a faster convergence rate~$\widetilde{\mathcal{O}}\left(1/{\varepsilon}\right)$ in Wasserstein-2 distance, compared to the standard rate~$\widetilde{\mathcal{O}}\left(1/{\varepsilon^2}\right)$ of vanilla diffusion models. 
These results contribute to the understanding of Wasserstein convergence in score-based models, shedding light on aspects that have not been extensively explored before.

\vspace{0.2cm}
\noindent \textbf{More Related Work.}
Score-based diffusion models can be formulated using either SDEs or their deterministic counterparts, known as probability flow ODEs~\cite{song2020score}. While SDE-based samplers generate samples through stochastic simulation, ODE-based samplers provide a deterministic alternative. Theoretical advancements in accelerating these samplers have emerged only recently.
A significant step toward designing provably accelerated, training-free methods were made by~\cite{li2024accelerating}, who propose and analyze acceleration for both ODE- and SDE-based samplers. Their accelerated SDE sampler leverages higher-order expansions of the conditional density to enhance efficiency. This was followed by the work of~\cite{li2024sharp}, which provided convergence guarantees for probability flow ODEs.
Furthermore,~\cite{huang2024convergence} studies the convergence properties of deterministic samplers based on probability flow ODEs, using the Runge-Kutta integrator;~\cite{wu2024stochastic} propose and analyze a training-free acceleration algorithm for SDE-based samplers, based on the stochastic Runge-Kutta method.
~\cite{liang2024broadening} proposes a novel accelerated SDE-based sampler when Hessian information is available.
Another line of work involves the midpoint randomized method. 
In particular, ~\cite{gupta2024faster} explore ODE acceleration by incorporating a randomized midpoint method, leveraging its advantages in parallel computation. 
A more recent work by~\cite{li2024improved} improved upon the ODE sampler proposed by~\cite{gupta2024faster}, achieving the state-of-the-art convergence rate.

We note that all of these works provide convergence analysis in terms of either KL divergence or TV distance.
Among these, \cite{liang2024broadening} accelerates the stochastic DDPM sampler by leveraging precise score and Hessian estimations of the log density, even for possibly non-smooth target distributions. This is achieved through a novel Bayesian approach based on tilting factor representation and Tweedie’s formula.
\cite{huang2024convergence} accelerates the ODE sampler by utilizing 
$p$-th ($p\geqslant 1$) order information of the score function, with the target distribution supported on a compact set and employing early stopping. 
These two works are the most similar to our proposed accelerated sampler in that they all rely on the Hessian information of the log density.
However, their settings differ from ours, and their convergence analyses are neither directly applicable to our framework nor precisely expressed in terms of Wasserstein distance.\footnote{When the target distribution is compactly supported, Pinsker's inequality allows translating TV or KL divergence into Wasserstein distance. However, this often yields loose bounds, especially in high dimensions, where the actual Wasserstein distance may be much smaller.}.

\vspace{0.2cm}
\noindent \textbf{Outline.}
The rest of the paper is organized as follows.
In Section~\ref{sec:setting}, we introduce the framework of the score-based diffusion model and present the different discretization schemes.
In Section~\ref{sec:discretization}, we establish the convergence rate of the diffusion model under various discretization schemes. 
In Section~\ref{sec:accleration}, we propose a new sampler that leverages Hessian estimations and provide its convergence rate.
Section~\ref{sec:simulation} presents numerical studies that validate our theoretical results. 
Finally, in Section~\ref{sec:discussion}, we conclude with a discussion and outline future research directions. 
All proofs are provided in the Appendix.

\vspace{0.2cm}
\noindent \textbf{Notation.} Denote the $d$-dimensional Euclidean space by $\mathbb{R}^d$.
Denote the $d$-dimensional identity matrix by $I_d$. The gradient and the Hessian of a function $f:\mathbb{R}^d\to\mathbb{R}$ are denoted by $\nabla f$ and $\nabla^2f$. 
Given any pair of measures $\mu$ and $\nu$, the Wasserstein-2 distance between $\mu$ and $\nu$ is defined as
\begin{align*}
    \wass_2(\mu,\nu)=\left(\inf_{\varrho\in\Gamma(\mu,\nu)}\int_{\mathbb{R}^d\times\mathbb{R}^d}\l|x-y\r|^2\rmd \varrho(x,y)\right)^{1/2},
\end{align*}
where the infimum is taken over all joint distributions $\varrho$ that have $\mu$ and $\nu$ as marginals. For two symmetric $d\times d$ matrices $A$ and $B$, we use $A\preccurlyeq B$ or $B\succcurlyeq A$ to denote the relation that $B-A$ is positive semi-definite.
For any random object $X$, we use $\mathcal{L}(X)$ to denote its law.
Given a random vector $X\in\R^d$, define $\|X\|_{\Ltwo}=\sqrt{\E[\|X\|^2]},$ where $\|\cdot\|$ denotes the Euclidean norm.
For a matrix $A\in\R^{d\times d},$ we define the  $\l|A\r|_F$ as its Frobenius norm.

\section{Background and Our Setting}
\label{sec:setting}
\textbf{Framework.} \quad We consider the forward process
{
\begin{equation}
\label{eq:forward0}
dX_t= f(X_t,t)dt+g(X_t,t)dB_t\,,
\end{equation}}
where the initial point $X_0\sim  p_0$ follows the data distribution, and $B_t$ denotes the standard $d-$dimensional Brownian motion.
Here, the drift $f: \mathbb{R}^d\times \mathbb{R}_+\to \mathbb{R}^d$ and the function $g:\mathbb{R}^d\times \mathbb{R}_+\to \mathbb{R}^{d\times d}$ are diffusion parameters.
Some conditions are necessary to ensure that the SDE~\eqref{eq:forward0} is well-defined. 
In practice, various choices for the pair $(f, g)$ are employed, depending on the specific needs of the model; for a detailed survey, we refer to~\cite{tang2024score}.
For clarity, we adopt the simplest possible choice in this work by setting $f(X_t,t)=-X_t/2$ and $g(X_t,t)=1$.
This results in the Ornstein-Uhlenbeck process, which is described by the following SDE:
\begin{equation}
    \rmd X_t=-\frac{1}{2}X_t\rmd t+\rmd B_t,
    \label{eq:forward}
\end{equation}
The forward process~\eqref{eq:forward} is run until a sufficiently large time $T > 0$,  at which point the corrupted marginal distribution of $X_T$, denoted by $p_T$, is expected to approximate the standard Gaussian distribution.
Then, diffusion models generate new data by reversing the SDE~\eqref{eq:forward}, which leads to the following backward SDE
\begin{equation}
\rmd X_t^{\leftarrow} = \frac{1}{2}\left(X_t^{\leftarrow} + 2\nabla \log p_{T-t}(X_t^{\leftarrow})\right) \rmd t + \rmd W_t\,,
\label{eq:backward}
\end{equation}
where $X^\leftarrow_0\sim p_T$, and the term $\nabla \log p_t$, referred to as the \textit{score function} for $p_t$, is represented by the gradient of the log density function of $p_t$. 
Additionally, $W_t$ denotes another standard Brownian motion independent of $B_t$.
Under mild conditions, when initialized at
$X^\leftarrow_0\sim p_T$, the backward process $\{X^\leftarrow_t\}_{0\leqslant t\leqslant T}$ has the same distribution as the forward process $\{X_{T-t}\}_{0\leqslant t\leqslant T}$~\cite{anderson1982reverse,cattiaux2023time}. 
As a result, running the reverse diffusion $X^{\leftarrow}_t$ from $t=0$ to $T$ will generate a sample from the target data distribution $p_0$.
Note that the density $p_T$ is unknown, we approximate it using the distribution
\begin{align*}
\hat p_T= \mathcal{N}(\mathbf{0},(1-e^{-T})I_d)
\end{align*}
as proposed in~\cite{gao2023wasserstein}. Therefore, we derive a reverse diffusion process defined by
\begin{align}
    \label{eq:Yt}
    \rmd Y_t=\dfrac{1}{2}(Y_t+2\nabla\log p_{T-t}(Y_t))\rmd t+\rmd W_t,\quad Y_0\sim \hat{p}_T\,.
\end{align}

\vspace{0.2cm}
\noindent \textbf{Score Matching.} \quad Another  challenge in working with~\eqref{eq:backward} is that the score function $\nabla \log p_t$ is unknown, as the distribution $p_t$ is not explicitly available.
In practice, rather than using the exact score function $\nabla\log p_{T-t}$, approximate estimates for it are learned from the data by training neural networks on a score-matching objective~\cite{hyvarinen2005estimation,vincent2011connection,song2020sliced}. This objective is given by
\begin{align*}
\underset{\theta\in\Theta}{\text{minimize}}~~~ \E[\|s_\theta(t,X_t)-\nabla \log p_t(X_t)\|^2]\,,
\end{align*}
where $\{s_\theta:\theta\in\Theta\}$ is a sufficiently rich function class, such as that of neural network.
Substituting the learned score estimate $s_*$ into the backward process~\eqref{eq:backward}, we obtain the following practical continuous-time backward SDE,
\begin{equation}
dX^{\leftarrow}_t= \frac{1}{2}\big(X^{\leftarrow}_t+ 2s_{*}(T-t,X^{\leftarrow}_t)\big)\rmd t + \rmd W_t\,.
\label{eq:backward1}
\end{equation}
Since this continuous backward SDE cannot be simulated exactly, it is typically approximated using discretization methods.

\vspace{0.2cm}
\noindent \textbf{Discretization Schemes.}\quad In the following, we outline the four discretization methods considered in this work for solving the practical reverse SDE~\eqref{eq:backward1}. 
Let $h>0$ be the step size.
Without loss of generality, we assume $T = Nh,$
where $N$ is a positive integer.
For simplicity, we denote $\frac{1}{2}X^{\leftarrow}_t+ s_{*}(T-t,X^{\leftarrow}_t)$ by $\gamma(T-t,X^{\leftarrow}_t)$, and define 
\begin{align*}
\Delta_hW_t:=W_{t+h}-W_t,\quad %
\bar{\Delta}_hW_t:=\int_t^{t+h}e^{\frac{t+h-s}{2}}\rmd W_s\,.
\end{align*}
\noindent $\bullet$~~\textsc{Euler-Maruyama scheme}:\quad
Given the step size $h$, the following approximation holds
\begin{align*}
X^{\leftarrow}_{t+h}
&=X^{\leftarrow}_t+\int_0^h \gamma(T-(t+v),X^{\leftarrow}_{t+v})\rmd v+\Delta_hW_t \nonumber \\
&\approx X_t^{\leftarrow}+ h\gamma(T-t,X^{\leftarrow}_t)+\Delta_hW_t\,.
\end{align*}
We derive the following discretized process for $n=0,\dots,N-1$:
\begin{align*}
\vartheta^{\sf EM}_{n+1}=(1+h/2)\vartheta^{\sf EM}_{n}+hs_{*}(T-nh,\vartheta^{\sf EM}_n)+\sqrt{h}\xi_n\,,
\end{align*}
where $\vartheta^{\sf EM}_{0}\sim \hat p_T$ and $\xi_n\sim \mathcal{N}(0,I_{d}).$

\noindent $\bullet$~~\textsc{Exponential Integrator}:\quad
Inspired by the work~\cite{hochbruck2010exponential},~\cite{zhang2022fast} propose a more refined discretization method which solves the backward SDE~\eqref{eq:backward} explicitly, yielding the following approximation 
\begin{align*}
X^{\leftarrow}_{t+h}&=e^\frac{h}{2}X^{\leftarrow}_t+\int_0^h e^{\frac{h-v}{2}}   s_*(T-t-v,X^{\leftarrow}_{t-v})\rmd v+\bar{\Delta}_hW_t\\
&\approx e^{\frac{h}{2}}X^{\leftarrow}_t+ 2(e^{\frac{h}{2}}-1) s_*(T-t,X^{\leftarrow}_t)
+\bar{\Delta}_hW_t\,.
\end{align*}
We derive the following discretized process for $n=0,\dots,N-1$:
\begin{align*}
\vartheta^{\sf EI}_{n+1}&=e^{\frac{h}{2}}\vartheta^{\sf EI}_{n}
+2(e^{\frac{h}{2}}-1)s_*(T-nh,\vartheta^{\sf EI}_n)+\sqrt{e^h-1}\xi_n
\end{align*}
with the initial point $\vartheta^{\sf EI}_{0}\sim \hat p_T$ and $\xi_n\sim \mathcal{N}(0,I_{d}).$

\noindent $\bullet$~~\textsc{Vanilla Midpoint Randomization}:\quad
Unlike the Euler method, the midpoint randomization method evaluates the function $\gamma(T-t,X_t^{\leftarrow})$ at a random point within the time interval $[0,h]$ rather than at the start.
Let $U$ be a random variable uniformly distributed in $[0, 1]$ and independent of the Brownian motion $W_t$. The randomized midpoint method exploits the approximation
\begin{align}
X^{\leftarrow}_{t+h}
&=X^{\leftarrow}_t+\int_0^h \gamma(T-t-v, X^{\leftarrow}_{t+v})\rmd v+\Delta_hW_t\nonumber\\
&\approx X_t^{\leftarrow}+ h \gamma(T-t-hU,X^{\leftarrow}_{t+hU})
+ \Delta_hW_t\,.
\label{eq:mr}
\end{align}
The idea behind the randomized midpoint method is to introduce an $U$ in $[0,1]$,
making $h\gamma(T-t-hU,X^{\leftarrow}_{t+hU})$ an estimator for integral $\int_0^h\gamma(T-t-v,X^{\leftarrow}_{t+v})\rmd v$. 

Furthermore, the intermediate term $X^{\leftarrow}_{t+hU}$ is generated by employing the Euler method.
We then derive the following discretized process for $n=0,\dots,N-1$:

\vspace{0.2cm}
\textbf{Step 1}\label{noise}~~Generate $\xi'_n,\xi''_n\sim \mathcal{N}(\mathbf{0},I_d)$ and $U_n\sim\text{Unif}\,[0,1]$. 
Set $\xi_n=\sqrt{U_n}\xi'_n+\sqrt{1-U_n}\xi''_n$.

\vspace{0.2cm}
\textbf{Step 2}~~With the initialization $\vartheta^{\sf REM}_{0}\sim\hat p_T$, define
\begin{align*}
\vartheta^{\sf REM}_{n+U}&=
\vartheta^{\sf REM}_{n}
+hU_n\gamma(T-nh,\vartheta^{\sf REM}_{n})
 +{\sqrt{hU_n}\xi'_n}\\
\vartheta^{\sf REM}_{n+1}&=\vartheta^{\sf REM}_{n}
+h\gamma(T-(n+U_n)h, \vartheta^{\sf REM}_{n+U})
 +{\sqrt{h}\xi_n}\,.
\end{align*}
\noindent $\bullet$~~\textsc{Exponential Integrator with Randomized Midpoint Method:} \quad Combining midpoint randomization with the exponential integrator approach, we propose the following new discretization process
for $n=1,\dots,N-1$:

\vspace{0.2cm}
\textbf{Step 1}~~Generate $\xi'_n,\xi''_n\sim \mathcal{N}(\mathbf{0},I_d)$ and $U_n\sim\text{Unif}\,[0,1]$. 
Set 
\begin{align*}
 \rho_n= e^{\frac{h(1+U_n)}{2}}\big(1-e^{-hU_n}\big)\Big[{(e^{hU_n}-1)(e^h-1)}  \Big]^{-1/2}
\end{align*} 
and $\xi_n=\rho_n\xi'_n+\sqrt{1-\rho_n^2}\xi''_n$.

\vspace{0.2cm}
\textbf{Step 2}~~With the initialization $\vartheta^{\sf REI}_{0}\sim\hat p_T$, define
\begin{align*}
\vartheta^{\sf REI}_{n+U}&=e^{hU_n/2}\vartheta^{\sf REI}_{n}
+2(e^{hU_n/2}-1)s_*(T-nh,\vartheta^{\sf REI}_{n}) +\sqrt{e^{hU_n}-1}\xi'_n\\
\vartheta^{\sf REI}_{n+1}&=e^{h/2}\vartheta^{\sf REI}_{n}
+he^{(1-U_n)h/2}s_*(T-(n+U_n)h,\vartheta^{\sf REI}_{n+U}) +\sqrt{e^h-1}\xi_n\,.
\end{align*}
The resulting discrete process is then solved to generate new samples that approximately follow the data distribution $p_0$.

To simplify notation and improve clarity,
we denote $\vartheta_n$ as the sample points generated by discretization methods after $n$ iterations throughout the remainder of the paper.
For each discretization scheme, we construct a sample path $\{\vartheta_n\}_{n=0}^N$ through the iterative procedure
\begin{align*}
\vartheta_{n+1}=\mathcal{G}\big(h,\vartheta_n,\{W_t\}_{nh\leqslant t\leqslant (n+1)h}\big)
\end{align*}
with the initial point $\vartheta_0\sim\hat{p}_T$.
Here, $\mathcal{G}:\R_+\times\R^d\times C([0,T],\R^d)\to\mathbb{R}^d$ denotes the discrete transition operator parameterized by the step size $h$, mapping the current state and the Brownian path segment to the next state.
To distinguish the implementations of $\mathcal{G}$ across different discretization methods, we augment the notation with a superscript $\alpha\in\{\sf EM,\sf EI,\sf REM,\sf REI\}$, then the iterative law becomes
\begin{align*}
    \vartheta_{n+1}^{\alpha}=\mathcal{G}^{\alpha}(h,\vartheta_n^{\alpha},\{W_t\}_{nh\leqslant t\leqslant (n+1)h})\,.
\end{align*}

\section{Wasserstein Convergence Analysis under Various Discretization Schemes}
\label{sec:discretization}
In this section, we study the convergence of the diffusion model under various discretization schemes applied to the continuous backward SDE~\eqref{eq:backward}. These methods include the Euler method (EM), the exponential integrator (EI), the vanilla randomized midpoint method (REM), and the novel approach combining midpoint randomization with the exponential integrator (REI) described in the previous section.
Specifically, we establish the upper bounds on the Wasserstein-2 distance between the distribution of the $N$-th output of the SGMs under these discretization schemes and the target distribution
\begin{align*}
\wass_2(\mathcal{L}(\vartheta_N^{\alpha}),p_0), \quad \alpha\in\{\sf{EM,EI,REM,REI}\}\,.    
\end{align*}
Additionally, we analyze the number of iterations $N$ required for the Wasserstein distance to achieve a pre-specified error level $\varepsilon$ under different discretization schemes.
For clarity of presentation, we omit constants throughout the paper, retaining only the key components that influence the convergence rate in theory. However, our proofs include constants in the bounds.

To establish the convergence analysis, we require the following assumptions.
We first state our assumptions on the data distribution $p_0$.
\begin{assumption}
    \label{asm:p0scLipx}
    The target density $p_0$ is $m_0$-strongly log-concave, and the score function $\nabla\log p_0$ is $L_0$-Lipschitz. 
\end{assumption}
Under Assumption~\ref{asm:p0scLipx}, $p_t$ is $m(t)$-strongly log-concave, $\nabla\log p_t$ is $L(t)$-Lipschitz.
Moreover, $m(t)$ is lower bounded by $m_{\min}=\min(1,m_0)$, and $L(t)$ is upper bounded by $L_{\max}=1+L_0$.
These results are summarized in  Lemma~\ref{lem:GaoLt} and~\ref{lem:Gaomt} in the Appendix.

We also assume that the score function $\nabla\log p_t(x)$ exhibits linear growth in $\|x\|$, as stated below.
\begin{assumption}
    \label{asm:scLipt}
    There exists a constant $M_1>0$ such that for $n=0,1,\dots,N-1,$ 
    \begin{align*}
    \sup\limits_{nh\leqslant t,s\leqslant (n+1)h}
        \l|\nabla\log p_{T-t}(x)-\nabla\log p_{T-s}(x)\r|\leqslant M_1h(1+\l|x\r|)\,.
    \end{align*}
\end{assumption}
The above condition is a relaxation of the standard Lipschitz condition on the score function.

Recall that $\vartheta_n$ represents the sample generated by discretization for solving the practical backward SDE~\eqref{eq:backward1} after $n$ iterations. 
We require the following assumption on the score-matching approximation at each point $\vartheta_n$.
\begin{assumption}
    \label{asm:scoreerr}
Given a small $\varepsilon_{sc}>0,$ the score estimator satisfies
    \begin{align*}
        \sup\limits_{0\leqslant n\leqslant N}\l|\nabla\log p_{T-nh}(\vartheta_n)-s_*(T-nh,\vartheta_n)\r|_{\Ltwo}\leqslant \varepsilon_{sc}.
    \end{align*}
\end{assumption}
Assumption~\ref{asm:p0scLipx},~\ref{asm:scLipt} and~\ref{asm:scoreerr} are standard in the Wasserstein convergence analysis of the score-based diffusion model. These assumptions were previously adopted in \cite{Gao2023WassersteinCG,gao2024convergence} and can be easily verified in the Gaussian case.

\subsection{Euler-Maruyama Method}
Recall that the EM discretization for the backward process~\eqref{eq:backward1} is defined as follows
\begin{align*}
 \vartheta^{\sf EM}_{n+1}=(1+h/2)\vartheta^{\sf EM}_{n}+hs_{*}(T-nh,\vartheta^{\sf EM}_n)+\sqrt{h}\xi_n\,,
\end{align*}
where $\xi_n$ are i.i.d. standard Gaussian vectors and $n=0,1,\dots,N-1$.
In the following theorem, we quantify the Wasserstein distance between the distribution of $ \vartheta_{N}^{\sf EM}$ and the target distribution~$p_0.$
\begin{theorem}
    \label{thm:EM}
    Suppose that Assumptions~\ref{asm:p0scLipx},~\ref{asm:scoreerr} and~\ref{asm:scLipt} hold, it holds that
    \begin{align}
    \label{eq:Wassthm1}
    \wass_2(\law(\vartheta_N^{\sf EM}),p_0)
        \lesssim e^{-m_{\min}T}\l|X_0\r|_{\Ltwo}+\mathscr C_1^{\sf EM}\sqrt{dh}+\mathscr C_2^{\sf EM}\varepsilon_{sc}\,,
    \end{align}
    where 
    \begin{align*}
    \mathscr C_1^{\sf EM}=\dfrac{L_{\max}+1/2}{m_{\min}-1/2}~~\text{ and }~~
    \mathscr C_2^{\sf EM}=\dfrac{1}{m_{\min}-1/2}
    \end{align*}
    with $m_{\min}=\min(1,m_0)$ and  $L_{\max}=1+L_0$.
\end{theorem}
Prior to discussing the relation of the above bound to those available in the literature, let us state a consequence of it.
\begin{corollary}
\label{cor:EM}
For a given $\varepsilon>0$, the Wasserstein distance satisfies $\wass_2(\law(\vartheta_N^{\sf EM}),p_0)<\varepsilon$ after 
    \begin{align*}
    N=\mathcal{O}\bigg(\frac{1}{\varepsilon^2}\log\Big(\frac{1}{\varepsilon}\Big)\bigg)
    \end{align*}
    iterations, provided that
$T=\mathcal{O}\left(\log(\frac{1}{\varepsilon})\right)$ and $h=\mathcal{O}(\varepsilon^2)$.
\end{corollary}
We now provide an explanation of the upper bound obtained in Theorem~\ref{thm:EM}. The total error arises from three sources: initialization error, discretization error, and score-matching error.

\begin{itemize}
    \item The \textbf{first term} in the upper bound of \eqref{eq:Wassthm1} characterizes the behavior of the continuous-time reverse process $Y_t$ (defined in \eqref{eq:Yt}) as it converges to the distribution $p_0$, assuming no discretization or score-matching errors. Specifically, it bounds the error $\wass_2(\law(Y_T), p_0)$, which results from initializing the reverse SDE at $\hat{p}_T$ instead of $p_T$. 
    \item The \textbf{second term} in the upper bound of \eqref{eq:Wassthm1} accounts for discretization errors introduced when approximating the continuous process using the Euler algorithm.  
    \item The \textbf{third term} quantifies the error introduced by score matching.  
\end{itemize}
The convergence rates obtained in Theorem~\ref{thm:EM} and Corollary~\ref{cor:EM} align with Theorem 5 and Proposition 7 in~\cite{gao2023wasserstein}, demonstrating consistency between our results and their theoretical conclusions.

\subsection{Exponential Integrator}
As described in the previous section, the exponential integrator scheme leverages the semi-linear structure by discretizing only the nonlinear term while preserving the continuous dynamics of the linear term. 
The resulting discretized process for $n=0,\dots,N-1$ is:
\begin{align*}
\vartheta^{\sf EI}_{n+1}&=e^{\frac{h}{2}}\vartheta^{\sf EI}_{n}
+2(e^{\frac{h}{2}}-1)s_*(T-nh,\vartheta^{\sf EI}_n)+\sqrt{e^h-1}\xi_n\,.
\end{align*}
In the following theorem, we provide the Wasserstein distance between the distribution of the $N$-th iterate~$\vartheta_N^{\sf EI}$ and the target distribution~$p_0$.
\begin{theorem}
    \label{thm:EI}
    Suppose that Assumptions~\ref{asm:p0scLipx},~\ref{asm:scLipt} and~\ref{asm:scoreerr} hold, then
    \begin{align*}
        \wass_2(\law(\vartheta_N^{\sf EI}),p_0)\lesssim e^{-m_{\min}T}\l|X_0\r|_{\Ltwo}+\mathscr C_1^{\sf EI}\sqrt{dh}+\mathscr C_2^{\sf EI}\varepsilon_{sc}\,,
    \end{align*}
    where 
    \begin{align*}
     \mathscr C_1^{\sf EI}=\dfrac{L_{\max}}{m_{\min}-1/2}~~\text{ and }~~
     \mathscr C_2^{\sf EI}=\dfrac{1}{m_{\min}-1/2}
    \end{align*}
    with $L_{\max}$ and $m_{\min}$ as defined in Theorem~\ref{thm:EM}.
\end{theorem}
Comparing the error bound obtained in this theorem with that in Theorem~\ref{thm:EM}, we find that the convergence rate is comparable to that of the Euler discretization.
This observation is consistent with the error bounds for {\sf EI} and 
{\sf EM} schemes in KL divergence established in Theorem 1 of~\cite{chen2023improved}, where $N=\Omega(T^2)$ in their setting.

\subsection{Vanilla Randomized Midpoint Method}
\label{sec:REM}
The randomized midpoint method for the practical backward SDE is defined as
\begin{align*}
\vartheta^{\sf REM}_{n+U}&=(1+hU_n/{2})\vartheta^{\sf REM}_{n}
+hU_ns_*(T-nh,\vartheta^{\sf REM}_{n}) +\sqrt{2hU_n}\xi'_n\\
\vartheta^{\sf REM}_{n+1}&=\vartheta^{\sf REM}_{n}
+h\vartheta^{\sf REM}_{n+U}/2+hs_*(T-(n+U_n)h,\vartheta^{\sf REM}_{n+U})+\sqrt{2h}\xi_n\,,
\end{align*}
where 
$\xi'_n$ and $\xi_n$ are as defined in
Step 1 of the vanilla randomized midpoint scheme in Section~\ref{sec:discretization}.

Since the scheme involves an i.i.d. uniformly distributed random variable $U_n$, the resulting score matching function $s_*(T-t,x)$ can be evaluated at any $t\in[0,T]$. 
To proceed, we require a regularity condition on the deviation of the estimated score from the true score at these points. 
For this, we introduce the following stochastic process.

Given $U_n=u,$ we define the conditional realization of the random vector $\vartheta^{\sf REM}_{n+U}$ via
\begin{align*}
\vartheta^{\sf REM}_{n+u}:= \Big(1+\frac{uh}{2}\Big)\vartheta_n^{\sf REM}+uhs_*(T-nh,\vartheta_n^{\sf REM})+\sqrt{uh}\xi_n\,.
\end{align*}
We impose the following assumption on score estimates, which extends Assumption~\ref{asm:scoreerr}.
\begin{assumption}
    \label{asm:score4RMP}
    There exists a constant $\varepsilon_{sc}>0$ such that for any $u\in[0,1]$ and $n=0,\dots,N$,
    \begin{align*}
        \l|\nabla\log p_{T-(n+u)h}(\vartheta_{n+u}^{\sf REM})-s_*(T-(n+u)h,\vartheta_{n+u}^{\sf REM})\r|_{\Ltwo} 
        \leqslant \varepsilon_{sc}.
    \end{align*}
\end{assumption}
\vspace{-0.2cm}
In the following theorem, we provide the upper bound for the Wasserstein distance between the law of $\vartheta_N^{\sf REM}$ and the data distribution $p_0$.
\begin{theorem}
    \label{thm:RMPEM}
    Suppose that Assumptions~\ref{asm:p0scLipx},~\ref{asm:scLipt} and~\ref{asm:score4RMP} hold, then 
    \begin{align*}
        \wass_2(\mathcal{L}(\vartheta_N^{\sf REM}),p_0)\lesssim e^{-m_{\min}T}\l|X_0\r|_{\Ltwo}+\mathscr C_1^{\sf REM}(d)\sqrt{h}+\mathscr C_2^{\sf REM}\varepsilon_{sc}\,,
    \end{align*}
    where 
    \begin{align*}
       \mathscr C_1^{\sf REM}(d)=\dfrac{\sqrt{d/3}L_{\max}+\frac{1}{2\sqrt{3}}}{m_{\min}-1/2}~~\text{ and }~~ 
       \mathscr C_2^{\sf REM}=\dfrac{3}{m_{\min}-1/2}
    \end{align*}
    with $L_{\max}$ and $m_{\min}$ as defined in Theorem~\ref{thm:EM}.
\end{theorem}
As shown in display~\eqref{eq:mr}, the key idea behind the randomized midpoint method is to introduce a uniformly distributed random variable~$U_n$ in $[0,1]$ to evaluate the function $\gamma$ at a random point within the time interval $[0,h]$.

In the proof in Appendix~\ref{app:REM}, we demonstrate that $\E_{U_n}[\vartheta_{n+1}^{\sf REM}]$ provides a good approximation to the true distribution. 
However, this also introduces an additional error term~$\|\vartheta^{\sf REM}_{n+1}-\E_{U_n}[\vartheta^{\sf REM}_{n+1}]\|_{\Ltwo},$ which obscures the advantage of the improved estimation. As a result, the randomized midpoint method does not achieve a better convergence order compared to {\sf EI} and {\sf EM} methods.

Nonetheless, it is shown that this method has the advantage of enabling parallel computation~\cite{gupta2024faster, li2024improved}, which can significantly reduce computational complexity and improve efficiency. 
Therefore, while midpoint randomization does not yield a better convergence order in our setting, it offers notable benefits in terms of efficiency and scalability in parallel computing.

\subsection{Exponential Integrator with Randomized Midpoint Method}
\label{sec:REI}
With the initialization $\vartheta^{\sf REI}_{0}\sim\hat p_T$, the exponential integrator with randomized midpoint method is defined as
\begin{align*}
\vartheta^{\sf REI}_{n+U}&=e^{hU_n/2}\vartheta^{\sf REI}_{n}
+2(e^{hU_n/2}-1)s_*(T-nh,\vartheta^{\sf REI}_{n}) +\sqrt{e^{hU_n}-1}\xi'_n\\
\vartheta^{\sf REI}_{n+1}&=e^{h/2}\vartheta^{\sf REI}_{n}
+he^{(1-U_n)h/2}s_*(T-(n+U_n)h,\vartheta^{\sf REI}_{n+U}) +\sqrt{e^h-1}\xi_n\,,
\end{align*}
where $\xi'_n$ and $\xi_n$ is defined in Step 1 of this discretization scheme in Section~\ref{sec:discretization}.

Similar to the vanilla randomized midpoint scheme, since the process involves the i.i.d random variables $U_n$, we need to take into account the accuracy of the score estimates $s_*(T-t,x)$ at any $t\in[0,T].$
For this, we define the conditional realization of the
random vector $\vartheta_{n+U}^{\sf REI}$ given $U_n=u$ as following
\begin{align*}
\vartheta_{n+u}^{\sf REI}&:=
e^{uh/2}\vartheta^{\sf REI}_{n}
+2(e^{uh/2}-1)s_*(T-nh,\vartheta^{\sf REI}_{n})  +\sqrt{e^{uh}-1}\xi'_n\,.
\end{align*}
We impose the following condition on the score function estimates.
\begin{assumption}
    \label{asm:score4RMP1}
    There exists a constant $\varepsilon_{sc}>0$ such that for any $u\in[0,1]$ and $n=0,\cdots,N$,
    \begin{align*}
        \l|\nabla\log p_{T-(n+u)h}(\vartheta_{n+u}^{\sf REI})-s_*(T-(n+u)h,\vartheta_{n+u}^{\sf REI})\r|_{\Ltwo}\leqslant \varepsilon_{sc}.
    \end{align*}
\end{assumption}
\begin{theorem}
    \label{thm:RMPEI}
    Suppose that Assumption~\ref{asm:p0scLipx},~\ref{asm:scLipt} and~\ref{asm:score4RMP1} hold, then
    \begin{align*}
        \wass_2(\law(\vartheta_N^{\sf REI}),p_0)\lesssim e^{-m_{\min}T}\l|X_0\r|_{\Ltwo}+\mathscr C_1^{\sf REI}\sqrt{dh}+\mathscr C_2^{\sf REI}\varepsilon_{sc}\,,
    \end{align*}
    where 
    \begin{align*}
        \mathscr C_1^{\sf REI}=\dfrac{L_{\max}}{\sqrt{3}(m_{\min}-1/2)}~~\text{ and }~~
        \mathscr C_2^{\sf REI}=\dfrac{3}{m_{\min}-1/2}
    \end{align*}
    with $L_{\max}$ and $m_{\min}$ as defined in Theorem~\ref{thm:EM}.
\end{theorem}
The convergence rate of this scheme is generally consistent with the previous three methods, differing only in the coefficients.

\section{Second-order Acceleration}
\label{sec:accleration}
In this section, we propose an accelerated sampler that leverages Hessian estimation, %
The core idea behind the acceleration is the \textit{Local Linearization Method}, introduced in~\cite{Shoji1998}, which is to approximate the drift term of an SDE by its Itô expansion over small time intervals. 
We begin by considering the following general framework for the backward process.
\begin{align}
\label{eq:sec_sde}
dx_t=\gamma(T-t,x_t)\rmd t+\sigma \rmd W_t\,,
\end{align}
where $\sigma>0$ and $W_t$ is the $d$-dimensional Brownian motion.
We approximate the drift function $\gamma:\R_+\times\R^d\to \R^d$ by a linear function in both state and time within each discretization step. 
Applying Itô's formula to $\gamma(T-t,x)$, we obtain
\begin{align*}
\gamma(T-t,x_t)-\gamma(T-s,x_s)
&\approx\left(\frac{\sigma^2}{2}\frac{\partial^2 \gamma}{\partial x^2}(T-s,x_s)-\frac{\partial \gamma}{\partial t}(T-s,x_s)\right)(t-s) +\dfrac{\partial \gamma}{\partial x}(T-s,x_s)\cdot(x_t-x_s)\,.
\end{align*}
Here and henceforth, we abbreviate the partial derivative $\dfrac{\partial^{\alpha}g(z)}{\partial z^{\alpha}}\bigg|_{z=z_0}$ as $\dfrac{\partial^{\alpha}g}{\partial z^{\alpha}}(z_0)$. This allows us to express $\gamma(T-t,x_t)$ in the following form
\begin{align*}
\gamma(T-t,x_t)\approx \gamma(T-s,x_s)+L_s(x_t-x_s)+M_s(t-s).
\end{align*}
where $L_s, M_s$ are given by
\begin{align*}
L_s&=\dfrac{\partial \gamma}{\partial x}(T-s,x_s)\\
M_s&=\dfrac{\sigma^2}{2}\dfrac{\partial^2 \gamma}{\partial x^2}(T-s,x_s)-\dfrac{\partial \gamma}{\partial t}(T-s,x_s)\,.
\end{align*}
{Thus, $L_s$ is the first-order spatial derivative of $\gamma$, capturing its local variation with respect to changes in position $x$. 
On the other hand, $M_s$ represents the temporal evolution of $\gamma$, incorporating information about how its shape changes over both space and time.
}

Substituting this into the original SDE~\eqref{eq:sec_sde}, we obtain
\begin{align*}
dx_t=[\gamma(T-s,x_s)+L_s(x_t-x_s)+M_s(t-s)]\rmd t+\sigma \rmd W_t.
\end{align*}
This formulation ensures that the discretized process retains the essential structure of the original dynamics while being computationally tractable. 
Unlike discretization schemes we used in Section~\ref{sec:discretization}, which rely on direct numerical integration, this transformed SDE can be solved analytically within each small time interval. 

By using \text{It\^o}'s formula to $e^{-L_st}x_t$, we obtain
\begin{align*}
    \rmd(e^{-L_st}x_t)= e^{-L_st}(\gamma(T-s,x_s)-L_sx_s+M_s(t-s))\rmd t+e^{-L_st}\sigma\rmd W_t.
\end{align*}
Integrating from $s$ to $s+\Delta t$ then gives
\begin{align*}
    x_{s+\Delta t}
    &=e^{L_s\Delta t}x_s+\int_s^{s+\Delta t}e^{L_s(s+\Delta t-t)}\rmd t(\gamma(T-s,x_s)-L_sx_s)\\
    &\quad+\int_s^{s+\Delta t}e^{L_s(s+\Delta t-t)}(t-s)\rmd t\,M_s\\
    &\quad+\sigma\int_s^{s+\Delta t}e^{L_s(s+\Delta t-t)}\rmd W_t\\
    &=x_s+L_s^{-1}(e^{L_s\Delta t}-1)\gamma(T-s,x_s)\\
    &\quad+L_s^{-2}\big[(e^{L_s\Delta t}-1)-L_s\Delta t\big]M_s\\
    &\quad+\sigma\int_s^{s+\Delta t}e^{L_s(s+\Delta t-u)}dW_t\,.
\end{align*}
Setting $\gamma(T-t,x)=\dfrac{1}{2}x+\nabla\log p_{T-t}(x),\sigma=1$, 
let $\Delta t\in[0,h]$ and $s=nh$.
In the resulting expression, we denote $x_{s}$ by $\vartheta_n^{\sf SO}$.
Then, we obtain that for any $t\in[nh,(n+1)h]$
\begin{align}
\label{eq:SOxt}
    x_t&=\vartheta_n^{\sf SO}+\int_{nh}^t(\dfrac{1}{2}\vartheta_n^{\sf SO}+\nabla \log p_{T-nh}(\vartheta_n^{\sf SO})\\
    &\quad +L_n(x_u-\vartheta_n^{\sf SO})+M_n(u-nh))\rmd u+\int_{nh}^{t}\rmd W_u\notag
\end{align}
where
\begin{align*}
    L_n&=\dfrac{1}{2}I_d+\nabla^2\log p_{T-nh}(\vartheta_n^{\sf SO})\\
    M_n&=\dfrac{1}{2}\sum_{j=1}^d\dfrac{\partial^2}{\partial x_j^2}\nabla\log p_{T-nh}(\vartheta_n^{\sf SO})-\dfrac{\partial}{\partial t}\nabla\log p_{T-nh}(\vartheta_n^{\sf SO})\,.
\end{align*}
{Thus, $L_n$ contains the Hessian of the score function.
Meanwhile, $M_n$ measures the difference between the spatial and temporal changes in the score function, reflecting the balance between curvature effects and temporal adaptation in the diffusion process.}
Notice that $x_{(n+1)h}$ is the point we aim to approximate. For this, we need to estimate the score function and its higher-order derivatives to obtain accurate estimates of $L_n$ and $M_ n$, denoted by $s_*^{(L)}$ and $s_*^{(M)}$, respectively.

By the work of~\cite{meng2021highorder}, higher-order derivatives with respect to the spatial variable $x$ can be estimated with sufficient accuracy.
Additionally, we show in the proof of Proposition~\ref{prop:2order} that the partial derivative of $\nabla\log p_t(x)$ with respect to $t$ can be estimated without requiring additional assumptions. 
Thus, it suffices to impose the following assumption.
\begin{assumption}
\label{asm:scerr4SO}
    For some constants $\varepsilon_{sc}^{(L)},\varepsilon_{sc}^{(M)}>0$, the estimate for high-order derivatives of the score function satisfies that,
    \begin{align*}
        \sup_{0\leqslant n\leqslant N}\l|s_*^{(L)}(T-nh,\vartheta_n^{\sf SO})-L_n\r|_{\Ltwo}\leqslant \varepsilon_{sc}^{(L)}\\
        \sup_{0\leqslant n\leqslant N}\l|s_*^{(M)}(T-nh,\vartheta_n^{\sf SO})-M_n\r|_{\Ltwo}\leqslant \varepsilon_{sc}^{(M)}.
    \end{align*}
\end{assumption}
The conditions in Assumption~\ref{asm:scerr4SO} have been previously adopted in works such as Assumption 3 in~\cite{liang2024broadening} and Assumption 3 in~\cite{li2024accelerating}.

Substituting these estimates into the SDE~\eqref{eq:SOxt} then gives
\begin{align*}
    x_t&=\vartheta_n^{\sf SO}+\int_{nh}^t\Big(\gamma(T-nh,\vartheta_n^{\sf SO})+s_*^{(L)}(T-nh,\vartheta_n^{\sf SO})(x_u-\vartheta_n^{\sf SO})\\
    &\qquad +s_*^{(M)}(T-nh,\vartheta_n^{\sf SO})(u-nh)\Big)\rmd u+\int_{nh}^t\rmd W_u\,.
\end{align*}
Let $\vartheta_{n+1}^{\sf SO}$ denote $x_{(n+1)h}$. 
The second-order discretization scheme is given by
\begin{align*}
    \vartheta_{n+1}^{\sf SO}
    &=\vartheta_n^{\sf SO}+s_*^{(L)}(T-nh,\vartheta_n^{\sf SO})^{-1}\left(e^{s_*^{(L)}(T-nh,\vartheta_n^{\sf SO})h}-I_d\right)\left(\dfrac{1}{2}\vartheta_n^{\sf SO}+s_*(T-nh,\vartheta_n^{\sf SO})\right)\\
    &\quad +s_*^{(L)}(T-nh,\vartheta_n^{\sf SO})^{-2}\left(e^{s_*^{(L)}(T-nh,\vartheta_n^{\sf SO})h}-s_*^{(L)}(T-nh,\vartheta_n^{\sf SO})h-I_d\right) s_*^{(M)}(T-nh,\vartheta_n^{\sf SO})\\
    &\quad +\int_{nh}^{(n+1)h}e^{s_*^{(L)}(T-nh,\vartheta_n^{\sf SO})[(n+1)h-t]}\rmd W_t\,.
\end{align*}
We require additional assumptions on the smoothness of the score function, similar to those imposed on $\nabla\log p_t(x)$ in Section~\ref{sec:discretization}.
\begin{assumption}
    \label{asm:scLipx4SO}
   There exists a constant $L_F>0$ such that 
    \begin{align*}
        \left\lVert\nabla^2\log p_t(x)-\nabla^2\log p_t(y)\right\rVert_F\leqslant L_F\l|x-y\r|\,.
    \end{align*}
\end{assumption}
This implies $\nabla^2\log p_t(x)$ is $L_F$-Lipschitz in $x$ with respect to the Frobenius norm. 
As shown in Theorems 4 and 5 of \cite{liang2024broadening}, this condition plays a crucial role in bounding the Wasserstein distance for Hessian estimates and can be easily verified in the Gaussian case.

We also require the following assumption.
\begin{assumption}
    \label{asm:scLipt4SO}
    There exists a constant $M_2>0$ such that, for any $n=0,\dots,N-1$ and $t\in[nh,(n+1)h]$, it holds that
    \begin{align*}
        \l|\nabla^2\log p_{T-t}(x)-\nabla^2\log p_{T-nh}(x)\r|\leqslant M_2h(1+\l|x\r|)\,.
    \end{align*}
\end{assumption}
We are now ready to quantify the $\wass_2$ distance between the generated distribution $\law(\vartheta_N^{\sf SO})$ and the target distribution $p_0$.
\begin{theorem}
    \label{thm:2order}
    Suppose that Assumptions~\ref{asm:p0scLipx},~\ref{asm:scoreerr},~\ref{asm:scerr4SO},~\ref{asm:scLipx4SO} and~\ref{asm:scLipt4SO} hold, then
    \begin{align*}
        \wass_2(\law(\vartheta_N^{\sf SO}),p_0)
        \leqslant& e^{-m_{\min}T}\l|X_0\r|_{\Ltwo}+\mathscr C_1^{\sf SO}(d)h+\mathscr C_2^{\sf SO}\left(\varepsilon_{sc}+\dfrac{2}{3}h^{1/2}\varepsilon_{sc}^{(L)}+\dfrac{1}{2}h\varepsilon_{sc}^{(M)}\right)
    \end{align*}
    where 
    \begin{align*}
    \mathscr C_1^{\sf SO}(d)&=e^{(L_{\max}-1/2)h}\cdot\dfrac{(\sqrt{d}L_{\max}^{3/2}+3dL_F/2)}{m_{\min}-1/2}~~\text{ and }~~
    \mathscr C_2^{\sf SO}=\dfrac{e^{(L_{\max}-1/2)h}}{m_{\min}-1/2}
    \end{align*}
    with $L_{\max}$ and $m_{\min}$ as defined in Theorem~\ref{thm:EM}.
\end{theorem}
Before discussing the results of this theorem, let us state its direct consequence.
\begin{corollary}
    For a given $\varepsilon>0$, the Wasserstein distance satisfies $\wass_2(\law(\vartheta_N^{\sf SO}),p_0)<\varepsilon$ after
    \begin{align*}
        N=\mathcal{O}\left(\dfrac{1}{\varepsilon}\log\Big(\dfrac{1}{\varepsilon}\Big)\right)
    \end{align*}
    iterations, provided that $T=\mathcal{O}\left(\log(\frac{1}{\varepsilon})\right)$ and $h=\mathcal{O}(\varepsilon)$.
\end{corollary}
{
In the above theorem and corollary, we present the first Wasserstein convergence analysis of an accelerated sampler that utilizes accurate score function estimation and second-order information about log-densities.
More specifically, the error arises from three sources:
\begin{itemize}
    \item Initialization error: This is captured by the term ~$e^{-\min T}\|X_0\|_{\Ltwo}$,  which is independent of the chosen discretization scheme.
    \item Discretization error: This term reflects the advantage of second-order acceleration. By designing the discretization scheme to better approximate the original reverse SDE, we mitigate this error.
    \item Estimation error: This term accounts for errors in estimating both the spatial and temporal components of the score function. Since the scheme involves not only the score function estimate but also the terms $L_n$ and $M_n$, their corresponding errors are included in this term as well.
\end{itemize}
}

A comparison with the convergence rates of the four algorithms in Section~\ref{sec:discretization} highlights the clear computational advantage of the proposed second-order method.
Specifically,
\begin{itemize}
\setlength\itemsep{0.02em}
    \item The second-order method requires only $\widetilde{\mathcal{O}}(1/\varepsilon)$ iterations, whereas the other four methods require~$\widetilde{\mathcal{O}}(1/\varepsilon^2)$.
    \item The second-order method allows for a larger step size $h=\mathcal{O}(\varepsilon)$, enabling faster progression in each iteration, while other methods are limited to a much smaller step size of $h=\mathcal{O}(\varepsilon^2)$.
\end{itemize}
As a result, the second-order method achieves the same accuracy with fewer iterations and a larger step size, making it a more efficient numerical scheme for approximating the target distribution.

Moreover, the convergence rate $\widetilde{\mathcal{O}}(1/\varepsilon)$ is consistent with the rate obtained in \cite{liang2024broadening}, which proposes an accelerated DDPM sampler~\cite{ho2020denoising} relying on Hessian estimation. Additionally, this rate aligns with the iteration complexity results presented in \cite{huang2024convergence}, specifically when setting $p=1$ in their framework.

\section{Numerical Studies}
\label{sec:simulation}

\begin{figure*}[t]
    \centering
    \includegraphics[width=\textwidth]{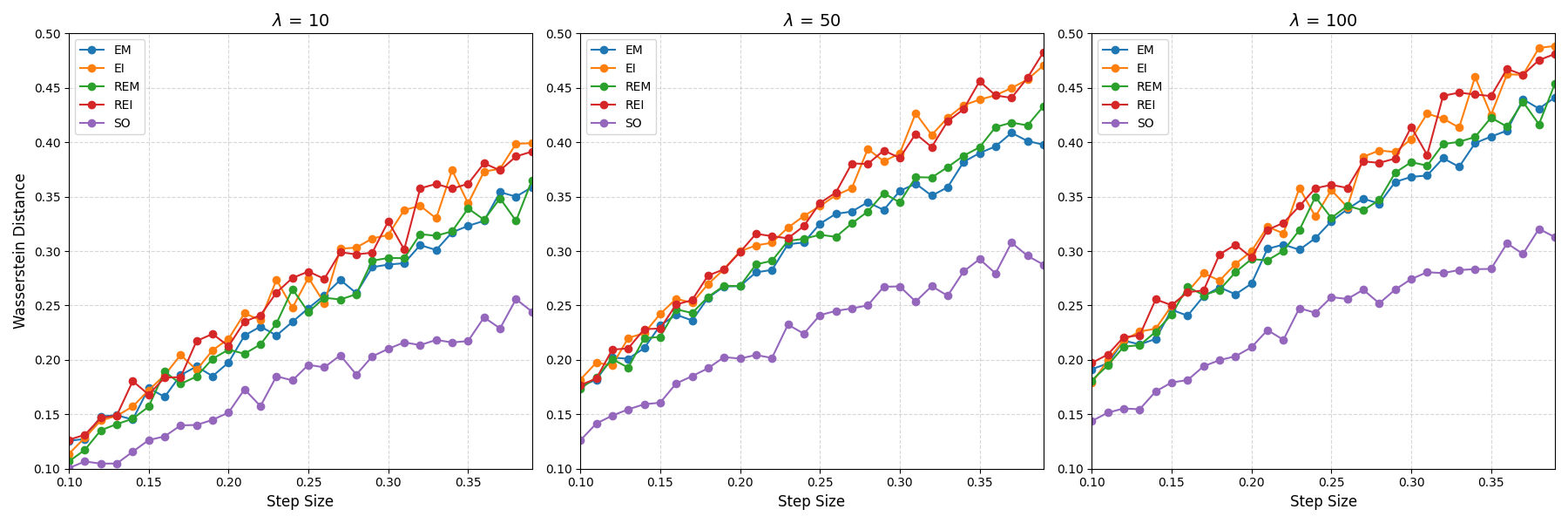}
    \caption{Error of various discretization schemes and second-order sampler with different choice of step size.}
    \label{fig:Wass_pic}
\end{figure*}
In this section, we compare the performance of the SGMs uner {\sf EM, EI, REM, REI} discretization schemes described in Section~\ref{sec:setting} and the second-order acceleration method ({\sf SO}) proposed in Section~\ref{sec:accleration}. 
We apply the five algorithms to the posterior of penalized logistic regression, defined by $p_0(\theta)\propto \exp(-f(\theta))$, with the potential function
\begin{align*}
    f(\theta)=\dfrac{\lambda}{2}\l|\theta\r|^2+\dfrac{1}{n_{\sf data}}\sum_{i=1}^{n_{\sf data}}\log\big(1+\exp(-y_ix_i^\top \theta)\big),
\end{align*}
where $\lambda>0$ denotes the tuning parameter. 
The data $\{x_i, y_i\}_{i=1}^{n_{\sf data}}$, composed of binary labels $y_i\in\{-1,1\}$ and features $x_i\in\mathbb{R}^d$ generated from $x_{i,j}\mathop{\sim}\limits^{iid}\mathcal{N}(0,100)$. 
In our experiments, we set $\lambda=10,50$ and $100$, corresponding to the plots from left to right, respectively, with $d=2$ and $n_{\sf data}=100$.
We provide more details of the calculations and implementations in Appendix~\ref{app:simulation}.

Figure~\ref{fig:Wass_pic} shows the Wasserstein distance measured along the first dimension between the empirical distributions of the $N$-th outputs from SGMs and the target distribution, with different choices of the step size~$h$. 
Here, we set $N=10/h$.
These numerical results support our theoretical findings—SGMs under different discretization schemes described in Section~\ref{sec:discretization} exhibit similar convergence behavior, while the Hessian-based sampler proposed in Section~\ref{sec:accleration} always outperforms the other four methods.
 
\section{Discussion}
\label{sec:discussion}
The Wasserstein-2 distance, used in this paper, serves as a natural and practical metric for measuring errors in diffusion models. However, recent work on the convergence theory of diffusion models has also explored alternative metrics such as total variation distance and KL divergence. A promising direction for future research is to establish convergence guarantees with respect to these alternative distances.

Moreover, while this work makes progress in provably accelerating SDE-based diffusion sampling in Wasserstein distance, it would also be valuable to explore deterministic samplers based on probability flow ODEs.

Additionally, for clarity and simplicity, we focus on a specific choice of drift functions $f$ and $g$ in the forward process, corresponding to the Ornstein–Uhlenbeck process. Extending this analysis to a more general framework with broader choices of $(f,g)$ is an interesting avenue for future research.

Finally, the assumption of strong log concavity of the data distribution is often considered restrictive. A key future direction is to relax this assumption and analyze the Wasserstein convergence behavior of the score-based diffusion models under weaker conditions.

\newpage
\bibliographystyle{amsalpha}
\bibliography{bib}
\newpage

\appendix

\section{Proof of Section~\ref{sec:discretization}}

We define several stochastic processes associated with the backward process $X_t^{\leftarrow}$ and the sample path $\vartheta_n$. First, recall that $X_t^{\leftarrow}$ is described by the following SDE:
\begin{align*}
    \rmd X_t^{\leftarrow}=\left(\dfrac{1}{2}X_t^{\leftarrow}+\nabla\log p_t(X_t^{\leftarrow})\right)\rmd t+\rmd W_t,\quad X_0^{\leftarrow}\sim p_T
\end{align*}
and $\{\vartheta_n^{\alpha},0\leqslant n\leqslant N\}$ satisfies the iterative law:
\begin{align*}
    \vartheta_{n+1}^{\alpha}=\mathcal{G}_h^{\alpha}(\vartheta_n^{\alpha},\{W_t\}_{nh\leqslant t\leqslant (n+1)h}),
\end{align*}
where $\alpha\in\{\sf EM,\sf EI,\sf REM,\sf REI,\sf SO\}$.

Based on $X_t^{\leftarrow}$, we define the following two processes, $\{Y_t,0\leqslant t\leqslant T\}$ and $\{\tilde{Y}_t,0\leqslant t\leqslant h\}$. $Y_t$ satisfies the SDE
\begin{align*}
    \rmd Y_t=\left(\dfrac{1}{2}Y_t+\nabla\log p_{T-t}(Y_t)\right)\rmd t+\rmd W_t,\quad Y_0\sim\hat p_T.
\end{align*}
$\tilde Y_t$ actually relies on $X_t^{\leftarrow}$ on the time interval $[nh,(n+1)h]$ for each $n$. However, we only need this notation in the proof of one-step discretization error, then we allow for some slight abuse of notation by omitting $n$, since it will not lead to any confusion. Therefore, $\{\tilde Y_t,0\leqslant t\leqslant h\}$ satisfies
\begin{align}
    \rmd\tilde Y_t=\left(\dfrac{1}{2}\tilde Y_t+\nabla\log p_{T-t}(\tilde Y_t)\right)\rmd t+\rmd W_t,\quad \tilde Y_0=\vartheta_n.
    \label{eq:tildeY}
\end{align}

Recall that we have defined two processes $\vartheta_{n+u}^{\sf REM}$ and $\vartheta_{n+u}^{\sf REI}$ in Section~\ref{sec:REM} and~\ref{sec:REI}, that is, for any $u\in[0,1]$ and $n=0,\cdots,N$, 
\begin{align*}
    \vartheta^{\sf REM}_{n+u}&:= (1+\frac{uh}{2})\vartheta_n^{\sf REM}+uhs_*(T-nh,\vartheta_n^{\sf REM})+\sqrt{uh}\xi_n\,,\\
    \vartheta_{n+u}^{\sf REI}&:=e^{uh/2}\vartheta^{\sf REI}_{n}+2(e^{uh/2}-1)s_*(T-nh,\vartheta^{\sf REI}_{n})+\sqrt{e^{uh}-1}\xi'_n\,.
\end{align*}

This section is devoted to proving the convergence rate of the diffusion model under various discretization schemes. 
To this end, we need the following auxiliary lemmas.
\begin{lemma}[Lemma 12 in~\cite{Gao2023WassersteinCG}]{\label{lem:GaoLt}}
    Suppose that Assumption~\ref{asm:p0scLipx} holds. Then, $\nabla\log p_t(x)$ is $L(t)$-Lipschitz, where $L(t)$ is given by 
    \begin{align*}
        L(t)=\min\{(1-e^{-t})^{-1},e^tL_0\}=
        \begin{cases}e^tL_0 & \text{ if }~t\leqslant \log(1+\frac{1}{L_0})\\
        (1-e^{-t})^{-1} & \text{ if }~t>\log(1+\frac{1}{L_0})
        \end{cases}\,.
    \end{align*}  
    Therefore, 
    \begin{align*}
        L(t)\leqslant  L_0+1\,.
    \end{align*}
\end{lemma}
\begin{lemma}[Proposition 10 in~\cite{Gao2023WassersteinCG}]{\label{lem:Gaomt}}
    Suppose that Assumption~\ref{asm:p0scLipx} holds. Then, $\nabla\log p_t(x)$ is $m(t)$-strongly log-concave, where $m(t)$ is given by
    \begin{align*}
        m(t)=\dfrac{1}{e^{-t}/m_0+(1-e^{-t})}\,.
    \end{align*}
    Therefore, 
    \begin{align*}
        m(t)\geqslant \min\{1,m_0\}\,.
    \end{align*}
\end{lemma}
Combining these two lemmas, we conclude that the Hessian matrix of $\log p_t$ satisfies the following condition
\begin{align*}
    -L(t)I_d \preccurlyeq \nabla^2\log p_t(\cdot)\preccurlyeq -m(t)I_d.
\end{align*}

We will frequently use Gr{\"o}nwall's inequality in the proof. 
Below, we present a specialized form tailored to the relevant processes.
\begin{lemma}
    \label{lem:Gron}
    Suppose the Assumption~\ref{asm:p0scLipx} holds, consider two stochastic processes $H_t$ and $G_t$ defined on the time interval $[t_1,t_2]$, if they satisfy the same SDE, especially motivated by the same Brownian motion, which means that
    \begin{align*}
        \rmd H_t=\Big(\dfrac{1}{2}H_t+\nabla\log p_{T-t}(H_t)\Big)\rmd t+\rmd W_t,\\
        \rmd G_t=\Big(\dfrac{1}{2}G_t+\nabla\log p_{T-t}(G_t)\Big)\rmd t+\rmd W_t.
    \end{align*}
    then for each $t\in[t_1,t_2]$, 
    \begin{align*}
        \l|H_t-G_t\r|_{\Ltwo}\leqslant e^{-\int_{t_1}^{t}(m(T-s)-\frac{1}{2})\rmd s}\l|H_{t_1}-G_{t_1}\r|_{\Ltwo}.
    \end{align*}
\end{lemma}
Applying Lemma~\ref{lem:Gron} to different processes and time intervals, we derive the following inequalities essential for our proof.
\begin{align}
    \l|Y_{nh+t}-\tilde{Y}_t\r|_{\Ltwo}\leqslant e^{-\int_{nh}^{nh+t}(m(T-s)-\frac{1}{2})\rmd s}\l|Y_{nh}-\tilde{Y}_0\r|_{\Ltwo},\quad\forall t\in[0,h].\label{eq:Gron1}\\
    \l|Y_t-X_t^{\leftarrow}\r|_{\Ltwo}\leqslant e^{-\int_0^{t}(m(T-s)-\frac{1}{2})\rmd s}\l|Y_0-X_0^{\leftarrow}\r|_{\Ltwo},\quad\forall t\in[0,T].\label{eq:Gron2}
\end{align}
which follow from the fact that $\{Y_t,0\leqslant t\leqslant T\}$, $\{\tilde{Y}_t,0\leqslant t\leqslant h\}$, $\{X_t^{\leftarrow},0\leqslant t\leqslant T\}$ all satisfy the same SDE as in Lemma~\ref{lem:Gron}, by applying a time-shifting operator to $\tilde{Y}_t$.\\

\subsection{Proof of Theorem~\ref{thm:EM}}
In this part, we provide the proof of Theorem~\ref{thm:EM}. To achieve this, we first prove the one-step discretization error in the following proposition.

\begin{proposition}
    \label{prop:EM}
    Suppose that Assumption~\ref{asm:p0scLipx},~\ref{asm:scLipt} and~\ref{asm:scoreerr} are satisfied. 
    Then, the following two claims hold.
    \begin{enumerate}[label=\textbf{(\arabic*)}, leftmargin=2em]
        \item \label{item:EMtilde} Firstly, it holds that
        \begin{align*}
            \l|\tilde{Y}_h-\vartheta_{n+1}^{\sf EM}\r|_{\Ltwo}\leqslant& h^2(C_1(n)^2+M_1)\l|Y_{nh}-\vartheta_n^{\sf EM}\r|_{\Ltwo}\\
            &+h^2\left[C_1(n)\left(C_1(n)C_2(n)+\dfrac{1}{2}C_4+C_3(n)\right)+M_1\left(1+C_2(n)+C_4\right)\right]\\
            &+h^{3/2}\sqrt{d}C_1(n)\\
            &+h\varepsilon_{sc},
        \end{align*}
        where
        \begin{align*}
            C_1(n)&=\dfrac{1}{2}+\dfrac{1}{h}\int_{nh}^{(n+1)h}L(T-t)\rmd t,\\
            C_2(n)&=e^{-\int_0^{nh}(m(T-t)-\frac{1}{2})\rmd t}\l|Y_0-X_T\r|_{\Ltwo},\\
            C_3(n)&=\dfrac{1}{h}\int_{nh}^{(n+1)h}(dL(T-t))^{1/2}\rmd t,\\
            C_4&=\sup_{0\leqslant t\leqslant T}\l|X_t\r|_{\Ltwo}.
        \end{align*}
        \item \label{item:EMYt}
        As a result, 
        \begin{align*}
            \l|Y_{(n+1)h}-\vartheta_{n+1}^{\sf EM}\r|_{\Ltwo}\leqslant r_n^{\sf EM}\l|Y_{nh}-\vartheta_n^{\sf EM}\r|_{\Ltwo}+h^2C_n^{\sf EM}+h^{3/2}\sqrt{d}C_1(n)+h\varepsilon_{sc},
        \end{align*}
        where
        \begin{align*}
            r_n^{\sf EM}&=e^{-\int_{nh}^{(n+1)h}(m(T-t)-\frac{1}{2})\rmd t}+h^2(C_1(n)^2+M_1),\\
            C_n^{\sf EM}&=C_1(n)\left(C_1(n)C_2(n)+\dfrac{1}{2}C_4+C_3(n)\right)+M_1\left(1+C_2(n)+C_4\right).
        \end{align*}
    \end{enumerate}
\end{proposition}
\begin{proof}
We prove the two claims sequentially.\\
\textbf{Proof of Claim~\ref{item:EMtilde}}. Rewrite display~\eqref{eq:tildeY} in the integral form,
\begin{align*}
    \tilde{Y}_h=\tilde{Y}_0+\int_{0}^{h}\left(\dfrac{1}{2}\tilde{Y}_t+\nabla\log p_{T-nh-t}(\tilde{Y}_t)\right)\rmd t+\int_{nh}^{(n+1)h}\rmd W_t \,.
\end{align*}
For Euler-Maruyama method, we can write $\vartheta_{n+1}^{\sf EM}$ in integral form as follows
\begin{align*}
    \vartheta_{n+1}^{\sf EM}=\vartheta_n^{\sf EM}+\int_{nh}^{(n+1)h}\left(\dfrac{1}{2}\vartheta_n^{\sf EM}+s_*(T-nh,\vartheta_n^{\sf EM})\right)\rmd t+\int_{nh}^{(n+1)h}\rmd W_t.
\end{align*}
Note that $\tilde Y_0=\vartheta_n^{\sf EM}$, then, it holds that
\begin{equation}
    \begin{aligned}
        \l|\tilde{Y}_h-\vartheta_{n+1}^{\sf EM}\r|_{\Ltwo}=&\left\lVert\dfrac{1}{2}\int_0^h(\tilde{Y}_t-\vartheta_n^{\sf EM})\rmd t+\int_0^h(\nabla\log p_{T-nh-t}(\tilde{Y}_t)-s_*(T-nh,\vartheta_n^{\sf EM}))\rmd t\right\rVert_{\Ltwo}\\
        \leqslant&\underbrace{\left\lVert\dfrac{1}{2}\int_0^h(\tilde{Y}_t-\tilde{Y}_0)\rmd t+\int_0^h(\nabla\log p_{T-nh-t}(\tilde{Y}_t)-\nabla\log p_{T-nh-t}(\tilde{Y}_0)\rmd t\right\rVert_{\Ltwo}}_{\text{\large I}}\\
        &+\underbrace{\left\lVert \int_0^h(\nabla\log p_{T-nh-t}(\vartheta_n^{\sf EM})-\nabla\log p_{T-nh}(\vartheta_n^{\sf EM}))\rmd t\right\rVert_{\Ltwo}}_{\text{\large II}}\\
        &+\underbrace{\left\lVert\int_0^h(\nabla\log p_{T-nh}(\vartheta_n^{\sf EM})-s_*(T-nh,\vartheta_n^{\sf EM}))\rmd t\right\rVert_{\Ltwo}}_{\text{\large III}}.
    \end{aligned}
    \label{eq:decomp}
\end{equation}
Here, we decompose the term $ \l|\tilde{Y}_h-\vartheta_{n+1}^{\sf EM}\r|_{\Ltwo}$ into a sum of three terms and then control each term individually.\\
For the term \text{I} of inequality~\eqref{eq:decomp}, by Assumption~\ref{asm:p0scLipx} and  Lipschitzness of $\nabla\log p_t$, we obtain 
\begin{align*}
    \rm{I}&=\left\lVert\dfrac{1}{2}\int_0^h(\tilde{Y}_t-\tilde{Y}_0)\rmd t+\int_0^h(\nabla\log p_{T-nh-t}(\tilde{Y}_t)-\nabla\log p_{T-nh-t}(\tilde{Y}_0)\rmd t\right\rVert_{\Ltwo}\\
    &\leqslant \dfrac{1}{2}\int_0^h\l|\tilde{Y}_t-\tilde{Y}_0\r|_{\Ltwo}\rmd t+\int_0^hL(T-nh-t)\l|\tilde{Y}_t-\tilde{Y}_0\r|_{\Ltwo}\rmd t\\
    &\leqslant \left(\dfrac{1}{2}h+\int_{nh}^{(n+1)h}L(T-t)\rmd t\right)\sup_{0\leqslant t\leqslant h}\l|\tilde{Y}_t-\tilde{Y}_0\r|_{\Ltwo}.
\end{align*}
We then proceed to derive the upper bound for the term $\sup_{0\leqslant t\leqslant h}\l|\tilde{Y}_t-\tilde{Y}_0\r|_{\Ltwo}$.
\begin{lemma}
    \label{lem:supYtY0}
    When $p_0$ satisfies Assumption~\ref{asm:p0scLipx}, it holds that
    \begin{align*}
        \sup_{0\leqslant t\leqslant h}\l|\tilde{Y}_t-\tilde{Y}_0\r|_{\Ltwo}\leqslant &\left(\dfrac{1}{2}h+\int_{nh}^{(n+1)h}L(T-t)\rmd t\right)\l|Y_{nh}-\vartheta_n^{\sf EM}\r|_{\Ltwo}\\
        +&\left(\dfrac{1}{2}h+\int_{nh}^{(n+1)h}L(T-t)\rmd t\right)e^{-\int_0^{nh}(m(T-t)-\frac{1}{2})\rmd t}\l|Y_0-X_T\r|_{\Ltwo}\\
        +&\dfrac{1}{2}h\sup_{0\leqslant t\leqslant T}\l|X_t\r|_{\Ltwo}+\int_{nh}^{(n+1)h}(dL(T-t))^{1/2}\rmd t+\sqrt{dh}.
    \end{align*}
\end{lemma}
Notice that we have no initial limit on the $\tilde{Y}_t$ in Lemma~\ref{lem:supYtY0}, which means that we can use this lemma to any discretization scheme.\\
For the term \text{II} of \eqref{eq:decomp}, we first rely on Assumption~\ref{asm:scLipt} to derive
\begin{align*}
   \rm{II} &=\left\lVert \int_0^h(\nabla\log p_{T-nh-t}(\vartheta_n^{\sf EM})-\nabla\log p_{T-nh}(\vartheta_n^{\sf EM}))\rmd t\right\rVert_{\Ltwo}\\
    &\leqslant \int_0^h\l|\nabla\log p_{T-nh-t}(\vartheta_n^{\sf EM})-\nabla\log p_{T-nh}(\vartheta_n^{\sf EM})\r|_{\Ltwo}\rmd t\\
    &\leqslant h^2M_1(1+\l|\vartheta_n^{\sf EM}\r|_{\Ltwo}).
\end{align*}
Using the triangle inequality and \eqref{eq:Gron2}, we obtain
\begin{equation}
    \begin{aligned}
        \l|\vartheta_n^{\sf EM}\r|_{\Ltwo}\leqslant&\l|Y_{nh}-\vartheta_n^{\sf EM}\r|_{\Ltwo}+\l|Y_{nh}-X_{nh}^{\leftarrow}\r|_{\Ltwo}+\l|X_{nh}^{\leftarrow}\r|_{\Ltwo}\\
        \leqslant&\l|Y_{nh}-\vartheta_n^{\sf EM}\r|_{\Ltwo}+e^{-\int_0^{nh}(m(T-t)-\frac{1}{2})\rmd t}\l|Y_0-X_T\r|_{\Ltwo}+\sup_{0\leqslant t\leqslant T}\l|X_t\r|_{\Ltwo}
    \end{aligned}
    \label{eq:hatyn}
\end{equation}
For the term \text{III} of \eqref{eq:decomp}, it follows from Assumption~\ref{asm:scoreerr} that 
\begin{align*}
    \rm{III}=&\left\lVert\int_0^h(\nabla\log p_{T-nh}(\vartheta_n^{\sf EM})-s_*(T-nh,\vartheta_n^{\sf EM}))\rmd t\right\rVert_{\Ltwo}\\
    &\leqslant\int_0^h\l|\nabla\log p_{T-nh}(\vartheta_n^{\sf EM})-s_*(T-nh,\vartheta_n^{\sf EM})\r|_{\Ltwo}\rmd t\\
   &\leqslant h\varepsilon_{sc}.
\end{align*}
Combining these terms above, we obtain that
\begin{equation}
    \begin{aligned}
        \l|\tilde{Y}_h-\vartheta_{n+1}^{\sf EM}\r|_{\Ltwo}\leqslant& h^2(C_1(n)^2+M_1)\l|Y_{nh}-\vartheta_n^{\sf EM}\r|_{\Ltwo}\\
        &+h^2\left[C_1(n)\left(C_1(n)C_2(n)+\dfrac{1}{2}C_4+C_3(n)\right)+M_1\left(1+C_2(n)+C_4\right)\right]\\
        &+h^{3/2}\sqrt{d}C_1(n)\\
        &+h\varepsilon_{sc},
    \end{aligned}
    \label{eq:EM2term}
\end{equation}
where
\begin{align*}
    C_1(n)&=\dfrac{1}{2}+\dfrac{1}{h}\int_{nh}^{(n+1)h}L(T-t)\rmd t,\\
    C_2(n)&=e^{-\int_0^{nh}(m(T-t)-\frac{1}{2})\rmd t}\l|Y_0-X_T\r|_{\Ltwo},\\
    C_3(n)&=\dfrac{1}{h}\int_{nh}^{(n+1)h}(dL(T-t))^{1/2}\rmd t,\\
    C_4&=\sup_{0\leqslant t\leqslant T}\l|X_t\r|_{\Ltwo}.
\end{align*}
This completes the proof of Claim~\ref{item:EMtilde}.

\noindent \textbf{Proof of Claim~\ref{item:EMYt}}. By the triangle inequality, we have
\begin{align}
    \label{eq:EM1}
    \l|Y_{(n+1)h}-\vartheta_{n+1}^{\sf EM}\r|_{\Ltwo}\leqslant \l|Y_{(n+1)h}-\tilde{Y}_{h}\r|_{\Ltwo}+\l|\tilde{Y}_h-\vartheta_{n+1}^{\sf EM}\r|_{\Ltwo}.
\end{align}
Applying \eqref{eq:Gron1} to the first term of \eqref{eq:EM1}, we obtain that
\begin{equation}
    \l|Y_{(n+1)h}-\tilde{Y}_h\r|^2\leqslant e^{-\int_{nh}^{(n+1)h}(2m(T-t)-1)\rmd t}\l|Y_{nh}-\tilde{Y}_0\r|^2.
    \label{eq:err1}
\end{equation}
Notice that $\tilde{Y}_0=\vartheta_n^{\sf EM}$, it then follows that
\begin{align*}
    \l|Y_{(n+1)h}-\tilde{Y}_h\r|_{\Ltwo}\leqslant e^{-\int_{nh}^{(n+1)h}(m(T-t)-\frac{1}{2})\rmd t}\l|Y_{nh}-\vartheta_n^{\sf EM}\r|_{\Ltwo}.
\end{align*}
Claim~\ref{item:EMYt} follows directly from our previous results and Claim~\ref{item:EMtilde}.
Since this step is independent of the discretization method, it applies to all the schemes discussed in this section. In the following analysis, we omit this step and proceed directly with the proof of the first claim.
\end{proof}

We now proceed to derive the upper bound of the Wasserstein distance between the sample distribution generated after $N$ iterations and the target distribution $p_0$, based on the one-step discretization error bound given by Proposition~\ref{prop:EM}.

First, note that
\begin{align*}
    \wass_2(\law(\vartheta_N^{\sf EM}),p_0)\leqslant \l|\vartheta_N^{\sf EM}-X_0\r|_{\Ltwo}\leqslant\l|Y_{Nh}-\vartheta_N^{\sf EM}\r|_{\Ltwo}+\l|Y_{Nh}-X_0\r|_{\Ltwo}.
\end{align*}
Invoking Proposition~10 of~\cite{gao2023wasserstein}, we have
\begin{align}
    \l|Y_{Nh}-X_0\r|_{\Ltwo}\leq e^{-\int_0^Tm(t)\rmd t}\l|X_0\r|_{\Ltwo}.
    \label{eq:Wassterm1}
\end{align}
According to Proposition~\ref{prop:EM}, by induction, we obtain that
\begin{equation}
    \begin{aligned}
        \l|Y_{Nh}-\vartheta_N^{\sf EM}\r|_{\Ltwo}\leqslant& r_{N-1}^{\sf EM}\l|Y_{(N-1)h}-\vartheta_{N-1}^{\sf EM}\r|_{\Ltwo}+\left(h^2C_{N-1}^{\sf EM}+h^{3/2}\sqrt{d}C_1(N-1)+h\varepsilon_{sc}\right)\\
        \leqslant&\left(\prod_{j=0}^{N-1}r_j^{\sf EM}\right)\l|Y_0-\vartheta_0^{\sf EM}\r|_{\Ltwo}+\sum_{k=0}^{N-1}\left(\prod_{j=k+1}^{N-1}r_j^{\sf EM}\right)\left(h^2C_k^{\sf EM}+h^{3/2}\sqrt{d}C_1(k)+h\varepsilon_{sc}\right)\\
        =&\sum_{k=0}^{N-1}\left(\prod_{j=k+1}^{N-1}r_j^{\sf EM}\right)\left(h^2C_k^{\sf EM}+h^{3/2}\sqrt{d}C_1(k)+h\varepsilon_{sc}\right),
    \end{aligned}
    \label{eq:Wassinduction}
\end{equation}
where we define $\prod_{j=N}^{N-1}r_j^{\sf EM}=1$.
Notice that 
\begin{align*}
    \prod_{j=k+1}^{N-1}r_j^{\sf EM}=&\prod_{j=k+1}^{N-1}(e^{-\int_{jh}^{(j+1)h}(m(T-t)-\frac{1}{2})\rmd t}+h^2(C_1(k)^2+M_1))\\
    \lesssim& \prod_{j=k+1}^{N-1}e^{-h(m_{\min}-\frac{1}{2})}=e^{h(m_{\min}-\frac{1}{2})(N-k-1)}.
\end{align*}
Therefore, we have
\begin{equation}
    \begin{aligned}
        \l|Y_{Nh}-\vartheta_N^{\sf EM}\r|_{\Ltwo}\lesssim& \sum_{k=0}^{N-1}e^{-h(m_{\min}-\frac{1}{2})(N-k-1)}\left(h^2C_k^{\sf EM}+h^{3/2}\sqrt{d}C_1(k)+h\varepsilon_{sc}\right)\\
        \leqslant& \dfrac{1}{1-e^{h(m_{\min}-\frac{1}{2})}}\left(h^2\max_{0\leqslant k\leqslant N-1}C_k^{\sf EM}+h^{3/2}\sqrt{d}\max_{0\leqslant k\leqslant N-1}C_1(k)+h\varepsilon_{sc}\right)\\
        \lesssim&\dfrac{1}{m_{\min}-1/2}\left(\sqrt{dh}\max_{0\leqslant k\leqslant N-1}C_1(k)+\varepsilon_{sc}\right).
    \end{aligned}
    \label{eq:Wassterm2}
\end{equation}
Recall the definition of $C_1(k)$ and the upper bound of $L(t)$, it follows that
\begin{align*}
    \max_{0\leqslant k\leqslant N-1}C_1(k)\leqslant \dfrac{1}{2}+L_{\max},
\end{align*}
and thus we obtain that
\begin{align*}
    \l|Y_{Nh}-\vartheta_N^{\sf EM}\r|_{\Ltwo}\lesssim \sqrt{dh}\cdot\dfrac{L_{\max}+1/2}{m_{\min}-1/2}+\varepsilon_{sc}\cdot\dfrac{1}{m_{\min}-1/2}\,.
\end{align*}
Plugging this back into the previous display then we have
\begin{align*}
    \wass_2(\law(\vartheta_N^{\sf EM}),p_0)\lesssim e^{-\int_0^Tm(t)\rmd t}\l|X_0\r|_{\Ltwo}+\sqrt{dh}\cdot\dfrac{L_{\max}+1/2}{m_{\min}-1/2}+\varepsilon_{sc}\cdot \dfrac{1}{m_{\min}-1/2}\,,
\end{align*}
which completes the proof of Theorem~\ref{thm:EM}.

\subsection{Proof of Theorem~\ref{thm:EI}}
This part aims to prove the Wasserstein convergence result for the Exponential Integrator (EI) scheme. 
We will prove this theorem using the same method as in Theorem~\ref{thm:EM}. Following this approach, we first establish the one-step discretization error in the proposition below.
\begin{proposition}
    \label{prop:EI}
    Suppose that Assumption~\ref{asm:p0scLipx},~\ref{asm:scoreerr} and~\ref{asm:scLipt} hold, then one-step discretization error for Exponential Integrator scheme is obtained from the following two bounds.
    \begin{enumerate}[label=\textbf{(\arabic*)},leftmargin=2em]
        \item \label{item:EItilde} It holds that
        \begin{align*}
            \l|\tilde{Y}_h-\vartheta_{n+1}^{\sf EI}\r|_{\Ltwo}\leqslant&h^2\left(C_5(n)C_1(n)+M_1\dfrac{2(e^{h/2}-1)}{h}\right)\l|Y_{nh}-\vartheta_n^{\sf EM}\r|_{\Ltwo}\\
            &+h^2\left[C_5(n)\left(C_1(n)C_2(n)+\dfrac{1}{2}C_4+C_3(n)\right)+\dfrac{2(e^{h/2}-1)}{h}M_1(1+C_2(n)+C_4)\right]\\
            &+h^{3/2}\sqrt{d}C_5(n)\\
            &+h\cdot\dfrac{2(e^{h/2}-1)}{h}\varepsilon_{sc},
        \end{align*}
        where
        \begin{align*}
            C_5(n)=\dfrac{1}{h}\int_{nh}^{(n+1)h}e^{\frac{1}{2}((n+1)h-t)}L(T-t)\rmd t\approx C_1(n)-\dfrac{1}{2}\,.
        \end{align*}
        \item \label{item:EIYt} Therefore, we have the bound for one-step discretization error
        \begin{align*}
            \l|Y_{(n+1)h}-\vartheta_{n+1}^{\sf EI}\r|_{\Ltwo}\leqslant r_n^{\sf EI}\l|Y_{nh}-\vartheta_n^{\sf EI}\r|_{\Ltwo}+h^2C_n^{\sf EI}+h^{3/2}\sqrt{d}C_5(n)+h\cdot\dfrac{2(e^{h/2}-1)}{h}\varepsilon_{sc},
        \end{align*}
        where
        \begin{align*}
            r_n^{\sf EI}=&e^{-\int_{nh}^{(n+1)h}(m(T-t)-\frac{1}{2})\rmd t}+h^2\left(C_5(n)C_1(n)+M_1\dfrac{2(e^{h/2}-1)}{h}\right),\\
            C_n^{\sf EI}=&C_5(n)\left(C_1(n)C_2(n)+\dfrac{1}{2}C_4+C_3(n)\right)+\dfrac{2(e^{h/2}-1)}{h}M_1(1+C_2(n)+C_4).
        \end{align*}
    \end{enumerate}
Here, the constants $C_i,i=1,2,3,4$ are as defined in Proposition~\ref{prop:EM}.
\end{proposition}
\begin{proof}
We prove two claims in succession.\\
\textbf{Proof of Claim~\ref{item:EItilde}.} Consider the process defined in \eqref{eq:tildeY}, which satisfies the SDE  
\begin{align*}
    \rmd\tilde{Y}_t=\left[\dfrac{1}{2}\tilde{Y}_t+\nabla\log p_{T-nh-t}(\tilde{Y}_t)\right]\rmd t+\rmd W_t,
\end{align*}
Instead of integrating both sides of the SDE, we use \text{It\^o}'s formula to $e^{-\frac{t}{2}}\tilde Y_t$, then we have
\begin{align*}
    \rmd (e^{-\frac{t}{2}}\tilde Y_t)=-\dfrac{1}{2}e^{-\frac{t}{2}}\tilde Y_t+e^{-\frac{t}{2}}\rmd \tilde Y_t=e^{-\frac{t}{2}}\left(\nabla\log p_{T-nh-t}(\tilde Y_t)\rmd t+\rmd W_t\right),
\end{align*}
and we notice that we can write it in an integral form.  
\begin{align*}
    \tilde{Y}_t=e^{t/2}\tilde{Y}_0+\int_0^te^{\frac{1}{2}(t-s)}\nabla\log p_{T-nh-s}(\tilde{Y}_{s})\rmd s+\int_{nh}^{nh+t}e^{\frac{1}{2}((n+1)h-s)}\rmd W_s.
\end{align*}  
Then we obtain that
\begin{align*}
    \tilde{Y}_h-\vartheta_{n+1}^{\sf EI}=\int_0^he^{\frac{1}{2}(h-t)}(\nabla\log p_{T-nh-t}(\tilde{Y}_t)-s_*(T-nh,\vartheta_n^{\sf EI}))\rmd t.
\end{align*}
We make decomposition the same as the one in \eqref{eq:decomp}, that is
\begin{align*}
    \nabla\log p_{T-nh-t}(\tilde{Y}_t)-s_*(T-nh,\vartheta_n^{\sf EI})
    &= \nabla\log p_{T-nh-t}(\tilde{Y}_t)-\nabla\log p_{T-nh-t}(\tilde{Y}_0)\\
    & \quad +\nabla\log p_{T-nh-t}(\vartheta_n^{\sf EI})-\nabla\log p_{T-nh}(\vartheta_n^{\sf EI})\\
    & \quad +\nabla\log p_{T-nh}(\vartheta_n^{\sf EI})-s_*(T-nh,\vartheta_n^{\sf EI}).
\end{align*}
It then follows that
\begin{align*}
    \l|\tilde{Y}_h-\vartheta_{n+1}^{\sf EI}\r|_{\Ltwo}\leqslant& \int_0^he^{\frac{1}{2}(h-t)}\l|\nabla\log p_{T-nh-t}(\tilde{Y}_t)-s_*(T-nh,\vartheta_n^{\sf EI})\r|_{\Ltwo}\rmd t\\
    \leqslant& \int_0^he^{\frac{1}{2}(h-t)}\l|\nabla\log p_{T-nh-t}(\tilde{Y}_t)-\nabla\log p_{T-nh-t}(\tilde{Y}_0)\r|_{\Ltwo}\rmd t\\
    &+\int_0^he^{\frac{1}{2}(h-t)}\l|\nabla\log p_{T-nh-t}(\vartheta_n^{\sf EI})-\nabla\log p_{T-nh}(\vartheta_n^{\sf EI})\r|_{\Ltwo}\rmd t\\
    &+\int_0^he^{\frac{1}{2}(h-t)}\l|\nabla\log p_{T-nh}(\vartheta_n^{\sf EI})-s_*(T-nh,\vartheta_n^{\sf EI})\r|_{\Ltwo}\rmd t.
\end{align*}
Note that apart from the exponential term $e^{\frac{1}{2}(h-t)}$, the derivation of the remaining parts is completely consistent with that of \eqref{eq:decomp}, until we encounter the term involving $\vartheta_n^{\sf EI}$, at which point we obtain
\begin{align*}
    \l|\tilde{Y}_h-\vartheta_{n+1}^{\sf EI}\r|_{\Ltwo}\leqslant& \left(\int_0^he^{\frac{1}{2}(h-t)}L(T-nh-t)\rmd t\right)\sup_{0\leqslant t\leqslant h}\l|\tilde{Y}_t-\tilde{Y}_0\r|_{\Ltwo}\\
    &+ \left(\int_0^he^{\frac{1}{2}(h-t)}\rmd t\right)M_1h\Big(1+\l|\vartheta_n^{\sf EI}\r|_{\Ltwo}\Big)\\
    &+ \left(\int_0^he^{\frac{1}{2}(h-t)}\rmd t\right)\varepsilon_{sc}.
\end{align*}
By Lemma~\ref{lem:supYtY0}, we can bound the first term on the right-hand side of the previous display.
Moreover, from~\eqref{eq:hatyn}, $\l|\vartheta_n^{\sf EI}\r|_{\Ltwo}$ 
can be bounded similarly. Substituting all coefficients with $C_i(n)$ from Proposition~\ref{prop:EM}, we obtain
\begin{align*}
    \l|\tilde{Y}_h-\vartheta_{n+1}^{\sf EI}\r|_{\Ltwo}\leqslant & h^2\cdot C_5(n)\left[ C_1(n)\l|Y_{nh}-\vartheta_n^{\sf EI}\r|_{\Ltwo}+C_1(n)C_2(n)+\dfrac{1}{2}C_4+C_3(n)\right]\\
    &+h^2\cdot \dfrac{2(e^{h/2}-1)}{h}M_1\left[1+\l|Y_{nh}-\vartheta_n^{\sf EM}\r|_{\Ltwo}+C_2(n)+C_4\right]\\
    &+h^{3/2}\sqrt{d}C_5(n)\\
    &+h\cdot\dfrac{2(e^{h/2}-1)}{h}\varepsilon_{sc}\\
    =&h^2\left(C_5(n)C_1(n)+M_1\dfrac{2(e^{h/2}-1)}{h}\right)\l|Y_{nh}-\vartheta_n^{\sf EM}\r|_{\Ltwo}\\
    &+h^2\left[C_5(n)\left(C_1(n)C_2(n)+\dfrac{1}{2}C_4+C_3(n)\right)+\dfrac{2(e^{h/2}-1)}{h}M_1(1+C_2(n)+C_4)\right]\\
    &+h^{3/2}\sqrt{d}C_5(n)\\
    &+h\cdot\dfrac{2(e^{h/2}-1)}{h}\varepsilon_{sc},
\end{align*}
where
\begin{align*}
    C_5(n)=\dfrac{1}{h}\int_{nh}^{(n+1)h}e^{\frac{1}{2}((n+1)h-t)}L(T-t)\rmd t\approx C_1(n)-\dfrac{1}{2}.
\end{align*} 

\noindent\textbf{Proof of Claim~\ref{item:EIYt}.} The proof is omitted for brevity, as it merely requires incorporating $\l|Y_{(n+1)h}-\tilde Y_h\r|_{\Ltwo}$ into the conclusion of Claim~\ref{item:EItilde}, following a similar argument as in the proof of Claim~\ref{item:EMYt} in Proposition~\ref{prop:EM}.

\end{proof}

For the proof of Theorem~\ref{thm:EI}, recall that in the proof of Theorem~\ref{thm:EM}, the three key steps~\eqref{eq:Wassterm1},~\eqref{eq:Wassinduction} and~\eqref{eq:Wassterm2} lead to the desired result. 
We now revisit these steps within the framework of other discretization schemes.

Since \eqref{eq:Wassterm1} is independent of the discretization scheme, we can directly apply it throughout the proofs of Theorems~\ref{thm:EI},~\ref{thm:RMPEM},~\ref{thm:RMPEI} and~\ref{thm:2order}. 
For \eqref{eq:Wassinduction}, we note that the $h^2$ term in $r_j^{\alpha}$ is neglected, which results in the same upper bound for $\prod_{j=k+1}^{N-1}r_j^{\alpha}$ across all discretization schemes.

Given the consistency of these two steps, for the remaining discretization schemes, we can directly derive an analogue of \eqref{eq:Wassterm2} from Claim~\ref{item:EIYt}. Therefore, in the subsequent proofs of these theorems, after establishing the corresponding proposition, we proceed directly from an expression similar to \eqref{eq:Wassterm2}.

For this theorem, we begin the proof with the following inequality
\begin{align*}
    \l|Y_{Nh}-\vartheta_N^{\sf EI}\r|_{\Ltwo}
    &\lesssim \dfrac{1}{m_{\min}-1/2}\left(hC_n^{\sf EI}+h^{1/2}\sqrt{d}\max_{0\leqslant k\leqslant N-1}C_5(k)+\varepsilon_{sc}\right)\\
    &\lesssim \dfrac{1}{m_{\min}-1/2}\left(\sqrt{dh}\max_{0\leqslant k\leqslant N-1}C_5(k)+\varepsilon_{sc}\right)\\
    &\leqslant \sqrt{dh}\cdot\dfrac{L_{\max}}{m_{\min}-1/2}+\varepsilon_{sc}\cdot\dfrac{1}{m_{\min}-1/2}.
\end{align*}
Combining this with the bound of $\l|X_0-Y_{Nh}\r|_{\Ltwo}$, we obtain
\begin{align*}
    \wass_2(\law(\vartheta_N^{\sf EI}),p_0)\lesssim e^{-m_{\min}T}\l|X_0\r|_{\Ltwo}+\sqrt{dh}\cdot\dfrac{L_{\max}}{m_{\min}-1/2}+\varepsilon_{sc}\cdot\dfrac{1}{m_{\min}-1/2}
\end{align*}
as desired.

\subsection{Proof of Theorem~\ref{thm:RMPEM}}
\label{app:REM}
In this section, we prove the Wasserstein distance between the generated distribution~$\law(\vartheta_N^{\sf REM})$ and the target distribution.
The following proposition is established for the one-step discretization error.
\begin{proposition}
    \label{prop:RMPEM}
    Suppose that Assumptions~\ref{asm:p0scLipx},~\ref{asm:scLipt} and~\ref{asm:score4RMP} are satisfied, the following two claims hold.
    \begin{enumerate}[label=\textbf{(\arabic*)}, leftmargin=2em]
        \item \label{item:REMtilde}
       It holds that
        \begin{align*}
            &\quad\l|\tilde{Y}_h-\vartheta_{n+1}^{\sf REM}\r|_{\Ltwo}\\
            &\leqslant h^2\Bigg\{\left[\int_0^1\int_0^1\left[|u-v|L(T-(n+u)h)\left(\dfrac{1}{2}+L(T-nh)\right)+M_1\right]^2\rmd u\rmd v\right]^{1/2}\\
            &\qquad\qquad+\dfrac{1}{4\sqrt{3}}L(T-nh)+\dfrac{1}{8\sqrt{3}}\Bigg\}\l|Y_{nh}-\vartheta_n^{\sf REM}\r|_{\Ltwo}\\
            &\quad+h^2\Bigg\{\bigg\{\int_0^1\int_0^1\bigg\{(u-v)\left[\left(\dfrac{1}{2}+L(T-nh)\right)C_2(n)+\dfrac{1}{2}C_4+(dL(T-nh))^{1/2}\right]L(T-(n+u)h)\\
            &\qquad\qquad+M_1(1+C_2(n)+C_4)\bigg\}^2\rmd u\rmd v\bigg\}^{1/2}\\
            &\qquad \quad+\dfrac{1}{8\sqrt{3}}(C_2(n)+C_4)+\dfrac{1}{4\sqrt{3}}\left(L(T-nh)C_2(n)+(dL(T-nh))^{1/2}\right)\Bigg\}\\
            &\quad +h^{3/2}\Bigg\{\sqrt{d}\left[\int_0^1\int_0^1L(T-(n+u)h)^2|u-v|\rmd u\rmd  v\right]^{1/2}+\dfrac{1}{2\sqrt{3}}\Bigg\}\\
            &\quad +2h\varepsilon_{sc}.
        \end{align*}
        \item \label{item:REMYt}
       Moreover, it holds that
        \begin{align*}
            \l|Y_{(n+1)h}-\vartheta_{n+1}^{\sf REM}\r|\leqslant r_n^{\sf REM}\l|Y_{nh}-\vartheta_n\r|_{\Ltwo}+h^2C_{n,1}^{\sf REM}+h^{3/2}C_{n,2}^{\sf REM}+3h\varepsilon_{sc},
        \end{align*}
        where
        \begin{align*}
            r_n^{\sf REM}&=e^{-\int_{nh}^{(n+1)h}(m(T-t)-\frac{1}{2})\rmd t}\\
            &\quad +h^2\Bigg\{\left[\int_0^1\int_0^1\left[|u-v|L(T-(n+u)h)\left(\dfrac{1}{2}+L(T-nh)\right)+M_1\right]^2\rmd u\rmd v\right]^{1/2}\\
            &\qquad\qquad+\dfrac{1}{4\sqrt{3}}L(T-nh)+\dfrac{1}{8\sqrt{3}}\Bigg\},\\
            C_{n,1}^{\sf REM}=&\bigg\{\int_0^1\int_0^1\bigg\{(u-v)\left[\left(\dfrac{1}{2}+L(T-nh)\right)C_2(n)+\dfrac{1}{2}C_4+(dL(T-nh))^{1/2}\right]L(T-(n+u)h)\\
            &\qquad\qquad\quad+M_1(1+C_2(n)+C_4)\bigg\}^2\rmd u\rmd v\bigg\}^{1/2}\\
            &\quad+\dfrac{1}{8\sqrt{3}}(C_2(n)+C_4)+\dfrac{1}{4\sqrt{3}}\left(L(T-nh)C_2(n)+(dL(T-nh))^{1/2}\right)\\
            C_{n,2}^{\sf REM}=&\sqrt{d}\left[\int_0^1\int_0^1L(T-(n+u)h)^2|u-v|\rmd u\rmd v\right]^{1/2}+\dfrac{1}{2\sqrt{3}}.
        \end{align*}
    \end{enumerate}
\end{proposition}

\begin{proof}[Proof of Proposition~\ref{prop:RMPEM}]\\
\noindent \textbf{Proof of Claim~\ref{item:REMtilde}.} We make the following decomposition of one-step discretization error
\begin{align}
    \label{eq:REMdecom}
    \l|\tilde{Y}_{h}-\vartheta_{n+1}^{\sf REM}\r|_{\Ltwo}\leqslant\l|\tilde{Y}_h-\mathbb{E}_{U_n}\big[\vartheta_{n+1}^{\sf REM}\big]\r|_{\Ltwo}+\l|\mathbb{E}_{U_n}\big[\vartheta_{n+1}^{\sf REM}\big]-\vartheta_{n+1}^{\sf REM}\r|_{\Ltwo}.
\end{align}

We first derive the upper bound for the term~$\l|\tilde{Y}_h-\mathbb{E}_{U_n}\big[\vartheta_{n+1}^{\sf REM}\big]\r|_{\Ltwo}$. 
By the definitions of $\vartheta_n^{\sf REM}$ and $\tilde Y_h,$ we have
\begin{align*}
    &\l|\tilde{Y}_{h}-\mathbb{E}_{U_n}\big[\vartheta_{n+1}^{\sf REM}\big]\r|_{\Ltwo}\\
    =&\left\lVert \dfrac{1}{2}\int_0^h\tilde{Y}_t\rmd t+\int_0^h\nabla\log p_{T-nh-t}(\tilde{Y}_t)\rmd t-\dfrac{1}{2}h\mathbb{E}_{U_n}(\vartheta_{n+U}^{\sf REM})-h\mathbb{E}_{U_n}[s_*(T-(n+U_n)h,\vartheta_{n+U}^{\sf REM}])\right\rVert_{\Ltwo}\,.
\end{align*}
Notice that 
\begin{align*}
\int_0^h\tilde{Y}_t\rmd t = h\mathbb{E}_{U_n}[\tilde{Y}_{U_nh}],\qquad \int_0^h\nabla\log p_{T-nh-t}(\tilde{Y}_t)\rmd t= h\mathbb{E}_{U_n}[\nabla\log p_{T-(n+U_n)h}(\tilde{Y}_{U_nh})]\,.
\end{align*}
Plugging this back into the previous display then gives
\begin{align*}
 &\l|\tilde{Y}_{h}-\mathbb{E}_{U_n}[\vartheta_{n+1}^{\sf REM}]\r|_{\Ltwo}\\
 =&\,\,\left\lVert \dfrac{1}{2}h\mathbb{E}_{U_n}[\tilde{Y}_{U_nh}]+h\mathbb{E}_{U_n}[\nabla\log p_{T-(n+U_n)h}(\tilde{Y}_{U_nh})]-\dfrac{1}{2}h\mathbb{E}_{U_n}(\vartheta_{n+U}^{\sf REM})-h\mathbb{E}_{U_n}(s_*(T-(n+U_n)h,\vartheta_{n+U}^{\sf REM}))\right\rVert_{\Ltwo}\\
    \leqslant&\,\,\dfrac{1}{2}h\left\lVert\mathbb{E}_{U_n}[\tilde{Y}_{U_nh}-\vartheta_{n+U_n}^{\sf REM}]\right\rVert_{\Ltwo}+h\left\lVert\mathbb{E}_{U_n}[\nabla\log p_{T-(n+U_n)h}(\tilde{Y}_{U_nh})-s_*(T-(n+U_n)h,\vartheta_{n+U_n}^{\sf REM})]\right\rVert_{\Ltwo}.
\end{align*}
By the definition of $\tilde{Y}_{U_nh}$ and $\vartheta_{n+U_n}^{\sf REM}$, we have
\begin{align*}
    &\left\lVert \mathbb{E}_{U_n}[\tilde{Y}_{U_nh}-\vartheta_{n+U_n}^{\sf REM}]\right\rVert_{\Ltwo}\\
    =&\,\,\left\lVert\mathbb{E}_{U_n}\left[ \dfrac{1}{2}\int_0^{U_nh}(\tilde{Y}_t-\vartheta_n^{\sf REM})\rmd t+\int_0^{U_nh}(\nabla\log p_{T-nh-t}(\tilde{Y}_t)-s_*(T-nh,\vartheta_n^{\sf REM}))\rmd t\right]\right\rVert_{\Ltwo}\\
    \leqslant &\,\, \left\lVert\mathbb{E}_{U_n}\left[ \dfrac{1}{2}\int_0^{U_nh}\l|\tilde{Y}_t-\vartheta_n^{\sf REM}\r|\rmd t+\int_0^{U_nh}\l|\nabla\log p_{T-nh-t}(\tilde{Y}_t)-s_*(T-nh,\vartheta_n^{\sf REM})\r|\rmd t\right]\right\rVert_{\Ltwo}\\
    \leqslant &\,\,\left\lVert\mathbb{E}_{U_n}\left[\dfrac{1}{2}\int_0^{h}\l|\tilde{Y}_t-\vartheta_n^{\sf REM}\r|\rmd t+\int_0^h\l|\nabla\log p_{T-nh-t}(\tilde{Y}_t)-s_*(T-nh,\vartheta_n^{\sf REM})\r|\rmd t\right]\right\rVert_{\Ltwo}\\
    \leqslant &\,\,\dfrac{1}{2}\int_0^h\l|\tilde{Y}_t-\vartheta_n^{\sf REM}\r|_{\Ltwo}\rmd t+\int_0^h\l|\nabla\log p_{T-nh-t}(\tilde{Y}_t)-s_*(T-nh,\vartheta_n^{\sf REM})\r|_{\Ltwo}\rmd t.
\end{align*}
The second inequality arises because the integrand is non-negative, the last inequality follows from the fact that the random variables inside the inner expectation $\mathbb{E}_{U_n}$ are independent of $U_n$, and thus the inner expectation can be ignored. Then using the same argument as in the proof of Proposition~\ref{prop:EM}, especially adopting the same procedure as the one following \eqref{eq:decomp}, we can apply the conclusion of Proposition~\ref{prop:EM} to the term above, then we obtain that
\begin{align*}
    \l|\mathbb{E}_{U_n}\big[\tilde{Y}_{U_nh}-\vartheta_{n+U_n}^{\sf REM}\big]\r|_{\Ltwo}
    &\leqslant \left(\dfrac{1}{2}h+\int_{nh}^{(n+1)h}L(T-t)\rmd t\right)\sup_{0\leqslant t\leqslant h}\l|\tilde{Y}_t-\tilde{Y}_0\r|_{\Ltwo}\\
    & \quad +\int_{nh}^{(n+1)h}\l|\nabla\log p_{T-nh-t}(\vartheta_n^{\sf REM})-s_*(T-nh,\vartheta_n^{\sf REM})\r|_{\Ltwo}\rmd t\\
   & \leqslant h^2(C_1(n)^2+M_1)\l|Y_{nh}-\vartheta_n^{\sf REM}\r|_{\Ltwo}\\
    &\quad +h^2\left[C_1(n)\left(C_1(n)C_2(n)+\dfrac{1}{2}C_4+C_3(n)\right)+M_1(1+C_2(n)+C_4)\right]\\
    &\quad +h^{3/2}\sqrt{d}C_1(n)\\
    &\quad +h\varepsilon_{sc}\\
    &\mathop{=}^{\triangle}h^2r_1\l|Y_{nh}-\vartheta_n^{\sf REM}\r|_{\Ltwo}+h^2r_2+h^{3/2}\sqrt{d}C_1(n)+h\varepsilon_{sc},
\end{align*}
where
\begin{align*}
    r_1=&C_1(n)^2+M_1,\\
    r_2=&C_1(n)\left(C_1(n)C_2(n)+\dfrac{1}{2}C_4+C_3(n)\right)+M_1(1+C_2(n)+C_4).
\end{align*}
We now derive the upper bound of the second term in \eqref{eq:REMdecom}.
Note that
\begin{align}
    &\left\lVert\mathbb{E}_{U_n}\big[\nabla\log p_{T-(n+U_n)h}(\tilde{Y}_{(n+U_n)h})-s_*(T-(n+U_n)h,\vartheta_{n+U_n}^{\sf REM})\big]\right\rVert_{\Ltwo}\\
    = & \,\,\left\lVert \int_0^1\Big(\nabla\log p_{T-(n+u)h}(\tilde{Y}_{(n+u)h})-s_*(T-(n+u)h,\vartheta_{n+u}^{\sf REM}) \Big)\rmd u \right\rVert_{\Ltwo}\\
    \leqslant &\int_0^1\left\lVert \nabla\log p_{T-(n+u)h}(\tilde{Y}_{(n+u)h})-s_*(T-(n+u)h,\vartheta_{n+u}^{\sf REM})\right\rVert_{\Ltwo}\rmd u\\
    \leqslant &\int_0^1\Bigg(\l|\nabla\log p_{T-(n+u)h}(\tilde{Y}_{(n+u)h})-\nabla\log p_{T-(n+u)h}(\vartheta_{n+u}^{\sf REM})\r|_{\Ltwo}\\
    &\quad\qquad+\l|\nabla\log p_{T-(n+u)h}(\vartheta_{n+u}^{\sf REM})-s_*(T-(n+u)h,\vartheta_{n+u}^{\sf REM})\r|_{\Ltwo}\Bigg)\rmd u\\
    \leqslant &\int_0^1L(T-(n+u)h)\l|\tilde{Y}_{(n+u)h}-\vartheta_{n+u}^{\sf REM}\r|_{\Ltwo}\rmd u+\varepsilon_{sc},
\label{eq:help2}
\end{align}
the second inequality follows from the triangle inequality, and the last inequality depends on Assumption~\ref{asm:p0scLipx} and~\ref{asm:score4RMP}. By \eqref{eq:EM2term}, changing the value of $h$ to $uh$, we have
\begin{align}
    \l|\tilde{Y}_{(n+u)h}-\vartheta_{n+u}^{\sf REM}\r|_{\Ltwo}
    &\leqslant (uh)^2(C_{1,n}(u)^2+M_1)\l|Y_{nh}-\vartheta_n^{\sf REM}\r|_{\Ltwo}\\
    &\quad +(uh)^2\left[C_{1,n}(u)\left(C_{1,n}(u)C_2(n)+\dfrac{1}{2}C_4+C_{3,n}(u)\right)+M_1(1+C_2(n)+C_4)\right]\\
    &\quad +(uh)^{3/2}\sqrt{d}C_{1,n}(u)\\
    &\quad+ uh\varepsilon_{sc},
    \label{eq:help1}
\end{align}
where $C_{1,n}(u)$ and $C_{3,n}(u)$ is the $uh$-version of $C_1(n)$ and $C_3(n)$, respectively, that is
\begin{align*}
    C_{1,n}(u)=&\dfrac{1}{2}+\dfrac{1}{uh}\int_{nh}^{(n+u)h}L(T-t)\rmd t,\\
    C_{3,n}(u)=&\dfrac{1}{uh}\int_{nh}^{(n+u)h}(dL(T-t))^{1/2}\rmd t,
\end{align*}
Plugging the previous display~\eqref{eq:help1} back into display~\eqref{eq:help2}, then rearranging and simplifying the expression, yields
\begin{align*}
    &\left\lVert\mathbb{E}_{U_n}[\nabla\log p_{T-(n+U_n)h}(\tilde{Y}_{(n+U_n)h})-s_*(T-(n+U_n)h,\vartheta_{n+U_n}^{\sf REM})]\right\rVert_{\Ltwo}\\
   & \leqslant h^2\left(\int_0^1L(T-(n+u)h)u^2(C_{1,n}(u)^2+M_1)\rmd u\right)\l|Y_{nh}-\vartheta_n^{\sf EM}\r|_{\Ltwo}\\
    &\quad +h^2\left\{\int_0^1L(T-(n+u)h)u^2\left[C_{1,n}(u)\left(C_{1,n}(u)C_2(n)+\dfrac{1}{2}C_4+C_{3,n}(u)\right)+M_1(1+C_2(n)+C_4)\right]\rmd u\right\}\\
    &\quad +h^{3/2}\left(\int_0^1L(T-(n+u)h)u^{3/2}\rmd u\right)\sqrt{d}C_1(n)\\
    &\quad +h\left(\int_0^1L(T-(n+u)h)u\rmd u\right)\varepsilon_{sc}\\
    &\quad +\varepsilon_{sc}\\
    & \mathop{=}^{\triangle}h^2r_3\l|Y_{nh}-\vartheta_n^{\sf EM}\r|_{\Ltwo}+h^2r_4+h^{3/2}r_5+hr_6\varepsilon_{sc}+\varepsilon_{sc},
\end{align*}
where
\begin{align*}
    r_3&=\int_0^1L(T-(n+u)h)u^2(C_{1,n}(u)^2+M_1)\rmd u,\\
    r_4&=\int_0^1L(T-(n+u)h)u^2\left[C_{1,n}(u)\left(C_{1,n}(u)C_2(n)+\dfrac{1}{2}C_4+C_{3,n}(u)\right)\right]\rmd u+M_1(1+C_2(n)+C_4),\\
    r_5&=\left(\int_0^1L(T-(n+u)h)u^{3/2}\rmd u\right)\sqrt{d}C_1(n),\\
    r_6&=\int_0^1L(T-(n+u)h)u\rmd u.
\end{align*}
From the bounds we have obtained for two terms, it follows that
\begin{align*}
    &\l|\tilde{Y}_{h}-\mathbb{E}_{U_n}[\vartheta_{n+1}^{\sf REM}]\r|_{\Ltwo}\\
    \leqslant& \,h^3(\dfrac{1}{2}r_1+r_3)\l|Y_{nh}-\vartheta_n^{\sf EM}\r|_{\Ltwo}+h^3(\dfrac{1}{2}r_2+r_4)+h^{5/2}(\dfrac{1}{2}\sqrt{d}C_1(n)+r_5)+h^2(\dfrac{1}{2}+r_6)\varepsilon_{sc}+h\varepsilon_{sc}.
\end{align*}
Considering the second term of one-step discretization error
\begin{equation}
    \begin{aligned}
        &\vartheta_{n+1}^{\sf REM}-\mathbb{E}_{U_n}[\vartheta_{n+1}^{\sf REM}]\\
        =&\,\dfrac{1}{2}h\left[\vartheta_{n+U}^{\sf REM}-\mathbb{E}_{U_n}[\vartheta_{n+U}^{\sf REM}]\right]+h\left[s_*(T-(n+U_n)h,\vartheta_{n+U}^{\sf REM})-\mathbb{E}_{U_n}[s_*(T-(n+U_n)h,\vartheta_{n+U}^{\sf REM})]\right]\\
        =&\,\dfrac{1}{2}h\left[\dfrac{1}{2}h(U_n-\dfrac{1}{2})\vartheta_n^{\sf REM}+h(U_n-\dfrac{1}{2})s_*(T-nh,\vartheta_n^{\sf REM})\right]\\
        &+\dfrac{1}{2}h\left[\int_{nh}^{(n+U_n)h}\rmd W_t-\int_0^1\left(\int_{nh}^{(n+u)h}\rmd W_t\right)\rmd u\right]\\
        &+h\left[s_*(T-(n+U_n)h,\vartheta_{n+U}^{\sf REM})-\mathbb{E}_{U_n}[s_*(T-(n+U_n)h,\vartheta_{n+U}^{\sf REM})]\right].
    \end{aligned}
    \label{eq:REMdecom2}
\end{equation}
The second equality follows from the fact that
\begin{align*}
    \mathbb{E}_{U_n}[\vartheta_{n+U}^{\sf REM}]&=\vartheta_n^{\sf REM}+\dfrac{1}{2}h\mathbb{E}_{U_n}[U_n]\vartheta_n^{\sf REM}+h\mathbb{E}_{U_n}[U_n]s_*(T-nh,\vartheta_n^{\sf REM})+\mathbb{E}_{U_n}\int_{nh}^{(n+U_n)h}\rmd W_t\\
    &=\vartheta_n^{\sf REM}+\dfrac{1}{4}h\vartheta_n^{\sf REM}+\dfrac{1}{2}hs_*(T-nh,\vartheta_n^{\sf REM})+\int_0^1\left(\int_{nh}^{(n+u)h}\rmd W_t\right)\rmd u,
\end{align*}
since $U_n$ is independent of $\vartheta_n^{\sf REM}$.\\
We proceed to bound each term in \eqref{eq:REMdecom2}. For the first term, still notice that the independence between $U_n$ and $\vartheta_n^{\sf REM}$, then we find that
\begin{align*}
    \l|(U_n-\dfrac{1}{2})\vartheta_n^{\sf REM}\r|_{\Ltwo}&=\bigg\{\mathbb{E}\left[\mathbb{E}_{U_n}\left\lVert (U_n-\dfrac{1}{2})\vartheta_n^{\sf REM}\right\rVert^2\right]\bigg\}^{1/2}\\
    &=\bigg\{\mathbb{E}\left[\mathbb{E}_{U_n}[(U_n-\dfrac{1}{2})^2]\cdot\left\lVert\vartheta_n^{\sf REM}\right\rVert^2\right]\bigg\}^{1/2}\\
    &=\bigg\{\mathbb{E}\left[\dfrac{1}{12}\l|\vartheta_n^{\sf REM}\r|^2\right]\bigg\}^{1/2}\\
    &=\dfrac{1}{2\sqrt{3}}\l|\vartheta_n^{\sf REM}\r|_{\Ltwo}.
\end{align*}
The bounding of another part of the first term follows in a similar manner, we obtain that 
\begin{align*}
    \l|(U_n-\dfrac{1}{2})s_*(T-nh,\vartheta_n^{\sf REM})\r|_{\Ltwo}=\dfrac{1}{2\sqrt{3}}\l|s_*(T-nh,\vartheta_n^{\sf REM})\r|_{\Ltwo}.
\end{align*}
For the second term of \eqref{eq:REMdecom2}, notice that due to It\^o's isometry formula, for any well-defined stochastic process $X_t$ and its It\^o stochastic integral $I_t(X)=\int_0^tX_u\rmd M_u$, we have
\begin{align}
    \label{eq:Itoiso}
    \mathbb{E}[I_t(X)^2]=\mathbb{E}\int_0^tX_u^2\rmd\langle M \rangle_u,
\end{align}
then we can establish a lemma.
\begin{lemma}
    \label{lem:Brown1}
    Suppose $W_t$ is a $d$-dim standard Brownian motion, then
    \begin{align*}
        \left\lVert\int_{nh}^{(n+U_n)h}\rmd W_t-\int_0^1\left(\int_{nh}^{(n+u)h}\rmd W_t\right)\rmd u\right\rVert_{\Ltwo}^2\leqslant\dfrac{h}{3}.
    \end{align*}
\end{lemma}
For the third term of \eqref{eq:REMdecom2},  we get
\begin{align*}
    &\left\lVert s_*(T-(n+U_n)h,\vartheta_{n+U}^{\sf REM})-\mathbb{E}_{U_n}[s_*(T-(n+U_n)h,\vartheta_{n+U}^{\sf REM})]\right\rVert_{\Ltwo}\\
    = &\,\left\lVert \int_0^1 s_*(T-(n+U_n)h,\vartheta_{n+U}^{\sf REM})-s_*(T-(n+v)h,\vartheta_{n+v}^{\sf REM})\rmd v\right\rVert_{\Ltwo} \\
    = &\,\bigg\{\mathbb{E}\int_0^1\left[\int_0^1s_*(T-(n+u)h,\vartheta_{n+u}^{\sf REM})-s_*(T-(n+v)h,y_{(n+v)h}^{})\rmd v\right]^2\rmd u\bigg\}^{1/2}\\
    \leqslant &\,\bigg\{\mathbb{E}\int_0^1\int_0^1 \left[s_*(T-(n+u)h,\vartheta_{n+u}^{\sf REM})-s_*(T-(n+v)h,\vartheta_{n+v}^{\sf REM})\right]^2\rmd u\rmd v\bigg\}^{1/2}\\
    = &\,\bigg\{\int_0^1\int_0^1\left\lVert s_*(T-(n+u)h,\vartheta_{n+u}^{\sf REM})-s_*(T-(n+v)h,\vartheta_{n+v}^{\sf REM})\right\rVert_{\Ltwo}^2\rmd u\rmd v\bigg\}^{1/2}.
\end{align*}
Then by the triangle inequality and Assumption~\ref{asm:score4RMP}, we have
\begin{equation}
    \begin{aligned}
        &\left\lVert s_*(T-(n+u)h,\vartheta_{n+u}^{\sf REM})-s_*(T-(n+v)h,\vartheta_{n+v}^{\sf REM})\right\rVert_{\Ltwo}\\
        \leqslant&\,\left\lVert s_*(T-(n+u)h,\vartheta_{n+u}^{\sf REM})-\nabla\log p_{T-(n+u)h}(\vartheta_{n+u}^{\sf REM})\right\rVert_{\Ltwo}\\
        &+\left\lVert s_*(T-(n+v)h,\vartheta_{n+v}^{\sf REM})-\nabla\log p_{T-(n+v)h}(\vartheta_{n+v}^{\sf REM})\right\rVert_{\Ltwo}\\
        &+\left\lVert\nabla\log p_{T-(n+u)h}(\vartheta_{n+u}^{\sf REM})-\nabla\log p_{T-(n+v)h}(\vartheta_{n+v}^{\sf REM})\right\rVert_{\Ltwo}\\
        \leqslant& \,2\varepsilon_{sc}+\left\lVert\nabla\log p_{T-(n+u)h}(\vartheta_{n+u}^{\sf REM})-\nabla\log p_{T-(n+v)h}(\vartheta_{n+v}^{\sf REM})\right\rVert_{\Ltwo}.
    \end{aligned}
    \label{eq:diffmatch}
\end{equation}
Combining the three terms of \eqref{eq:REMdecom2} together, we have
\begin{align*}
    \l|\vartheta_{n+1}^{\sf REM}-\mathbb{E}_{U_n}[\vartheta_{n+1}^{\sf REM}]\r|_{\Ltwo}
   & \leqslant \dfrac{1}{8\sqrt{3}}h^2\l|\vartheta_n^{\sf REM}\r|_{\Ltwo}+\dfrac{1}{4\sqrt{3}}h^2\l|s_*(T-nh,\vartheta_n^{\sf REM})\r|_{\Ltwo}+\dfrac{1}{2\sqrt{3}}h^{3/2}\\
    &\quad +h\bigg\{\int_0^1\int_0^1\left\lVert \nabla\log p_{T-(n+v)h}(\vartheta_{n+v}^{\sf REM})-\nabla\log p_{T-(n+u)h}(\vartheta_{n+v}^{\sf REM})\right\rVert_{\Ltwo}^2\rmd u\rmd v\bigg\}^{1/2}\\
    &\quad +2h\varepsilon_{sc}.
\end{align*}
By applying the same technique used in the proofs of Proposition~\ref{prop:EM} and Proposition~\ref{prop:EI}, the upper bounds for $\l|\vartheta_n^{\sf REM}\r|_{\Ltwo}$ and $\l|s_*(T-nh,\vartheta_n^{\sf REM})\r|_{\Ltwo}$ follows readily. Thus, the proposition follows immediately from the bound on the second last term. 
We now consider the case for $u>v$, due to Assumptions~\ref{asm:p0scLipx} and~\ref{asm:scLipt},
\begin{align*}
    &\left\lVert\nabla\log p_{T-(n+u)h}(\vartheta_{n+u}^{\sf REM})-\nabla\log p_{T-(n+v)h}(\vartheta_{n+v}^{\sf REM})\right\rVert_{\Ltwo}\\
    \leqslant &\,\, L(T-(n+u)h)\l|\vartheta_{n+u}^{\sf REM}-\vartheta_{n+v}^{\sf REM}\r|_{\Ltwo}+M_1h\Big(1+\l|\vartheta_{n+v}^{\sf REM}\r|_{\Ltwo}\Big).
\end{align*}
Since
\begin{align*}
    \l|\vartheta_{n+u}^{\sf REM}-\vartheta_{n+v}^{\sf REM}\r|_{\Ltwo}
    &\leqslant \dfrac{1}{2}(u-v)h\l|\vartheta_n^{\sf REM}\r|_{\Ltwo}+(u-v)h\l|s_*(T-nh,\vartheta_{n}^{\sf REM})\r|_{\Ltwo}+\left\lVert\int_{(n+v)h}^{(n+u)h}\rmd W_t\right\rVert_{\Ltwo}\\
    & \leqslant  \dfrac{1}{2}(u-v)h\left(\l|Y_{nh}-\vartheta_n^{\sf REM}\r|_{\Ltwo}+C_2(n)+C_4\right)\\
    &\quad  +(u-v)h\left[\varepsilon_{sc}+L(T-nh)\left(\l|Y_{nh}-\vartheta_n^{\sf REM}\r|_{\Ltwo}+C_2(n)\right)+(dL(T-nh))^{1/2}\right]\\
    &\quad +\sqrt{(u-v)h}\\
    &\leqslant  (u-v)h\left[\dfrac{1}{2}+L(T-nh)\right]\l|Y_{nh}-\vartheta_n^{\sf REM}\r|_{\Ltwo}\\
    &\quad +(u-v)h\left[(\dfrac{1}{2}+L(T-nh))C_2(n)+\dfrac{1}{2}C_4+(dL(T-nh))^{1/2}\right]\\
    &\quad +\sqrt{(u-v)dh}+(u-v)h\varepsilon_{sc}.
\end{align*}
The second inequality follows from \eqref{eq:hatyn}, Assumptions~\ref{asm:p0scLipx},~\ref{asm:scoreerr} and Lemma~\ref{lem:Enabla}. Similarly,
\begin{align*}
    \l|\vartheta_{n+v}^{\sf REM}\r|_{\Ltwo}
    &\leqslant \l|\vartheta_{n+v}^{\sf REM}-\vartheta_{n}^{\sf REM}\r|_{\Ltwo}+\l|\vartheta_n^{\sf REM}\r|_{\Ltwo}\\
    &\leqslant vh\left[\dfrac{1}{2}+L(T-nh)\right]\l|Y_{nh}-\vartheta_n^{\sf REM}\r|_{\Ltwo}\\
    &\quad +vh\left[(\dfrac{1}{2}+L(T-nh))C_2(n)+\dfrac{1}{2}C_4+(dL(T-nh))^{1/2}\right]\\
    &\quad +\sqrt{vdh}+vh\varepsilon_{sc}\\
    &\quad +\l|Y_{nh}-\vartheta_n^{\sf REM}\r|_{\Ltwo}+C_2(n)+C_4.
\end{align*}
Therefore, we obtain that
\begin{align*}
    &\left\lVert \nabla\log p_{T-(n+u)h}(\vartheta_{n+u}^{\sf REM})-\nabla\log p_{T-(n+v)h}(\vartheta_{n+v}^{\sf REM})\right\rVert_{\Ltwo}\\
    \leqslant & h\bigg\{(u-v)\left[\dfrac{1}{2}+L(T-nh)\right]L(T-(n+u)h)+M_1\bigg\}\l|Y_{nh}-\vartheta_n^{\sf REM}\r|_{\Ltwo}\\
    &+h^2M_1v\left[\dfrac{1}{2}+L(T-nh)\right]\l|Y_{nh}-\vartheta_n^{\sf REM}\r|_{\Ltwo}\\
    &+h^2vM_1\left[\left(\dfrac{1}{2}+L(T-nh)\right)C_2(n)+\dfrac{1}{2}C_4+(dL(T-nh))^{1/2}\right]\\
    &+h^{3/2}M_1\sqrt{vd}\\
    &+h\bigg\{(u-v)\left[\left(\dfrac{1}{2}+L(T-nh)\right)C_2(n)+\dfrac{1}{2}C_4+(dL(T-nh))^{1/2}\right]L(T-(n+u)h)\\
    &\qquad+M_1(1+C_2(n)+C_4)\bigg\}\\
    &+h^{1/2}L(T-(n+u)h)\sqrt{(u-v)d}\\
    &+h(u-v)L(T-(n+u)h)\varepsilon_{sc}+h^2M_1v\varepsilon_{sc}.
\end{align*}
We claim that we only consider the lowest order of each part, which means the relative higher order term with the combination of $d$ and $h$ will be ignored. Then take the supremum with respect to $v$, which indicates that
\begin{equation}
    \begin{aligned}
        &\bigg\{\int_0^1\int_0^1\left\lVert \nabla\log p_{T-(n+v)h}(\vartheta_{n+v}^{\sf REM})-\nabla\log p_{T-(n+u)h}(\vartheta_{n+v}^{\sf REM})\right\rVert_{\Ltwo}^2\rmd u\rmd v\bigg\}^{1/2}\\
        \leqslant& \, h\bigg\{\int_0^1\int_0^1\left[|u-v|L(T-(n+u)h)\left(\dfrac{1}{2}+L(T-nh)\right)+M_1\right]^2\rmd u\rmd v\bigg\}^{1/2}\l|Y_{nh}-\vartheta_n^{\sf REM}\r|_{\Ltwo}\\
        &+h^{1/2}\sqrt{d}\left[\int_0^1\int_0^1L(T-(n+u)h)^2|u-v|\rmd u\rmd v\right]^{1/2}\\
        &+h\left[\int_0^1\int_0^1(u-v)^2L(T-(n+u)h)^2\rmd u\rmd v\right]^{1/2}\varepsilon_{sc}.
    \end{aligned}
    \label{eq:diffscore}
\end{equation}
Combining the above,
\begin{align*}
    &\l|\vartheta_{n+1}^{\sf REM}-\mathbb{E}_{U_n}[\vartheta_{n+1}^{\sf REM}]\r|_{\Ltwo}\\
    \leqslant & \,\,\dfrac{1}{8\sqrt{3}}h^2\left(\l|Y_{nh}-\vartheta_n^{\sf REM}\r|_{\Ltwo}+C_2(n)+C_4\right)\\
    &+\dfrac{1}{4\sqrt{3}}h^2\left[\varepsilon_{sc}+L(T-nh)\left(\l|Y_{nh}-\vartheta_n^{\sf REM}\r|_{\Ltwo}+C_2(n)\right)+(dL(T-nh))^{1/2}\right]\\
    &+\dfrac{1}{2\sqrt{3}}h^{3/2}\\
    &+ h^2\bigg\{\int_0^1\int_0^1\left[|u-v|L(T-(n+u)h)\left(\dfrac{1}{2}+L(T-nh)\right)+M_1\right]^2\rmd u\rmd v\bigg\}^{1/2}\l|Y_{nh}-\vartheta_n^{\sf REM}\r|_{\Ltwo}\\
    &+h^2\bigg\{\int_0^1\int_0^1\bigg\{(u-v)\left[\left(\dfrac{1}{2}+L(T-nh)\right)C_2(n)+\dfrac{1}{2}C_4+(dL(T-nh))^{1/2}\right]L(T-(n+u)h)\\
    &\qquad\qquad+M_1(1+C_2(n)+C_4)\bigg\}^2\rmd u\rmd v\bigg\}^{1/2}\\
    &+h^{3/2}\sqrt{d}\left[\int_0^1\int_0^1L(T-(n+u)h)^2|u-v|\rmd u\rmd v\right]^{1/2}\\
    &+h^2\left[\int_0^1\int_0^1(u-v)^2L(T-(n+u)h)^2\rmd u\rmd v\right]^{1/2}\varepsilon_{sc}\\
    &+2h\varepsilon_{sc}\\
    \lesssim& \,\,h^2\bigg\{\left[\int_0^1\int_0^1\left[|u-v|L(T-(n+u)h)\left(\dfrac{1}{2}+L(T-nh)\right)+M_1\right]^2\rmd u\rmd v\right]^{1/2}\\
    &\quad+\dfrac{1}{4\sqrt{3}}L(T-nh)+\dfrac{1}{8\sqrt{3}}\bigg\}\l|Y_{nh}-\vartheta_n^{\sf REM}\r|_{\Ltwo}\\
    &+h^2\bigg\{\bigg\{\int_0^1\int_0^1\bigg\{(u-v)\left[\left(\dfrac{1}{2}+L(T-nh)\right)C_2(n)+\dfrac{1}{2}C_4+(dL(T-nh))^{1/2}\right]L(T-(n+u)h)\\
    &\qquad\qquad\qquad+M_1(1+C_2(n)+C_4)\bigg\}^2\rmd u\rmd v\bigg\}^{1/2}\\
    &\qquad\quad+\dfrac{1}{8\sqrt{3}}(C_2(n)+C_4)+\dfrac{1}{4\sqrt{3}}(L(T-nh)C_2(n)+(dL(T-nh))^{1/2})\bigg\}\\
    &+h^{3/2}\bigg\{\sqrt{d}\left[\int_0^1\int_0^1L(T-(n+u)h)^2|u-v|\rmd u\rmd v\right]^{1/2}+\dfrac{1}{2\sqrt{3}}\bigg\}\\
    &+2h\varepsilon_{sc}.
\end{align*}
Compared to the term  $\tilde{Y}_h-\mathbb{E}_{U_n}[\vartheta_{n+1}^{\sf REM}]$, we can focus on the lower-order terms, ignoring the score matching error. Therefore, we have
\begin{align*}
    &\l|\tilde{Y}_{(n+1)h}-\vartheta_{n+1}^{\sf REM}\r|_{\Ltwo}\\
    \leqslant &\,\,h^2\Bigg\{\left[\int_0^1\int_0^1\left[|u-v|L(T-(n+u)h)\left(\dfrac{1}{2}+L(T-nh)\right)+M_1\right]^2\rmd u\rmd v\right]^{1/2}\l|Y_{nh}-\vartheta_n^{\sf REM}\r|_{\Ltwo}\\
    &\qquad+\dfrac{1}{4\sqrt{3}}L(T-nh)+\dfrac{1}{8\sqrt{3}}\Bigg\}\\
    &+h^2\Bigg\{\bigg\{\int_0^1\int_0^1\bigg\{(u-v)\left[\left(\dfrac{1}{2}+L(T-nh)\right)C_2(n)+\dfrac{1}{2}C_4+(dL(T-nh))^{1/2}\right]L(T-(n+u)h)\\
    &\qquad\qquad+M_1(1+C_2(n)+C_4)\bigg\}^2\rmd u\rmd v\bigg\}^{1/2}\\
    &\qquad \quad+\dfrac{1}{8\sqrt{3}}(C_2(n)+C_4)+\dfrac{1}{4\sqrt{3}}\left(L(T-nh)C_2(n)+(dL(T-nh))^{1/2}\right)\Bigg\}\\
    &+h^{3/2}\Bigg\{\sqrt{d}\left[\int_0^1\int_0^1L(T-(n+u)h)^2|u-v|\rmd u\rmd v\right]^{1/2}+\dfrac{1}{2\sqrt{3}}\Bigg\}\\
    &+3h\varepsilon_{sc}.
\end{align*}
\end{proof}
Returning to the proof of Theorem~\ref{thm:RMPEM}, by the conclusion of Proposition~\ref{prop:RMPEM}, we have
\begin{align*}
    \l|Y_{Nh}-\vartheta_N^{\sf REM}\r|_{\Ltwo}
    &\lesssim \dfrac{1}{m_{\min}-1/2}\left(h\max_{0\leqslant k\leqslant N-1}C_{k,1}^{\sf REM}+h^{1/2}\max_{0\leqslant k\leqslant N-1}C_{k,2}^{\sf REM}+3\varepsilon_{sc}\right)\\
    &\lesssim \sqrt{h}\cdot\dfrac{\sqrt{d/3}L_{\max}+\frac{1}{2\sqrt{3}}}{m_{\min}-1/2}+\varepsilon_{sc}\cdot\dfrac{3}{m_{\min}-1/2}.
\end{align*}
This completes the proof of Theorem~\ref{thm:RMPEM}.

\subsection{Proof of Theorem~\ref{thm:RMPEI}}
\label{proof:RMPEI}
We begin with the following proposition.
\begin{proposition}
    \label{prop:RMPEI}
    Suppose that Assumptions~\ref{asm:p0scLipx},~\ref{asm:scLipt} and~\ref{asm:score4RMP} are satisfied, the following two claims hold
    \begin{enumerate}[label=\textbf{(\arabic*)}, leftmargin=2em]
        \item \label{item:REItilde} It holds that
        \begin{align*}
            &\l|\tilde{Y}_h-\vartheta_{n+1}^{\sf REI}\r|_{\Ltwo}\\
            \leqslant&\,\,h^2\Bigg\{\int_0^1\int_0^1\bigg[|u-v|L(T-(n+u)h)\left(\dfrac{1}{2}+L(T-nh)\right)+M_1\\
            &\qquad+\dfrac{1}{2}|u-v|L(T-(n+v)h)r_n^{\sf EI}(v)\bigg]^2\rmd u\rmd v\Bigg\}^{1/2}\l|Y_{nh}-\vartheta_n^{\sf REI}\r|_{\Ltwo}\\
            &\quad +h^2\Bigg\{\dfrac{e^{\frac{1}{2}(1-v)h}-e^{\frac{1}{2}(1-u)h}}{h}e^{\frac{1}{2}vh}L(T-(n+u)h)\left[C_2(n)+C_4+2L(T-nh)C_2(n)+(dL(T-nh))^{1/2}\right]\\
            &\qquad\quad +e^{\frac{1}{2}(1-u)h}M_1\left[1+2e^{\frac{1}{2}vh}\left(L(T-nh)C_2(n)+(dL(T-nh))^{1/2}\right)+C_2(n)+C_4\right]\\
            &\qquad\quad +\dfrac{|e^{\frac{1}{2}(1-u)h}-e^{\frac{1}{2}(1-v)h}|}{h}\left(L(T-(n+v)h)C_2(n)+(dL(T-(n+v)h))^{1/2}\right)\Bigg\}\\
            &\quad +h^{3/2}\sqrt{d}\bigg\{\int_0^1\int_0^1L(T-(n+u)h)^2|u-v|\rmd u\rmd v\bigg\}^{1/2}\\
            &\quad +3h\varepsilon_{sc}.
        \end{align*}
        \item \label{item:REIYt}
        Furthermore, it holds that
        \begin{align*}
            \l|Y_{(n+1)h}-\vartheta_{n+1}^{\sf REI}\r|_{\Ltwo}\leqslant r_n^{\sf REI}\l|Y_{nh}-\vartheta_n^{\sf REI}\r|_{\Ltwo}+h^2C_{n,1}^{\sf REI}+h^{3/2}C_{n,2}^{\sf REI}+3h\varepsilon_{sc}\,,
        \end{align*}
        where
        \begin{align*}
            r_n^{\sf REI}&=e^{-\int_{nh}^{(n+1)h}(m(T-t)-\frac{1}{2})\rmd t}\\
            &\quad +h^2\bigg\{\int_0^1\int_0^1\bigg[|u-v|L(T-(n+u)h)(\dfrac{1}{2}+L(T-nh))\\
            &\qquad\qquad\qquad\qquad+M_1+\dfrac{1}{2}|u-v|L(T-(n+v)h)r_n^{\sf EI}(v)\bigg]^2\rmd u\rmd v\bigg\}^{1/2},\\
            C_{n,1}^{\sf REI}&=\dfrac{e^{\frac{1}{2}(1-v)h}-e^{\frac{1}{2}(1-u)h}}{h}e^{\frac{1}{2}vh}L(T-(n+u)h)\left[C_2(n)+C_4+2L(T-nh)C_2(n)+(dL(T-nh))^{1/2}\right]\\
            &\quad +e^{\frac{1}{2}(1-u)h}M_1\left[1+2e^{\frac{1}{2}vh}\left(L(T-nh)C_2(n)+(dL(T-nh))^{1/2}\right)+C_2(n)+C_4\right]\\
            &\quad +\dfrac{|e^{\frac{1}{2}(1-u)h}-e^{\frac{1}{2}(1-v)h}|}{h}\left(L(T-(n+v)h)C_2(n)+(dL(T-(n+v)h))^{1/2}\right),\\
            C_{n,2}^{\sf REI}&=\sqrt{d}\bigg\{\int_0^1\int_0^1L(T-(n+u)h)^2|u-v|\rmd u\rmd v\bigg\}^{1/2}.
        \end{align*}
    \end{enumerate}
\end{proposition}
\begin{proof}[Proof of Proposition~\ref{prop:RMPEI}]
This proposition can be proven following the same approach as in the proof of Proposition~\ref{prop:RMPEM}, with the only difference being the inclusion of the exponential coefficient term. 
However, this term does not significantly affect the overall proof.\\
Similarly, we make a decomposition as
\begin{align}
    \label{eq:REIdecom}
    \l|\tilde{Y}_h-\vartheta_{n+1}^{\sf REI}\r|_{\Ltwo}\leqslant&\l|\tilde{Y}_h-\mathbb{E}_{U_n}[\vartheta_{n+1}^{\sf REI}]\r|_{\Ltwo}+\l|\mathbb{E}_{U_n}[\vartheta_{n+1}^{\sf REI}]\r|_{\Ltwo}.
\end{align}
Note that
\begin{align*}
    \tilde{Y}_h-\mathbb{E}_{U_n}[\vartheta_{n+1}^{\sf REI}]&=\int_0^he^{\frac{1}{2}(h-t)}\nabla\log p_{T-nh-t}(\tilde{Y}_t)\rmd t-h\mathbb{E}_{U_n}\left[e^{\frac{1}{2}(1-U_n)h}s_*(T-nh-U_nh,\vartheta_{n+U_n}^{\sf REI})\right]\\
    &=h\int_0^1e^{\frac{1}{2}(1-u)h}\left(\nabla\log p_{T-nh-uh}(\tilde{Y}_{uh})-s_*(T-nh-uh,\vartheta_{n+u}^{\sf REI})\right)\rmd u\\
    &=h\int_0^1e^{\frac{1}{2}(1-u)h}\left(\nabla\log p_{T-nh-uh}(\tilde{Y}_{uh})-\nabla\log p_{T-nh-uh}(\vartheta_{n+u}^{\sf REI})\right)\rmd u\\
    &\quad +h\int_0^1e^{\frac{1}{2}(1-u)h}\left(\nabla\log p_{T-nh-uh}(\vartheta_{n+u}^{\sf REI})-s_*(T-nh-uh,\vartheta_{n+u}^{\sf REI})\right)\rmd u.
\end{align*}
Then, we obtain 
\begin{align*}
    &\l|\tilde{Y}_h-\mathbb{E}_{U_n}[\vartheta_{n+1}^{\sf REI}]\r|_{\Ltwo}\\
    \leqslant &\,\, h\int_0^1\left\lVert e^{\frac{1}{2}(1-u)h}(\nabla\log p_{T-nh-uh}(\tilde{Y}_{uh})-s_*(T-nh-uh,\vartheta_{n+u}^{\sf REI}))\right\rVert_{\Ltwo}\rmd u\\
    \leqslant & \,\,h\int_0^1e^{\frac{1}{2}(1-u)h}L(T-nh-uh)\l|\tilde{Y}_{uh}-\vartheta_{n+u}^{\sf REI}\r|_{\Ltwo}\rmd u+h\int_0^1e^{\frac{1}{2}(1-u)h}\rmd u\,\varepsilon_{sc}\\
    \leqslant & \,\,h^3\int_0^1e^{\frac{1}{2}(1-u)h}L(T-nh-uh)u^2\left( C_{5,n}(u)C_{1,n}(u)+M_1\dfrac{2(e^{uh/2}-1)}{uh}\right)\rmd u\,\l|Y_{nh}-\vartheta_n^{\sf REI}\r|_{\Ltwo}\\
    & +h^3\int_0^1e^{\frac{1}{2}(1-u)h}L(T-nh-uh)u^2\bigg[C_{5,n}(u)\left(C_{1,n}(u)C_2(n)+\dfrac{1}{2}C_4+C_{3,n}(u)\right)\\
    & \qquad\qquad\qquad\qquad\qquad\qquad\qquad\qquad+\dfrac{2(e^{uh}-1)}{uh}M_1(1+C_2(n)+C_4)\bigg]\rmd u\\
    &+h^{5/2}\int_0^1e^{\frac{1}{2}(1-u)h}L(T-nh-uh)u^{3/2}C_{5,n}(u)\rmd u\,\sqrt{d}\\
    & +h^2\int_0^1e^{\frac{1}{2}(1-u)h}L(T-nh-uh)\dfrac{2(e^{uh/2}-1)}{h}\rmd u\,\varepsilon_{sc}+h\dfrac{2(e^{h/2}-1)}{h}\varepsilon_{sc},
\end{align*}
where
\begin{align*}
    C_{5,n}(u)&=\dfrac{1}{uh}\int_{nh}^{(n+u)h}e^{\frac{1}{2}((n+u)h-t)}L(T-t)\rmd t.
\end{align*}
In the third inequality, we can directly bound $\l|\tilde{Y}_{uh}-\vartheta_{n+u}^{\sf REI}\r|_{\Ltwo}$, as it is a special case of Proposition~\ref{prop:EI}, where the step size is replaced by $uh$.\\
For the second term of \eqref{eq:REIdecom}, we have
\begin{align*}
    &\vartheta_{n+1}^{\sf REI}-\mathbb{E}_{U_n}[\vartheta_{n+1}^{\sf REI}]\\
    =& \,\, he^{\frac{1}{2}(1-U_n)h}s_*(T-nh-U_nh,\vartheta_{n+U}^{\sf REI})-h\mathbb{E}_{U_n}\left[e^{\frac{1}{2}(1-U_n)h}s_*(T-nh-U_nh,\vartheta_{n+U}^{\sf REI})\right]\\
    =&\,\,h\int_0^1\left[e^{\frac{1}{2}(1-U_n)h}s_*(T-nh-U_nh,\vartheta_{n+U_n}^{\sf REI})-e^{\frac{1}{2}(1-v)h}s_*(T-nh-vh,\vartheta_{n+v}^{\sf REI})\right]\rmd v\,.
\end{align*}
Similar to display~\eqref{eq:diffmatch}, we then obtain
\begin{align*}
    &\left\lVert\vartheta_{n+1}^{\sf REI}-\mathbb{E}_{U_n}[\vartheta_{n+1}^{\sf REI}]\right\rVert_{\Ltwo}\\
    \leqslant &\,\,\bigg\{\mathbb{E}\int_0^1\left[h\int_0^1e^{\frac{1}{2}(1-u)h}s_*(T-(n+u)h,\vartheta_{n+u}^{\sf REI})-e^{\frac{1}{2}(1-v)h}s_*(T-(n+v)h,\vartheta_{n+v}^{\sf REI})\rmd v\right]^2\rmd u\bigg\}^{1/2}\\
    \leqslant & \,\, h\bigg\{\int_0^1\int_0^1\left\lVert e^{\frac{1}{2}(1-u)h}s_*(T-(n+u)h,\vartheta_{n+u}^{\sf REI})-e^{\frac{1}{2}(1-v)h}s_*(T-(n+v)h,\vartheta_{n+v}^{\sf REI})\right\rVert_{\Ltwo}^2\rmd u\rmd v\bigg\}^{1/2}\\
    \leqslant & \,\, h\bigg\{\int_0^1\int_0^1\left\lVert e^{\frac{1}{2}(1-u)h}\nabla\log p_{T-(n+u)h}(\vartheta_{n+u}^{\sf REI})-e^{\frac{1}{2}(1-v)h}\nabla\log p_{T-(n+v)h}(\vartheta_{n+v}^{\sf REI})\right\rVert_{\Ltwo}^2\rmd u\rmd v\bigg\}^{1/2}\\
    &+2h\left(\int_0^1e^{(1-u)h}\rmd u\right)^{1/2}\varepsilon_{sc},
\end{align*}
Using the same strategy as in display~\eqref{eq:diffscore}, we arrive at
\begin{align*}
    &\left\lVert e^{\frac{1}{2}(1-u)h}\nabla\log p_{T-(n+u)h}(\vartheta_{n+u}^{\sf REI})-e^{\frac{1}{2}(1-v)h}\nabla\log p_{T-(n+v)h}(\vartheta_{n+v}^{\sf REI})\right\rVert_{\Ltwo}\\
    \leqslant &\,\, e^{\frac{1}{2}(1-u)h}\left\lVert \nabla\log p_{T-(n+u)h}(\vartheta_{n+u}^{\sf REI})-\nabla\log p_{T-(n+u)h}(\vartheta_{n+v}^{\sf REI})\right\rVert_{\Ltwo}\\
    &+e^{\frac{1}{2}(1-u)h}\left\lVert \nabla\log p_{T-(n+u)h}(\vartheta_{n+v}^{\sf REI})-\nabla\log p_{T-(n+v)h}(\vartheta_{n+v}^{\sf REI})\right\rVert_{\Ltwo}\\
    &+\left|e^{\frac{1}{2}(1-u)h}-e^{\frac{1}{2}(1-v)h}\right|\left\lVert \nabla\log p_{T-(n+v)h}(\vartheta_{n+v}^{\sf REI})\right\rVert_{\Ltwo}\\
    \leqslant & \,\,e^{\frac{1}{2}(1-u)h}L(T-(n+u)h)\l|\vartheta_{n+u}^{\sf REI}-\vartheta_{n+v}^{\sf REI}\r|_{\Ltwo}\\
    &+e^{\frac{1}{2}(1-u)h}M_1h(1+\l|\vartheta_{n+v}^{\sf REI}\r|_{\Ltwo})\\
    &+\left|e^{\frac{1}{2}(1-u)h}-e^{\frac{1}{2}(1-v)h}\right|\left[L(T-(n+v)h)\left(\l|Y_{(n+v)h}-\vartheta_{n+v}^{\sf REI}\r|_{\Ltwo}+C_2(n)\right)+(dL(T-(n+v)h))^{1/2}\right].
\end{align*}
The second inequality follows from Assumptions~\ref{asm:p0scLipx} and~\ref{asm:scLipt}. 
We bound the term $\left\lVert \nabla\log p_{T-(n+v)h}(\vartheta_{n+v}^{\sf REI})\right\rVert_{\Ltwo}$ by decomposing it as follows
\begin{align*}
    \left\lVert\nabla\log p_{T-(n+v)h}(\vartheta_{n+v}^{\sf REI})\right\rVert_{\Ltwo}
    &\leqslant \left\lVert\nabla\log p_{T-(n+v)h}(\vartheta_{n+v}^{\sf REI})-\nabla\log p_{T-(n+v)h}(Y_{(n+v)h})\right\rVert_{\Ltwo}\\
    &\quad+\l|\nabla\log p_{T-(n+v)h}(Y_{(n+v)h})-\nabla\log p_{T-(n+v)h}(X_{((n+v)h)}^{\leftarrow})\r|_{\Ltwo}\\
    &\quad +\l|\nabla\log p_{T-(n+v)h}(X_{T-(n+v)h})\r|_{\Ltwo}.
\end{align*}
Without loss of generality, we consider the case where $u>v$; the other case follows similarly.
\begin{equation}
    \begin{aligned}
        &\l|\vartheta_{n+u}^{\sf REI}-\vartheta_{n+v}^{\sf REI}\r|_{\Ltwo}\\
        =&\,\,(e^{\frac{1}{2}uh}-e^{\frac{1}{2}vh})\l|\vartheta_n^{\sf REI}\r|_{\Ltwo}+\int_{vh}^{uh}e^{\frac{1}{2}t}\rmd t\l|s_*(T-nh,\vartheta_n^{\sf REI})\r|_{\Ltwo}\\
        &+\left\lVert\int_{nh}^{(n+u)h}e^{\frac{1}{2}((n+u)h-t)}\rmd W_t-\int_{nh}^{(n+v)h}e^{\frac{1}{2}((n+v)h-t)}\rmd W_t\right\rVert_{\Ltwo}\\
        \leqslant& \,\,\big(e^{\frac{1}{2}(u-v)h}-1\big)e^{\frac{1}{2}vh}\left(\l|Y_{nh}-\vartheta_n^{\sf REI}\r|_{\Ltwo}+C_2(n)+C_4\right)\\
        &+2\big(e^{\frac{1}{2}(u-v)h}-1\big)e^{\frac{1}{2}vh}\left[\varepsilon_{sc}+L(T-nh)\left(\l|Y_{nh}-\vartheta_n^{\sf REI}\r|_{\Ltwo}+C_2(n)\right)+(dL(T-nh))^{1/2}\right]\\
        &+\left[(e^{uh}-1)+(e^{vh}-1)-2(e^{\frac{u+v}{2}h}-e^{\frac{u-v}{2}h})\right]^{1/2}\sqrt{d}.
    \end{aligned}
    \label{eq:yEIdiff}
\end{equation}
Here, we apply the formula in \eqref{eq:Itoiso} to bound the last term.
\begin{align*}
    &\left\lVert\int_{nh}^{(n+u)h}\textbf{1}_{\{t\leqslant(n+v)h\}}(e^{\frac{1}{2}((n+u)h-t)}-e^{\frac{1}{2}((n+v)h-t)})+\textbf{1}_{\{t>(n+v)h\}}e^{\frac{1}{2}((n+u)h-t)}\rmd W_t\right\rVert_{\Ltwo}\\
    =&\,\,\sqrt{d}\left[\int_{nh}^{(n+v)h}(e^{\frac{1}{2}((n+u)h-t)}-e^{\frac{1}{2}((n+v)h-t)})^2\rmd t+\int_{(n+v)h}^{(n+u)h}e^{(n+u)h-t}\rmd t\right]^{1/2}\\
    =&\,\,\sqrt{d}\left[(e^{uh}-1)+(e^{vh}-1)-2(e^{\frac{u+v}{2}h}-e^{\frac{u-v}{2}h})\right]^{1/2},
\end{align*}
We then bound the term $\l|\vartheta_{n+v}^{\sf REI}\r|_{\Ltwo}$ following display~\eqref{eq:yEIdiff} above. 
To this end, let $u=0$, we then have
\begin{align*}
    \l|\vartheta_{n+v}^{\sf REI}\r|_{\Ltwo}
    &\leqslant (e^{\frac{1}{2}vh}-1)(\l|Y_{nh}-\vartheta_n^{\sf REI}\r|_{\Ltwo}+C_2(n)+C_4)\\
    &\quad+2(e^{\frac{1}{2}vh}-1)\left[\varepsilon_{sc}+L(T-nh)\left(\l|Y_{nh}-\vartheta_n^{\sf REI}\r|_{\Ltwo}+C_2(n)\right)+(dL(T-nh))^{1/2}\right]\\
    &\quad +\sqrt{d}(e^{vh}-1)^{1/2}\\
    &\quad +\l|Y_{nh}-\vartheta_n^{\sf REI}\r|_{\Ltwo}+C_2(n)+C_4.
\end{align*}
Additionally, we can bound $\l|Y_{(n+v)h}-\vartheta_{n+v}^{\sf REI}\r|_{\Ltwo}$, as it is a special case of the one-step discretization error under the Exponential Integrator scheme, where the step size is replaced by $vh$.
Specifically, we have
\begin{align*}
    &\l|Y_{(n+v)h}-\vartheta_{n+v}^{\sf REI}\r|_{\Ltwo}\\
    \leqslant &\,\,r_n^{\sf EI}(v)\l|Y_{nh}-\vartheta_n^{\sf REI}\r|_{\Ltwo}+h^2C_n^{\sf EI}(v)+h^{3/2}u^{3/2}\sqrt{d}C_{5,n}(v)+vh\dfrac{2(e^{\frac{1}{2}vh}-1)}{vh}\varepsilon_{sc},
\end{align*}
where
\begin{align*}
    r_n^{\sf EI}(v)&=e^{-\int_{nh}^{(n+v)h}(m(T-t)-\frac{1}{2})\rmd t}+v^2h^2\left(C_{5,n}(v)C_{1,n}(v)+M_1\dfrac{2(e^{\frac{1}{2}vh}-1)}{vh}\right),\\
    C_n^{\sf EI}(v)&=C_{5,n}(v)\left(C_{1,n}(v)C_2(n)+\dfrac{1}{2}C_4+C_{3,n}(v)\right)+\dfrac{2(e^{\frac{1}{2}vh}-1)}{vh}M_1(1+C_2(n)+C_4).
\end{align*}
Then, we obtain
\begin{align*}
    &\left\lVert e^{\frac{1}{2}(1-u)h}\nabla\log p_{T-(n+u)h}(\vartheta_{n+u}^{\sf REI})-e^{\frac{1}{2}(1-v)h}\nabla\log p_{T-(n+v)h}(\vartheta_{n+v}^{\sf REI})\right\rVert_{\Ltwo}\\
    \leqslant & \,\,\Bigg\{2(e^{\frac{1}{2}(1-v)h}-e^{\frac{1}{2}(1-u)h})e^{\frac{1}{2}vh}L(T-(n+u)h)\left[\dfrac{1}{2}+L(T-nh)\right]\\
    &\quad +e^{\frac{1}{2}(1-u)h}M_1h\left[(e^{\frac{1}{2}vh}-1)+2(e^{\frac{1}{2}vh}-1)L(T-nh)+1\right]\\
    &\quad +\left|e^{\frac{1}{2}(1-u)h}-e^{\frac{1}{2}(1-v)h}\right|L(T-(n+v)h)r_n^{\sf EI}(v)\Bigg\}\l|Y_{nh}-\vartheta_n^{\sf REI}\r|_{\Ltwo}\\
    &+h^3L(T-(n+v)h)\dfrac{|e^{\frac{1}{2}(1-u)h}-e^{\frac{1}{2}(1-v)h}|}{h}C_n^{\sf EI}(v)\\
    &+h^{5/2}L(T-(n+v)h)\dfrac{|e^{\frac{1}{2}(1-u)h}-e^{\frac{1}{2}(1-v)h}|}{h}u^{3/2}\sqrt{d}C_{5,n}(v)\\
    &+h^2M_1e^{\frac{1}{2}(1-u)h}\dfrac{e^{\frac{1}{2}vh}-1}{h}(C_2(n)+C_4)\\
    &+h^{3/2}M_1e^{\frac{1}{2}(1-u)h}\sqrt{vd}\left(\dfrac{e^{vh}-1}{vh}\right)^{1/2}\\
    &+h\bigg\{\dfrac{e^{\frac{1}{2}(1-v)h}-e^{\frac{1}{2}(1-u)h}}{h}e^{\frac{1}{2}vh}L(T-(n+u)h)\left[C_2(n)+C_4+2L(T-nh)C_2(n)+(dL(T-nh))^{1/2}\right]\\
    &\quad\quad +e^{\frac{1}{2}(1-u)h}M_1\left[1+2e^{\frac{1}{2}vh}\left(L(T-nh)C_2(n)+(dL(T-nh))^{1/2}\right)+C_2(n)+C_4\right]\\
    &\quad\quad +\dfrac{|e^{\frac{1}{2}(1-u)h}-e^{\frac{1}{2}(1-v)h}|}{h}\left(L(T-(n+v)h)C_2(n)+(dL(T-(n+v)h))^{1/2}\right)\bigg\}\\
    &+h^{1/2}L(T-(n+u)h)e^{\frac{1}{2}(1-u)h}\left[\dfrac{(e^{uh}-1)+(e^{vh}-1)-2(e^{\frac{u+v}{2}h}-e^{\frac{u-v}{2}h})}{h}\right]^{1/2}\sqrt{d}\\
    &+h^2\varepsilon_{sc}L(T-(n+v)h)\dfrac{|e^{\frac{1}{2}(1-u)h}-e^{\frac{1}{2}(1-v)h}|}{h}\dfrac{2(e^{\frac{1}{2}vh}-1)}{h}\\
    &+h\varepsilon_{sc}\cdot 2e^{\frac{1}{2}(1-u)h}\left[L(T-(n+u)h)\dfrac{e^{\frac{1}{2}uh}-e^{\frac{1}{2}vh}}{h}+M_1e^{\frac{1}{2}vh}\right].
\end{align*}
Ignoring the higher-order terms, we take the supremum with respect to $v$ and substitute it back into the original expression, yielding
\begin{align*}
    &\l|\vartheta_{n+1}^{\sf REI}-\mathbb{E}_{U_n}[\vartheta_{n+1}^{\sf REI}]\r|_{\Ltwo}\\
    \leqslant& \,\,h\bigg\{\int_0^1\int_0^1\left\lVert e^{\frac{1}{2}(1-u)h}s_*(T-(n+u)h,\vartheta_{n+u}^{\sf REI})-e^{\frac{1}{2}(1-v)h}s_*(T-(n+v)h,\vartheta_{n+v}^{\sf REI})\right\rVert_{\Ltwo}^2\rmd u\rmd v\bigg\}^{1/2}\\
    \leqslant& \,\,h^2\bigg\{\int_0^1\int_0^1\big[|u-v|L(T-(n+u)h)\left(\dfrac{1}{2}+L(T-nh)\right)+M_1\\
    &\qquad\qquad +\dfrac{1}{2}|u-v|L(T-(n+v)h)r_n^{\sf EI}(v)\big]^2\rmd u\rmd v\bigg\}^{1/2}\l|Y_{nh}-\vartheta_n^{\sf REI}\r|_{\Ltwo}\\
    &\quad +h^{3/2}\bigg\{\int_0^1\int_0^1dL(T-(n+u)h)^2|u-v|\rmd u\rmd v\bigg\}^{1/2} +2h\varepsilon_{sc},.
\end{align*}
This completes the proof.

\end{proof}
Now, we have
\begin{align*}
    \l|Y_{Nh}-\vartheta_N^{\sf REI}\r|_{\Ltwo}
    &\lesssim \,\,\dfrac{1}{m_{\min}-1/2}\left(h\max_{0\leqslant k\leqslant N-1}C_{n,1}^{\sf REI}+\sqrt{h}\max_{0\leqslant k\leqslant N-1}C_{n,2}^{\sf REI}+3\varepsilon_{sc}\right)\\
    &\lesssim\,\,\sqrt{dh}\dfrac{L_{\max}}{\sqrt{3}(m_{\min}-1/2)}+\varepsilon_{sc}\dfrac{3}{m_{\min}-1/2}\,.
\end{align*}
The desired result follows readily.

\section{The proof of the upper bound of error of the second-order acceleration scheme}
This section is dedicated to proving the Wasserstein convergence result for second-order acceleration. To this end, we first establish the following proposition.
\begin{proposition}
    \label{prop:2order}
    Suppose that Assumptions~\ref{asm:p0scLipx},~\ref{asm:scoreerr},~\ref{asm:scerr4SO},~\ref{asm:scLipx4SO},~\ref{asm:scLipt4SO} are satisfied, the following results hold. 
    \begin{enumerate}[label=\textbf{(\arabic*)}, leftmargin=2em]
        \item \label{item:SOtilde}
        First, we have an upper bound for $\l|\tilde Y_h-\vartheta_{n+1}^{\sf SO}\r|_{\Ltwo}$ as follows,
        \begin{align*}
            \l|\tilde{Y}_h-\vartheta_{n+1}^{\sf SO}\r|_{\Ltwo}\leqslant&A_{n,1}e^{(L(nh)-\frac{1}{2})h}h^2\l|Y_{nh}-\vartheta_n^{\sf SO}\r|_{\Ltwo}+A_{n,2}e^{(L(nh)-\frac{1}{2})h}h^2\\
            &+\left(h\varepsilon_{sc}+\dfrac{2}{3}h^{3/2}\varepsilon_{sc}^{(L)}+\dfrac{1}{2}h^2\varepsilon_{sc}^{(M)}\right)e^{(L(T-nh)-\frac{1}{2})h},
        \end{align*}
        where
        \begin{align*}
            A_{n,1}=\sup_{nh\leqslant t\leqslant (n+1)h}&\dfrac{1}{t^2}\int_0^t\Big(\int_0^s\big[(1+L(T-nh-u))L(T-nh-u)\\
            &\qquad \qquad\qquad+(1+L(T-nh))L(T-nh)\big]\rmd u\Big)\rmd s,\\
            A_{n,2}=\sup_{nh\leqslant t\leqslant (n+1)h}&\dfrac{1}{t^2}\bigg[\int_0^t\Big(\int_0^s\big[(1+L(T-nh-u))L(T-nh-u)\\
            &\qquad \qquad\qquad+(1+L(T-nh))L(T-nh)\big]\rmd u\Big)\rmd s\cdot C_2(n)\\
            &\qquad+\sqrt{d}\int_0^t\Bigg(\int_0^s\Big[\Big(\dfrac{1}{2}+L(T-nh-u)\Big)L(T-nh-u)^{1/2}\\
            &\qquad\qquad \qquad\qquad+\Big(\dfrac{1}{2}+L(T-nh)\Big)L(T-nh)^{1/2}\Big]\rmd u\Bigg)\rmd s\\
            &\qquad+\int_0^t\left(\int_0^s\dfrac{1}{2}(L(T-nh-u)+L(T-nh))\rmd u\right)\rmd s\cdot C_4\bigg]\\
            &+\dfrac{\sqrt{2}}{4}L_F+\dfrac{3}{2}dL_F.
        \end{align*}
        \item \label{item:SOYt}
        Furthermore, it holds that
        \begin{align*}
            \l|Y_{(n+1)h}-\vartheta_{n+1}^{\sf SO}\r|_{\Ltwo}
            &\leqslant r_n^{\sf SO}\l|Y_{nh}-\vartheta_n^{\sf SO}\r|_{\Ltwo}+h^2C_n^{\sf SO}\\
            &\quad +\left[h\varepsilon_{sc}+\dfrac{2}{3}h^{3/2}\varepsilon_{sc}^{(L)}+\dfrac{1}{2}h^2\varepsilon_{sc}^{(M)}\right]e^{(L(T-nh)-\frac{1}{2})h},
        \end{align*}
        where
        \begin{align*}
            r_n^{\sf SO}=&e^{-\int_{nh}^{(n+1)h}(m(T-t)-\frac{1}{2})\rmd t}+h^2A_{n,1}e^{(L(T-nh)-\frac{1}{2})h}\\
            C_n^{\sf SO}=&A_{n,2}e^{(L(T-nh)-\frac{1}{2})h}\,.
        \end{align*}
    \end{enumerate}
\end{proposition}
\begin{proof}
Recall the expression in display~\eqref{eq:SOxt}, which states that
\begin{align*}
    x_t=\vartheta_n^{\sf SO}+\int_{nh}^{t}\left(\dfrac{1}{2}\vartheta_n^{\sf SO}+\nabla \log p_{T-nh}(\vartheta_n^{\sf SO})+L_n(x_s-\vartheta_n^{\sf SO})+M_n(s-nh)\right)\rmd s+\int_{nh}^{t}\rmd W_s
\end{align*}
with
\begin{align*}
    L_n&=\dfrac{1}{2}I_d+\nabla^2\log p_{T-nh}(\vartheta_n^{\sf SO})\in\mathbb{R}^{d\times d},\\
    M_n&=\dfrac{1}{2}\sum_{j=1}^d\dfrac{\partial^2}{\partial x_j^2}\nabla\log p_{T-nh}(\vartheta_n^{\sf SO})-\dfrac{\partial}{\partial t}\nabla\log p_{T-nh}(\vartheta_n^{\sf SO})\in\mathbb{R}^d.
\end{align*}
Plugging the estimates of $\nabla \log p_{T-nh}(\vartheta_n^{\sf SO}),L_n$ and $M_n$ into the previous display yields the following process for $x_t^{\sf SO}$ 
\begin{align*}
    x_t^{\sf SO}&=\vartheta_n^{\sf SO}+\int_{nh}^t\dfrac{1}{2}\vartheta_n^{\sf SO}+s_*(T-nh,\vartheta_n^{\sf SO})\rmd s\\
    &\quad +\int_{nh}^ts_*^{(L)}(T-nh,\vartheta_n^{\sf SO})(x_s^{\sf SO}-\vartheta_n^{\sf SO})+s_*^{(M)}(T-nh,\vartheta_n^{\sf SO})(s-nh)\rmd s+\int_{nh}^t\rmd W_s\,.
\end{align*}
Then, we obtain 
\begin{align*}
    x_t^{\sf SO}-x_t&=(t-nh)(s_*(T-nh,\vartheta_n^{\sf SO})-\nabla\log p_{T-nh}(\vartheta_n^{\sf SO}))\\
    &\quad +(s_*^{(L)}(T-nh,\vartheta_n^{\sf SO})-L_n)\int_{nh}^t(x_s^{\sf SO}-\vartheta_n^{\sf SO})\rmd s+L_n\int_{nh}^t(x_s^{\sf SO}-x_s)\rmd s\\
    &\quad +(s_*^{(M)}(T-nh,\vartheta_n^{\sf SO})-M_n)\cdot\dfrac{1}{2}(t-nh)^2.
\end{align*}
Notice that
\begin{align}
    \label{eq:SOterm1}
    \l|L_n\r|_{\Ltwo}=\l|\dfrac{1}{2}I_d+\nabla^2\log p_{T-nh}(\vartheta_n^{\sf SO})\r|_{\Ltwo}\leq L(T-nh)-\dfrac{1}{2}.
\end{align}
Combining this with  Assumptions~\ref{asm:scoreerr} and~\ref{asm:scerr4SO} then provides us with
\begin{align*}
    &\l|x_t^{\sf SO}-x_t\r|_{\Ltwo}\\
    \leqslant&\,\, (t-nh)\varepsilon_{sc}+\varepsilon_{sc}^{(L)}\int_{nh}^t\l|x_s^{\sf SO}-\vartheta_n^{\sf SO}\r|_{\Ltwo}\rmd s+\l|L_n\r|_{\Ltwo}\int_{nh}^t\l|x_s^{\sf SO}-x_s\r|\rmd s+\dfrac{1}{2}(t-nh)^2\varepsilon_{sc}^{(M)}\\
    \lesssim& \,\, \Big(L(T-nh)-\dfrac{1}{2}\Big)\int_{nh}^t\l|x_s^{\sf SO}-x_s\r|\rmd s+(t-nh)\varepsilon_{sc}+\dfrac{2}{3}(t-nh)^{3/2}\varepsilon_{sc}^{(L)}+\dfrac{1}{2}(t-nh)^2\varepsilon_{sc}^{(M)}.
\end{align*}
We need the following Gr{\"o}nwall-type inequality to handle this integral inequality.
\begin{lemma}
    \label{lem:Gron2}
    Let $z(t)\geqslant t_0$ satisfy the following inequality:
    \begin{align*}
        z(t)\leqslant \alpha(t)+\int_{t_0}^t\beta(s)z(s)\rmd s,\quad t\geqslant t_0,
    \end{align*}
    where $\beta(s)$ is non-negative, and $t_0$ is the initial time. Then, the solution $z(t)$ satisfies the following bound:
    \begin{align*}
        z(t)\leqslant \alpha(t)+\int_{t_0}^t\alpha(s)\beta(s)\exp\left(\int_s^t\beta(r)\rmd r\right)\rmd s,\quad t\geqslant t_0.
    \end{align*}
    Additionally, if $\alpha(t)$ is non-decreasing function, then
    \begin{align*}
        z(t)\leqslant\alpha(t)\exp\left(\int_{t_0}^t\beta(s)\rmd s\right),\quad t\geqslant t_0.
    \end{align*}
\end{lemma}
Let 
\begin{align*}
z(t)&=\l|x_t^{\sf SO}-x_t\r|_{\Ltwo}\\
\alpha(t)&=(t-nh)\varepsilon_{sc}+\dfrac{2}{3}(t-nh)^{3/2}\varepsilon_{sc}^{(L)}+\dfrac{1}{2}(t-nh)^2\varepsilon_{sc}^{(M)}\\
\beta(t)&=L(T-nh)-\dfrac{1}{2}\,,
\end{align*}
and set $t_0=nh$. By Lemma~\ref{lem:Gron2}, we have
\begin{align}
    \label{eq:SO1}
    \l|\vartheta_{n+1}^{\sf SO}-x_{(n+1)h}\r|_{\Ltwo}\leqslant\left(h\varepsilon_{sc}+\dfrac{2}{3}h^{3/2}\varepsilon_{sc}^{(L)}+\dfrac{1}{2}h^2\varepsilon_{sc}^{(M)}\right)e^{(L(T-nh)-\frac{1}{2})h}.
\end{align}
The original SDE can be rewritten as follows
\begin{align*}
\tilde{Y}_t=\tilde{Y}_0+\int_0^t\Big(\dfrac{1}{2}\tilde{Y}_s+\nabla\log p_{T-nh-s}(\tilde{Y}_s)\Big)\rmd s+\int_{nh}^{nh+t}\rmd W_s.
\end{align*}
Combining this with the definition of $x_{t}$ in~\eqref{eq:SOxt}, we then have
\begin{align*}
\tilde{Y}_t-x_{nh+t}&=\int_0^t\left(\dfrac{1}{2}\tilde{Y}_s+\nabla\log p_{T-nh-s}(\tilde{Y}_s)-\dfrac{1}{2}\vartheta_n^{\sf SO}-\nabla\log p_{T-nh}(\vartheta_n^{\sf SO})-L_n(x_{nh+s}-\vartheta_n^{\sf SO})-M_ns\right)\rmd s\\
&=\int_0^tL_n(\tilde Y_s-x_{nh+s})\rmd s+\int_0^t\left(\dfrac{1}{2}\tilde Y_s+\nabla\log p_{T-nh-s}(\tilde Y_s)-\dfrac{1}{2}\vartheta_n^{\sf SO}-\nabla\log p_{T-nh}(\vartheta_n^{\sf SO})\right)\rmd s\\
&\quad -\int_0^t\left(\int_0^s L_n\rmd \tilde Y_u\rmd u\right)\rmd s-\int_0^t\left(\int_0^s M_n\rmd u\right)\rmd s\\
&=\int_0^t L_n(\tilde Y_s-x_{nh+s})\rmd s+\int_0^t\left(\int_0^s\rmd\left(\dfrac{1}{2}\tilde Y_u+\nabla\log p_{T-nh-u}(\tilde Y_u)\right)\right)\rmd s\\
&\quad -\int_0^t\left(\int_0^s L_n\rmd \tilde Y_u\rmd u\right)\rmd s-\int_0^t\left(\int_0^s M_n\rmd u\right)\rmd s\,.
\end{align*}
We then apply the \text{It\^o} formula to the term $\rmd\left(\dfrac{1}{2}\tilde Y_u+\nabla\log p_{T-nh-u}(\tilde Y_u)\right)$. Recall the definitions of $L_n$ and $M_n$, and after rearranging the expression, we obtain
\begin{align*}
\tilde Y_t-x_{nh+t}&=\underbrace{\int_0^tL_n(\tilde{Y}_s-x_{nh+s})\rmd s}_{\text{\large I}}
+\underbrace{\int_0^t\left(\int_0^s\nabla^2\log p_{T-nh-u}(\tilde{Y}_u)-\nabla^2\log p_{T-nh}(\tilde{Y}_0)\rmd\tilde{Y}_u\right)\rmd s}_{\text{\large II}}\\
&\quad +\underbrace{\int_0^t\left(\int_0^s\sum_{j=1}^d\dfrac{\partial^2}{\partial x_j^2}\nabla\log p_{T-nh-u}(\tilde{Y}_u)-\sum_{j=1}^d\dfrac{\partial^2}{\partial x_j^2}\nabla\log p_{T-nh}(\tilde{Y}_0)\rmd u\right)\rmd s}_{\text{\large III}}\\
&\quad -\underbrace{\int_0^t\left(\int_0^s\partial_t\nabla\log p_{T-nh-u}(\tilde{Y}_u)-\partial_t\nabla\log p_{T-nh}(\tilde{Y}_0)\rmd u\right)\rmd s}_{\text{\large IV}}\,.
\end{align*}
In what follows, we derive the upper bounds for each term on the right-hand side of the previous display. 

\noindent \textbf{Upper bound for term~$\rm I$:}
The upper bound of the term~$\rm I$ follows directly from the fact that 
\begin{align*}
\l|L_n\r|_{\Ltwo}\leqslant L(T-nh)-\dfrac{1}{2}\,.
\end{align*}

\noindent \textbf{Upper bound for term~$\rm II$:}
To derive the upper bound for the second term, we expand the term $\rmd \tilde{Y}_u$, yielding
\begin{align*}
    &\left\lVert\int_0^s\nabla^2\log p_{T-nh-u}(\tilde{Y}_u)-\nabla^2\log p_{T-nh}(\tilde{Y}_0)\rmd\tilde{Y}_u\right\rVert_{\Ltwo}\\
    \leqslant&\,\,\left\lVert\int_0^s\big(\nabla^2\log p_{T-nh-u}(\tilde{Y}_u)-\nabla^2\log p_{T-nh}(\tilde{Y}_0)\big)\Big(\dfrac{1}{2}\tilde{Y}_u+\nabla\log p_{T-nh-u}(\tilde{Y}_u)\Big)\rmd u\right\rVert_{\Ltwo}\\
    &\quad +\left\lVert\int_0^s\big(\nabla^2\log p_{T-nh-u}(\tilde{Y}_u)-\nabla^2\log p_{T-nh}(\tilde{Y}_0)\big)\rmd W_u\right\rVert_{\Ltwo}\\
    \leqslant&\,\,\int_0^s\left\lVert\nabla^2\log p_{T-nh-u}(\tilde{Y}_u)-\nabla^2\log p_{T-nh}(\tilde{Y}_0)\right\lVert_{\Ltwo}\cdot\left\lVert\dfrac{1}{2}\tilde{Y}_u+\nabla\log p_{T-nh-u}(\tilde{Y}_u)\right\rVert_{\Ltwo}\rmd u\\
    &\quad +\left(\int_0^s\l|\nabla^2\log p_{T-nh-u}(\tilde{Y}_u)-\nabla^2\log p_{T-nh}(\tilde{Y}_0)\r|_{\Ltwo}^2\rmd u\right)^{1/2}.
\end{align*}
The second inequality follows from display~\eqref{eq:Itoiso}.
We note that by Assumptions~\ref{asm:scLipx4SO} and~\ref{asm:scLipt4SO}, it holds that
\begin{align*}
    &\l|\nabla^2\log p_{T-nh-u}(\tilde{Y}_u)-\nabla^2\log p_{T-nh}(\tilde{Y}_0)\r|_{\Ltwo}\\
    \leqslant& \l|\nabla^2\log p_{T-nh-u}(\tilde{Y}_u)-\nabla^2\log p_{T-nh}(\tilde{Y}_u)\r|_{\Ltwo}+\l|\nabla^2\log p_{T-nh}(\tilde{Y}_u)-\nabla^2\log p_{T-nh}(\tilde{Y}_0)\r|_{\Ltwo}\\
    \leqslant& M_2h(1+\l|\tilde{Y}_u\r|_{\Ltwo})+L_F\l|\tilde{Y}_u-\tilde{Y}_0\r|_{\Ltwo}\\
    \leqslant& M_2h+(L_F+M_2h)\l|\tilde{Y}_u-\tilde{Y}_0\r|_{\Ltwo}+M_2h\l|\vartheta_n^{\sf SO}\r|_{\Ltwo}\\
    \lesssim& L_F\sqrt{u},
\end{align*}
Combining this with the previous display provides us with
\begin{align*}
\l|\int_0^s\nabla^2\log p_{T-nh-u}(\tilde{Y}_u)-\nabla^2\log p_{T-nh}(\tilde{Y}_0)\rmd \tilde Y_u\r|_{\Ltwo}
&\lesssim \int_0^sL_F\sqrt{u}\cdot\left\lVert\dfrac{1}{2}\tilde{Y}_u+\nabla\log p_{T-nh-u}(\tilde{Y}_u)\right\rVert_{\Ltwo}\rmd u\\
&\quad +\left(\int_0^sL_F^2u\rmd u\right)^{1/2}\,.
\end{align*}
Hence, we obtain 
\begin{equation}
    \begin{aligned}
        \rm{II}&=\left\lVert\int_0^t\left(\int_0^s\nabla^2\log p_{T-nh-u}(\tilde{Y}_u)-\nabla^2\log p_{T-nh}(\tilde{Y}_0)\rmd\tilde{Y}_u\right)\rmd s\right\rVert_{\Ltwo}\\
        &\leqslant\int_0^t\left\lVert\int_0^s\nabla^2\log p_{T-nh-u}(\tilde{Y}_u)-\nabla^2\log p_{T-nh}(\tilde{Y}_0)\rmd\tilde{Y}_u\right\rVert_{\Ltwo}\rmd s\\
        &\leqslant \int_0^t\left[\int_0^sL_F\sqrt{u}\cdot\left\lVert\dfrac{1}{2}\tilde{Y}_u+\nabla\log p_{T-nh-u}(\tilde{Y}_u)\right\rVert_{\Ltwo}\rmd u+\left(\int_0^sL_F^2u\rmd u\right)^{1/2}\right]\rmd s\\
        &\lesssim \dfrac{\sqrt{2}}{4}L_Ft^2.
    \end{aligned}
    \label{eq:SOterm2}
\end{equation}

\noindent \textbf{Upper bound for term~$\rm III$:}
To derive the upper bound for term~$\rm III$, we require the following lemma to relate the Lipschitz continuity of $\nabla^2\log p_t(x)$  to the bound on the third-order derivatives of $\log p_t(x)$.
\begin{lemma}
    Let $f\in C^3(\mathbb{R}^d)$. Suppose that the Hessian matrix $\nabla^2f(x)$ is $L_F$-Lipschitz continuous with respect to the Frobenius norm. That is, for all $x,y\in\mathbb{R}^d$, the following inequality holds,
    \begin{align*}
        \l|\nabla^2f(x)-\nabla^2f(y)\r|_F\leqslant L_F\l|x-y\r|_2.
    \end{align*}
    Then, the Frobenius norm of the third-order derivative tensor $\nabla^3f(x)$ is bounded above by $\sqrt{d}L_F$ for all $x\in\mathbb{R}^d$. Formally,
    \begin{align*}
        \l|\nabla^3f(x)\r|_F\leqslant \sqrt{d}L_F,
    \end{align*}
    where each element of $\nabla^3f(x)$ is a third-order derivative of $f$.
    \label{lem:Frob}
\end{lemma}
Lemma~\ref{lem:Frob} indicates that the third-order derivatives of the score function are all bounded by the constant $L_F$, therefore we obtain the following result.
\begin{align*}
    \left\lVert\sum_{j=1}^d\dfrac{\partial^2}{\partial x_j^2}\nabla\log p_t(x)\right\rVert^2
    &=\sum_{i=1}^d\left(\sum_{j=1}^d\dfrac{\partial^3}{\partial x_i\partial x_i\partial x_j}\log p_t(x)\right)^2\\
    &\leqslant d\sum_{i=1}^d\sum_{j=1}^d\left(\dfrac{\partial^3}{\partial x_i\partial x_i\partial x_j}\log p_t(x)\right)^2\\
    &\leqslant
    d\left\lVert\nabla^3\log p_t(x)\right\rVert_F^2\\
    &\leqslant d^2L_F^2.
\end{align*}
It then follows that
\begin{equation}
    \begin{aligned}
        \rm{III}&=\l|\int_0^t\left(\int_0^s\sum_{j=1}^d\dfrac{\partial^2}{\partial x_j^2}\nabla\log p_{T-nh-u}(\tilde{Y}_u)-\sum_{j=1}^d\dfrac{\partial^2}{\partial x_j^2}\nabla\log p_{T-nh}(\tilde{Y}_0)\rmd u\right)\rmd s\r|_{\Ltwo}\\
        &\leqslant\int_0^t\int_0^s\l|\sum_{j=1}^d\dfrac{\partial^2}{\partial x_j^2}\nabla\log p_{T-nh-u}(\tilde{Y}_u)\r|_{\Ltwo}+\l|\sum_{j=1}^d\dfrac{\partial^2}{\partial x_j^2}\nabla\log p_{T-nh}(\tilde{Y}_0)\r|_{\Ltwo}\rmd u\rmd v\\
        &\leqslant \int_0^t\int_0^s 2dL_F\rmd u\rmd v\\
        &=dL_Ft^2.
    \end{aligned}
    \label{eq:SOterm3}
\end{equation}

\noindent \textbf{Upper bound for term~$\rm IV$:}
In Section~\ref{sec:accleration}, it is noted that the partial derivative of $\nabla\log p_t$ with respect to $t$ can be estimated without requiring additional assumptions. 
This is achieved by transforming the $t$-derivative into $x$-derivative via the Fokker-Planck equation, as detailed below.
\begin{align}
    \label{eq:Fokker}
    \partial_t p_t(x)=\dfrac{d}{2}p_t(x)+\dfrac{1}{2}x^\top\nabla p_t(x)+\dfrac{1}{2}\sum_{i=1}^d \dfrac{\partial^2p_t(x)}{\partial x_i^2}.
\end{align}
We need the following auxiliary lemma. 
\begin{lemma}
    \label{lem:Fokker}
    Let $p_t$ be the probability density function of $X_t$, then
    \begin{align*}
        \sum_{i=1}^d\dfrac{\partial^2 p_t(x)}{\partial x_i^2}\cdot\dfrac{1}{p_t(x)}&=\Tr\left(\nabla^2\log p_t(x)\right)+\l|\nabla\log p_t(x)\r|^2,\\
        \nabla\left(\sum_{i=1}^d\dfrac{\partial^2p_t(x)}{\partial x_i^2}\right)\cdot\dfrac{1}{p_t(x)}&=\nabla\left(\Tr(\nabla^2\log p_t(x))\right)+\nabla(\l|\nabla\log p_t(x)\r|^2)\\
        &\quad+\left[\Tr(\nabla^2\log p_t(x))+\l|\nabla\log p_t(x)\r|^2\right]\cdot\nabla\log p_t(x).
    \end{align*}
\end{lemma}
We begin by taking the gradient of $\log p_t$, and then compute the partial derivative of $\nabla \log p_t$ with respect to $t$. This results in
\begin{align*}
    \partial_t\nabla\log p_t(x)=\partial_t\left(\dfrac{\nabla p_t(x)}{p_t(x)}\right)=\dfrac{\partial_t\nabla p_t(x)}{p_t(x)}-\dfrac{\nabla p_t(x)}{p_t(x)}\cdot\dfrac{\partial_t p_t(x)}{p_t(x)}.
\end{align*}
Under certain regularity conditions, we can interchange the operators $\partial_t$ and $\nabla$ in the term $\partial_t\nabla p_t(x)$, and substitute $\partial_t p_t(x)$ by \eqref{eq:Fokker}, it follows that
\begin{align*}
    \partial_t\nabla\log p_t(x)=\dfrac{\nabla\partial_t p_t(x)}{p_t(x)}-\nabla\log p_t(x)\cdot\left(\dfrac{d}{2}+\dfrac{1}{2}x^\top\nabla\log p_t(x)+\dfrac{1}{2}\sum_{i=1}^d\dfrac{\partial^2p_t(x)}{\partial x_i^2}\cdot\dfrac{1}{p_t(x)}\right)\,.
\end{align*}
and 
\begin{align*}
    \dfrac{\nabla\partial_tp_t(x)}{p_t(x)}&=\dfrac{1}{p_t(x)}\cdot\nabla\left(\dfrac{d}{2}p_t(x)+\dfrac{1}{2}x^\top\nabla p_t(x)+\dfrac{1}{2}\sum_{i=1}^d \dfrac{\partial^2p_t(x)}{\partial x_i^2}\right)\\
    &=\dfrac{d}{2}\dfrac{\nabla p_t(x)}{p_t(x)}+\dfrac{1}{2}\dfrac{\nabla p_t(x)}{p_t(x)}+\dfrac{1}{2}\dfrac{\nabla^2 p_t(x)x}{p_t(x)}+\dfrac{1}{2}\nabla\left(\sum_{i=1}^d\dfrac{\partial^2 p_t(x)}{\partial x_i^2}\right)\cdot\dfrac{1}{p_t(x)}\\
    &=\dfrac{d+1}{2}\nabla\log p_t(x)+\dfrac{1}{2}\dfrac{\nabla^2 p_t(x)x}{p_t(x)}+\dfrac{1}{2}\nabla\left(\sum_{i=1}^d\dfrac{\partial^2 p_t(x)}{\partial x_i^2}\right)\cdot\dfrac{1}{p_t(x)}\,.
\end{align*}
Therefore, we obtain that
\begin{align*}
    \partial_t\nabla\log p_t(x)&=\dfrac{d+1}{2}\nabla\log p_t(x)+\dfrac{1}{2}\dfrac{\nabla^2 p_t(x)x}{p_t(x)}+\dfrac{1}{2}\nabla\left(\sum_{i=1}^d\dfrac{\partial^2 p_t(x)}{\partial x_i^2}\right)\cdot\dfrac{1}{p_t(x)}\\
    &\quad -\nabla\log p_t(x)\cdot\left(\dfrac{d}{2}+\dfrac{1}{2}x^\top\nabla\log p_t(x)+\dfrac{1}{2}\sum_{i=1}^d\dfrac{\partial^2p_t(x)}{\partial x_i^2}\cdot\dfrac{1}{p_t(x)}\right)\\
    &=\dfrac{1}{2}\nabla\log p_t(x)+\dfrac{1}{2}\left(\dfrac{\nabla^2p_t(x)x}{p_t(x)}-\nabla\log p_t(x)\nabla\log p_t(x)^\top x\right)\\
    &\quad +\dfrac{1}{2}\nabla\left(\sum_{i=1}^d\dfrac{\partial^2 p_t(x)}{\partial x_i^2}\right)\cdot\dfrac{1}{p_t(x)}-\dfrac{1}{2}\nabla\log p_t(x)\sum_{i=1}^d\dfrac{\partial^2p_t(x)}{\partial x_i^2}\cdot\dfrac{1}{p_t(x)}\,.
\end{align*}
By Lemma~\ref{lem:Fokker}, the last two terms above can be calculated. Additionally, it holds that
\begin{align*}
    \nabla^2\log p_t(x)=\dfrac{\nabla^2p_t(x)}{p_t(x)}-\dfrac{\nabla p_t(x)\nabla p_t(x)^\top}{p_t(x)^2}\,.
\end{align*}
Thus, $\partial_t\nabla\log p_t(x)$ can be simplified to
\begin{align*}
    \partial_t\nabla\log p_t(x)&=\dfrac{1}{2}\nabla\log p_t(x)+\dfrac{1}{2}\nabla^2\log p_t(x)x\\
    &\quad +\dfrac{1}{2}\left[\nabla\left(\Tr(\nabla^2\log p_t(x))\right)+\nabla(\l|\nabla\log p_t(x)\r|^2)\right]\\
    &\quad+\dfrac{1}{2}\left(\Tr(\nabla^2\log p_t(x))+\l|\nabla\log p_t(x)\r|^2\right)\cdot\nabla\log p_t(x)\\
    &\quad -\dfrac{1}{2}\nabla\log p_t(x)\left(\Tr(\nabla^2\log p_t(x))+\l|\nabla\log p_t(x)\r|^2\right)\\
    &=\dfrac{1}{2}\nabla\log p_t(x)+\dfrac{1}{2}\nabla^2\log p_t(x)x+\dfrac{1}{2}\nabla(\Tr(\nabla^2\log p_t(x)))+\dfrac{1}{2}\nabla(\l|\nabla\log p_t(x)\r|^2).
\end{align*}
Then, it follows that
\begin{align*}
    &\l|\partial_t\nabla\log p_{T-nh-u}(\tilde Y_u)\r|_{\Ltwo}\\
    \leqslant&\,\, \dfrac{1}{2}\l|\nabla\log p_{T-nh-u}(\tilde Y_u)\r|_{\Ltwo}+\dfrac{1}{2}\l|\nabla^2\log p_{T-nh-u}(\tilde Y_u)\r|_{\Ltwo}\l|\tilde Y_u\r|_{\Ltwo}\\
    &+\dfrac{1}{2}\l|\sum_{j=1}^d\dfrac{\partial^2}{\partial x_j^2}\nabla\log p_{T-nh-u}(\tilde Y_u)\r|_{\Ltwo}+\l|\nabla^2\log p_{T-nh-u}(\tilde Y_u)\r|_{\Ltwo}\l|\nabla\log p_{T-nh-u}(\tilde Y_u)\r|_{\Ltwo}\\
    \leqslant&\,\, \dfrac{1}{2}\l|\nabla\log p_{T-nh-u}(\tilde Y_u)\r|_{\Ltwo}+\dfrac{1}{2}L(T-nh-u)\l|\tilde Y_u\r|_{\Ltwo}\\
    &+\dfrac{1}{2}dL_F+L(T-nh-u)\l|\nabla\log p_{T-nh-u}(\tilde Y_u)\r|_{\Ltwo}\,.
\end{align*}
The second inequality follows from Assumption~\ref{asm:p0scLipx} and \eqref{eq:SOterm3}. 
The bounds for $\l|\tilde Y_u\r|_{\Ltwo}$ and $\l|\nabla\log p_{T-nh-u}(\tilde Y_u)\r|_{\Ltwo}$ can be derived according to the proof of Lemma~\ref{lem:supYtY0}. 
We then find
\begin{align}
    &\l|\partial_t\nabla\log p_{T-nh-u}(\tilde{Y}_u)\r|_{\Ltwo}\\
    \leqslant&\,\,\dfrac{1}{2}\l|\nabla\log p_{T-nh-u}(\tilde{Y}_u)\r|_{\Ltwo}+\dfrac{1}{2}L(T-nh-u)\l|\tilde{Y}_u\r|_{\Ltwo}+\dfrac{1}{2}dL_F+L(T-nh-u)\l|\nabla\log p_{T-nh-u}(\tilde{Y}_u)\r|_{\Ltwo}\\
    \leqslant&\,\,\Big(\dfrac{1}{2}+L(T-nh-u)\Big)\left[L(T-nh-u)(\l|Y_{nh}-\vartheta_n^{\sf SO}\r|_{\Ltwo}+C_2(n))+\big(dL(T-nh-u)\big)^{1/2}\right]\\
    &+\dfrac{1}{2}L(T-nh-u)\Big(\l|Y_{nh}-\vartheta_n^{\sf SO}\r|_{\Ltwo}+C_2(n)+C_4\Big)\\
    &+\dfrac{1}{2}dL_F\,.
\end{align}
Therefore, we obtain
\begin{equation}
    \begin{aligned}
        &\l|\int_0^t\left(\int_0^s\partial_t\nabla\log p_{T-nh-u}(\tilde{Y}_u)-\partial_t\nabla\log p_{T-nh}(\tilde{Y}_0)\rmd u\right)\rmd s\r|_{\Ltwo}\\
        \leqslant& \int_0^t\int_0^s\l|\partial_t\nabla\log p_{T-nh-u}(\tilde{Y}_u)\r|_{\Ltwo}+\l|\partial_t\nabla\log p_{T-nh}(\tilde{Y}_0)\r|_{\Ltwo}\rmd u\rmd s\\
        \leqslant& \int_0^t\left(\int_0^s\left[(1+L(T-nh-u))L(T-nh-u)+(1+L(T-nh))L(T-nh)\right]\rmd u\right)\rmd s\cdot \l|Y_{nh}-\vartheta_n^{\sf SO}\r|_{\Ltwo}\\
        &+\int_0^t\left(\int_0^s\left[(1+L(T-nh-u))L(T-nh-u)+(1+L(T-nh))L(T-nh)\right]\rmd u\right)\rmd s\cdot C_2(n)\\
        &+\sqrt{d}\int_0^t\left(\int_0^s\left[(\dfrac{1}{2}+L(T-nh-u))L(T-nh-u)^{1/2}+(\dfrac{1}{2}+L(T-nh))L(T-nh)^{1/2}\right]\rmd u\right)\rmd s\\
        &+\int_0^t\left(\int_0^s\dfrac{1}{2}(L(T-nh-u)+L(T-nh))\rmd u\right)\rmd s\cdot C_4\\
        &+\dfrac{1}{2}dL_Ft^2\,.   \label{eq:SOterm4}
    \end{aligned}
\end{equation}
For simplicity, we focus on the lowest-order term. Recall equations \eqref{eq:SOterm1}, \eqref{eq:SOterm2}, \eqref{eq:SOterm3}, and \eqref{eq:SOterm4}, which lead to the following expression
\begin{align*}
    \l|\tilde{Y}_t-x_{nh+t}\r|_{\Ltwo}\leqslant& \Big(L(T-nh)-\dfrac{1}{2}\Big)\int_0^t\l|\tilde{Y}_s-x_{nh+s}\r|_{\Ltwo}\rmd s+\Big(A_{n,1}\l|Y_{nh}-\vartheta_n^{\sf SO}\r|_{\Ltwo}+A_{n,2}\Big)\cdot t^2,
\end{align*}
where
\begin{align*}
    A_{n,1}=\sup_{nh\leqslant t\leqslant (n+1)h}&\dfrac{1}{t^2}\int_0^t\left(\int_0^s\left[(1+L(T-nh-u))L(T-nh-u)+(1+L(T-nh))L(T-nh)\right]\rmd u\right)\rmd s,\\
    A_{n,2}=\sup_{nh\leqslant t\leqslant (n+1)h}&\dfrac{1}{t^2}\bigg[\int_0^t\left(\int_0^s\left[(1+L(T-nh-u))L(T-nh-u)+(1+L(T-nh))L(T-nh)\right]\rmd u\right)\rmd s\cdot C_2(n)\\
    &+\sqrt{d}\int_0^t\left(\int_0^s\left[(\dfrac{1}{2}+L(T-nh-u))L(T-nh-u)^{1/2}+(\dfrac{1}{2}+L(T-nh))L(T-nh)^{1/2}\right]\rmd u\right)\rmd s\\
    &+\int_0^t\left(\int_0^s\dfrac{1}{2}(L(T-nh-u)+L(T-nh))\rmd u\right)\rmd s\cdot C_4\bigg]\\
    &+\dfrac{\sqrt{2}}{4}L_F+\dfrac{3}{2}dL_F.
\end{align*}
By Lemma~\ref{lem:Gron2}, let
\begin{align*}
    z(t)&=\l|\tilde Y_t-x_{nh+t}\r|_{\Ltwo}\\
    \alpha(t)&=(A_{n,1}\l|Y_{nh}-\vartheta_n^{\sf SO}\r|_{\Ltwo}+A_{n,2})\cdot t^2\\
    \beta(t)&=L(T-nh)-\dfrac{1}{2}
\end{align*}
set $t_0=nh$, we then obtain
\begin{align*}
    \l|\tilde{Y}_h-x_{(n+1)h}\r|_{\Ltwo}&\leqslant (A_{n,1}\l|Y_{nh}-\vartheta_n^{\sf SO}\r|_{\Ltwo}+A_{n,2})h^2\exp\left((L(T-nh)-\dfrac{1}{2})h\right)\\
    &=A_{n,1}e^{(L(T-nh)-\frac{1}{2})h}h^2\l|Y_{nh}-\vartheta_n^{\sf SO}\r|_{\Ltwo}+A_{n,2}e^{(L(T-nh)-\frac{1}{2})h}h^2.
\end{align*}
Invoking display~\eqref{eq:SO1}, we  arrive at
\begin{align*}
    \l|\tilde{Y}_h-\vartheta_{n+1}^{\sf SO}\r|_{\Ltwo}&\leqslant\l|\tilde{Y}_h-x_{(n+1)h}\r|_{\Ltwo}+\l|\vartheta_{n+1}^{\sf SO}-x_{(n+1)h}\r|_{\Ltwo}\\
    &\leqslant A_{n,1}e^{(L(T-nh)-\frac{1}{2})h}h^2\l|Y_{nh}-\vartheta_n^{\sf SO}\r|_{\Ltwo}+A_{n,2}e^{(L(T-nh)-\frac{1}{2})h}h^2\\
    &\quad +\left[h\varepsilon_{sc}+\dfrac{2}{3}h^{3/2}\varepsilon_{sc}^{(L)}+\dfrac{1}{2}h^2\varepsilon_{sc}^{(M)}\right]e^{(L(T-nh)-\frac{1}{2})h}
\end{align*}
Furthermore, we can bound the coefficients~$A_{n,1}$ and $A_{n,2}$ as follows.
\begin{align*}
    A_{n,1}&\leqslant\dfrac{1}{t^2}\int_0^t\left(\int_0^s2(1+L_{\max})L_{\max}\rmd u\right)\rmd s=(1+L_{\max})L_{\max},\\
    A_{n,2}&\leqslant(1+L_{\max})L_{\max}C_2(n)+\sqrt{d}(\dfrac{1}{2}+L_{\max})L_{\max}^{1/2}+\dfrac{1}{2}L_{\max}C_4+\dfrac{\sqrt{2}}{4}L_F+\dfrac{3}{2}dL_F\\
    &\lesssim\sqrt{d}(\dfrac{1}{2}+L_{\max})L_{\max}^{1/2}+\dfrac{3}{2}dL_F.
\end{align*}
Collecting all the pieces then gives
\begin{align*}
    \l|Y_{nh}-\vartheta_{n+1}^{\sf SO}\r|_{\Ltwo}\leqslant r_n^{\sf SO}\l|Y_{nh}-\vartheta_n^{\sf SO}\r|_{\Ltwo}+C_n^{\sf SO}h^2+\left[h\varepsilon_{sc}+\dfrac{2}{3}h^{3/2}\varepsilon_{sc}^{(L)}+\dfrac{1}{2}h^2\varepsilon_{sc}^{(M)}\right]e^{(L(T-nh)-\frac{1}{2})h},
\end{align*}
where
\begin{align*}
    r_n^{\sf SO}&=e^{-\int_{nh}^{(n+1)h}(m(T-t)-\frac{1}{2})\rmd t}+A_{n,1}e^{(L(T-nh)-\frac{1}{2})h}h^2,\\
    C_n^{\sf SO}&=A_{n,2}e^{(L(T-nh)-\frac{1}{2})h}\,.
\end{align*}

\end{proof}
From the result above, we finally obtain
\begin{align*}
    \l|Y_{Nh}-\vartheta_N^{\sf SO}\r|_{\Ltwo}&\lesssim \dfrac{1}{m_{\min}-1/2}\left[h\max_{0\leqslant k\leqslant N-1}C_k^{\sf SO}+\left(\varepsilon_{sc}+\dfrac{2}{3}h^{1/2}\varepsilon_{sc}^{(L)}+\dfrac{1}{2}h\varepsilon_{sc}^{(M)}\right)e^{(L_{\max}-\frac{1}{2})h}\right]\\
    &\lesssim h\cdot\dfrac{(\sqrt{d}L_{\max}^{3/2}+3dL_F/2)e^{(L_{\max}-\frac{1}{2})h}}{m_{\min}-1/2}+\left(\varepsilon_{sc}+\dfrac{2}{3}\sqrt{h}\varepsilon_{sc}^{(L)}+\dfrac{1}{2}h\varepsilon_{sc}^{(M)}\right)e^{(L_{\max}-\frac{1}{2})h}\,.
\end{align*}
This completes the proof of Theorem~\ref{thm:2order}.

\section{Proof of Auxiliary Lemma}
\subsection{Proof of Lemma~\ref{lem:Gaomt}}
We have
\begin{align*}
    \dfrac{\rmd\l|H_t-G_t\r|^2}{\rmd t}&=2 \inprod{H_t-G_t}{\dfrac{\rmd(H_t-G_t)}{\rmd t}}\\
    &=2 \inprod{ H_t-G_t}{\dfrac{1}{2}(H_t-G_t)+\big(\nabla\log p_{T-t}(H_t)-\nabla\log p_{T-t}(G_t)\big)}\\
    &=\l|H_t-G_t\r|^2+2\langle H_t-G_t,\nabla\log p_{T-t}(H_t)-\nabla\log p_{T-t}(G_t)\rangle\\
    &\leqslant \big(1-2m(T-t)\big)\l|H_t-G_t\r|^2.
\end{align*}
The last inequality follows from Lemma~\ref{lem:GaoLt}. 
Then, we take the derivative of $e^{-\int_{t_1}^t(2m(T-s)-1)\rmd s}\l|H_t-G_t\r|^2$
\begin{align*}
    &\dfrac{\rmd}{\rmd t} \Big(e^{\int_{t_1}^t(2m(T-s)-1)\rmd s}\l|H_t-G_t\r|^2\Big)\\
    =&\,\,(2m(T-t)-1)e^{-\int_{t_1}^t(2m(T-s)-1)\rmd s}\l|H_t-G_t\r|^2+e^{-\int_{t_1}^t(2m(T-s)-1)\rmd s}\dfrac{\rmd\l|H_t-G_t\r|^2}{\rmd t}\\
    \leqslant& \,\,0.
\end{align*}
Therefore, we obtain that
\begin{align*}
    e^{\int_{t_1}^t(2m(T-s)-1)\rmd s}\l|H_t-G_t\r|^2\leqslant \l|H_{t_t}-G_{t_1}\r|^2. 
\end{align*}
taking the expectation of both sides and then applying the square root yields the desired result.

\subsection{Proof of Lemma~\ref{lem:supYtY0}}
By the definition of $\tilde{Y}_h$, we have
\begin{align*}
    \l|\tilde{Y}_t-\tilde{Y}_0\r|_{\Ltwo}
    &=\left\lVert\int_0^t(\dfrac{1}{2}\tilde{Y}_s+\nabla\log p_{T-nh-s}(\tilde{Y}_s))\rmd s+\int_{nh}^{nh+t}\rmd W_s\right\rVert_{\Ltwo}\\
   & \leqslant \int_0^t\dfrac{1}{2}\l|\tilde{Y}_s\r|_{\Ltwo}\rmd t+\int_0^t\l|\nabla\log p_{T-nh-s}(\tilde{Y}_s)\r|_{\Ltwo}\rmd s+\left\lVert\int_{nh}^{nh+t}\rmd W_s\right\rVert_{\Ltwo}\,.
\end{align*}
To bound the first term, we observe that for any $s\in[0,h]$, the following holds
\begin{align*}
\l|\tilde{Y}_s\r|_{\Ltwo}
&\leqslant \l|Y_{nh+s}\r|_{\Ltwo}+\l|\tilde{Y}_s-Y_{nh+s}\r|_{\Ltwo}\\
    &\leqslant \l|Y_{nh+s}-X_{nh+s}^{\leftarrow}\r|_{\Ltwo}+\l|X_{nh+s}^{\leftarrow}\r|_{\Ltwo}+\l|\tilde{Y}_s-Y_{nh+s}\r|_{\Ltwo}\\
    &\leqslant e^{-\int_0^{nh+s}(m(T-t)-\frac{1}{2})\rmd t}\l|Y_0-X_0^{\leftarrow}\r|_{\Ltwo}+\l|X_{T-(nh+s)}\r|_{\Ltwo}+e^{-\int_{nh}^s(m(T-u)-\frac{1}{2})\rmd u}\l|Y_{nh}-\vartheta_n^{\sf EM}\r|_{\Ltwo}\\
    &\leqslant  e^{-\int_0^{nh}(m(T-t)-\frac{1}{2})\rmd t}\l|Y_0-X_T\r|_{\Ltwo}+\sup_{0\leqslant t\leqslant T}\l|X_t\r|_{\Ltwo}+\l|Y_{nh}-\vartheta_n^{\sf EM}\r|_{\Ltwo}\,.
\end{align*}
Here, the second inequality follows from the \text{Gr{\"o}nwall} inequality applied on $\l|Y_{nh+s}-X_{nh+s}^{\leftarrow}\r|_{\Ltwo}$ and $\l|\tilde{Y}_s-Y_{nh+s}\r|_{\Ltwo}$, and the fact that $\l|X_t\r|_{\Ltwo}=\l|X_{T-t}^{\leftarrow}\r|_{\Ltwo}$.
To bound the second term, we need the following lemma.
\begin{lemma}
    \label{lem:Enabla}
    If the target distribution~$p_0$ satisfies Assumption~\ref{asm:p0scLipx}, it holds that
    \begin{align*}
        \l|\nabla\log p_{t}(X_t)\r|_{\Ltwo}\leqslant (dL(t))^{1/2}.
    \end{align*}
\end{lemma}
According to Lemma~\ref{lem:Enabla}, it follows that
\begin{align*}
    &\l|\nabla\log p_{T-nh-s}(\tilde{Y}_s)\r|_{\Ltwo}\\
    \leqslant&\,\,\l|\nabla\log p_{T-nh-s}(\tilde{Y}_s)-\nabla\log p_{T-nh-s}(X_{nh+s}^{\leftarrow})\r|_{\Ltwo}+\l|\nabla\log p_{T-nh-s}(X_{nh+s}^{\leftarrow})\r|_{\Ltwo}\\
    \leqslant& \,\,L(T-nh-s)\l|\tilde{Y}_s-X_{nh+s}^{\leftarrow}\r|_{\Ltwo}+(dL(T-nh-s))^{1/2}\\
    \leqslant&\,\,L(T-nh-s)\l|\tilde{Y}_0-X_{nh}^{\leftarrow}\r|_{\Ltwo}+(dL(T-nh-s))^{1/2}\\
    \leqslant&\,\,L(T-nh-s)\left(\l|\tilde{Y}_0-Y_{nh}\r|_{\Ltwo}+\l|Y_{nh}-X_{nh}^{\leftarrow}\r|_{\Ltwo}\right)+(dL(T-nh-s))^{1/2}\\
    \leqslant& \,\,L(T-nh-s)\left(\l|Y_{nh}-\vartheta_n^{\sf EM}\r|_{\Ltwo}+e^{-\int_0^{nh}(m(T-t)-\frac{1}{2})\rmd t}\l|Y_0-X_T\r|_{\Ltwo})+(dL(T-nh-s)\right)^{1/2}.
\end{align*}
Here, we use the fact that $\tilde{Y}_0=\vartheta_n^{\sf EM}$, and \text{Gr{\"o}nwall} inequality are used in the third inequality and the last one. This completes the proof.

\subsection{Proof of Lemma~\ref{lem:Brown1}}
For the stochastic integral of process $X$, we have
\begin{align*}
    \mathbb{E}(I_t(X))^2=\mathbb{E}\int_0^tX_u^2d\langle M \rangle_u.
\end{align*}
Then, we obtain 
\begin{align*}
    &\left\lVert\int_{nh}^{(n+U_n)h}\rmd W_t-\int_0^1\left(\int_{nh}^{(n+u)h}\rmd W_t\right)\rmd u\right\rVert_{\Ltwo}^2\\
    =&\,\,\mathbb{E}\left(\int_0^1\int_{nh}^{(n+1)h}-\textbf{1}_{\{U_n\leqslant u\}}\textbf{1}_{\{(n+U_n)h\leqslant t\leqslant (n+u)h\}}+\textbf{1}_{\{U_n>u\}}\textbf{1}_{\{(n+u)h\leqslant t\leqslant (n+U_n)h\}}\rmd W_t\rmd u\right)^2\\
    \leqslant&\,\,\int_0^1\mathbb{E}\left(\int_{nh}^{(n+1)h}-\textbf{1}_{\{U_n\leqslant u\}}\textbf{1}_{\{(n+U_n)h\leqslant t\leqslant (n+u)h\}}+\textbf{1}_{\{U_n>u\}}\textbf{1}_{\{(n+u)h\leqslant t\leqslant (n+U_n)h\}}\rmd W_t\right)^2\rmd u\\
    = & \,\,\int_0^1\left(\mathbb{E}\int_{nh}^{(n+1)h}\textbf{1}_{\{U_n\leqslant u\}}\textbf{1}_{\{(n+U_n)h\leqslant t\leqslant (n+u)h\}}+\textbf{1}_{\{U_n>u\}}\textbf{1}_{\{(n+u)h\leqslant t\leqslant(n+U_n)h\}}\rmd t\right)\rmd u\\
    = & \,\,\int_0^1\left(\mathbb{E}\left(\textbf{1}_{\{U_n\leqslant u\}}(u-U_n)h+\textbf{1}_{\{U_n>u\}}(U_n-u)h\right)\right)\rmd u\\
    =&\,\,h\int_0^1\Big(u^2-u+\dfrac{1}{2}\Big)\rmd u\\
    =&\,\,\dfrac{1}{3}h.
\end{align*}

\subsection{Proof of Lemma~\ref{lem:Gron2}}
Define the function $w(s)$ via
\begin{align*}
    w(s)=\exp\left(-\int_{t_0}^s\beta(r)dr\right)\int_{t_0}^s\beta(r)z(r)\rmd r,\quad \forall s\geqslant t_0.
\end{align*}
Differentiating this function gives
\begin{align*}
    w'(s)=\left(z(s)-\int_{t_0}^s\beta(r)z(r)\rmd r\right)\beta(s)\exp\left(-\int_{t_0}^s\beta(r)\rmd r\right)\leqslant \alpha(s)\beta(s)\exp\left(-\int_{t_0}^s\beta(r)\rmd r\right).
\end{align*}
Note that $w(t_0)=0$. Integrating the function $w$ from $t_0$ to $t$ yields
\begin{align*}
    w(t)\leqslant \int_{t_0}^t\alpha(s)\beta(s)\exp\left(-\int_{t_0}^s\beta(r)\rmd r\right)\rmd s.
\end{align*}
By the definition of $w(s)$, we also have
\begin{align*}
    \int_{t_0}^t\beta(s)z(s)\rmd s=\exp\left(\int_{t_0}^t\beta(r)\rmd r\right)w(t).
\end{align*}
Combining the previous two displays provides us with
\begin{align*}
    \int_{t_0}^t\beta(s)z(s)\rmd s\leqslant\int_{t_0}^t\alpha(s)\beta(s)\exp\left(\int_s^t\beta(r)\rmd r\right)\rmd s.
\end{align*}
By substituting this estimate into the inequality, we can obtain the first desired result.
Furthermore, if $\alpha$ is non-decreasing, then for any $s\leqslant t$, it holds that $\alpha(s)\leqslant \alpha(t)$.
This leads to 
\begin{align*}
    z(t)\leqslant \alpha(t)+\alpha(t)\int_{t_0}^t\beta(s)\exp\left(\int_s^t\beta(r)\rmd r\right)\rmd s.
\end{align*}
which can be simplified to
\begin{align*}
    z(t)\leqslant \alpha(t)\exp\left(\int_{t_0}^t\beta(r)\rmd r\right),\quad t\geqslant t_0.
\end{align*}
This completes the proof.

\subsection{Proof of Lemma~\ref{lem:Frob}}
Let $f \in C^3(\mathbb{R}^d)$ and assume that its Hessian matrix is $L_F$-Lipschitz continuous with respect to the Frobenius norm; that is, for all $x,y \in \mathbb{R}^d$ we have
\begin{align*}
    \l|\nabla^2 f(x) - \nabla^2 f(y)\r|_F \leqslant L_F \l|x-y\r|_2.
\end{align*}
Denote by $\{e_1, e_2, \cdots, e_d\}$ the standard orthonormal basis of $\mathbb{R}^d$. For any fixed $x \in \R^d$ and any index $i \in \{1,\cdots,d\}$, consider the directional derivative of the Hessian in the direction $e_i$:
\begin{align*}
    \partial_i \nabla^2 f(x) := \lim_{h \to 0} \frac{\nabla^2 f(x+he_i) - \nabla^2 f(x)}{h}.
\end{align*}
By the Lipschitz continuity of $\nabla^2 f$, for any nonzero $h$ we have
\begin{align*}
\l|\dfrac{\nabla^2 f(x+he_i) - \nabla^2 f(x)}{h}\r|_F \leqslant \dfrac{L_F |h|}{|h|} = L_F.
\end{align*}
Taking the limit as $h \to 0$ yields
\begin{align*}
    \l|\partial_i \nabla^2 f(x)\r|_F \leqslant L_F.
\end{align*}
Note that the matrix $\partial_i \nabla^2 f(x)$ is the $i$-th slice of the third-order derivative tensor $\nabla^3 f(x)$, its $(j,k)$-entry is $\dfrac{\partial^3 f(x)}{\partial x_i\,\partial x_j\,\partial x_k}$.

By the definition of Frobenius norm, we have
\begin{align*}
    \l|\nabla^3 f(x)\r|_F^2 = \sum_{i=1}^d \l|\partial_i \nabla^2 f(x)\r|_F^2\,.
\end{align*}
It then follows that
\begin{align*}
\l|\nabla^3 f(x)\r|_F^2 \leqslant \sum_{i=1}^d L_F^2 = d\,L_F^2.
\end{align*}
Taking the square root of both sides, we obtain
\begin{align*}
\l|\nabla^3 f(x)\r|_F \leqslant \sqrt{d}\, L_F.
\end{align*}
This completes the proof.

\subsection{Proof of Lemma~\ref{lem:Fokker}}
Notice that
\begin{align*}
    \nabla^2\log p_t(x)=&-\dfrac{1}{p_t(x)^2}\nabla p_t(x)\nabla p_t(x)^\top+\dfrac{1}{p_t(x)}\nabla^2p_t(x)\\
    =&-\nabla\log p_t(x)\nabla\log p_t(x)^\top+\dfrac{1}{p_t(x)}\nabla^2 p_t(x),
\end{align*}
which indicates
\begin{align*}
    \dfrac{1}{2}\sum_{i=1}^d\dfrac{\partial^2 p_t(x)}{\partial x_i^2}\cdot\dfrac{1}{p_t(x)}
    &=\dfrac{1}{2}\Tr\left(\dfrac{1}{p_t(x)}\nabla^2p_t(x)\right)\\
    &=\dfrac{1}{2}\Tr\Big(\nabla^2\log p_t(x)+\nabla\log p_t(x)\nabla\log p_t(x)^\top\Big)\\
    &=\dfrac{1}{2}\Tr\big(\nabla^2\log p_t(x)\big)+\dfrac{1}{2}\l|\nabla\log p_t(x)\r|^2,
\end{align*}
Additionally, we have
\begin{align*}
    \nabla\left(\dfrac{\partial^2\log p_t(x)}{\partial x_i^2}\right)&=\nabla\left(\dfrac{\partial^2p_t(x)}{\partial x_i^2}\cdot\dfrac{1}{p_t(x)}-\left(\dfrac{\partial p_t(x)}{\partial x_i}\cdot\dfrac{1}{p_t(x)}\right)^2\right)\\
    &=\nabla\left(\dfrac{\partial^2p_t(x)}{\partial x_i^2}\right)\cdot\dfrac{1}{p_t(x)}-\dfrac{\partial^2p_t(x)}{\partial x_i^2}\cdot\dfrac{1}{p_t(x)}\cdot\nabla\log p_t(x)-\nabla\left(\left(\dfrac{\partial\log p_t(x)}{\partial x_i}\right)^2\right)\,.
\end{align*}
Then, we obtain
\begin{align*}
    &\nabla\left(\sum_{i=1}^d\dfrac{\partial^2p_t(x)}{\partial x_i^2}\right)\cdot\dfrac{1}{p_t(x)}\\
    =&\nabla\left(\Tr(\nabla^2\log p_t(x))\right)+\left[\Tr(\nabla^2\log p_t(x))+\l|\nabla\log p_t(x)\r|^2\right]\cdot\nabla\log p_t(x)+\nabla(\l|\nabla\log p_t(x)\r|^2).
\end{align*}

\subsection{Proof of Lemma~\ref{lem:Enabla}}
Note that
\begin{align*}
    \mathbb{E}(\l|\nabla\log p_t(X_t)\r|^2)=&\int_{\mathbb{R}^d}\l|\nabla\log p_t(x)\r|^2p_t(x)\rmd x\\
    =&\lim_{R\to\infty}\int_{B(0,R)}\langle\nabla\log p_t(x),\nabla\log p_t(x)\rangle p_t(x)\rmd x\\
    =&\lim_{R\to\infty}\int_{B(0,R)}\langle\nabla\log p_t(x),\nabla p_t(x)\rangle \rmd x\,,
\end{align*}
where $B(0,R)$ denotes the Euclidean ball with radius $R>0$ centered at the origin.
Using integration by parts, we then obtain
\begin{align*}
    \mathbb{E}(\l|\nabla\log p_t(X_t)\r|^2)&=\lim_{R\to\infty}\int_{B(0,R)}-p_t(x)\Delta\log p_t(x)\rmd x+\int_{\partial B(0,R)}p_t(x)\dfrac{\partial \log p_t(x)}{\partial \vec{n}}\rmd S\\
    &=\int_{\mathbb{R}^d}p_t(x)\cdot(-\Delta\log p_t(x))\rmd x\\
    &\leqslant d L(t),
\end{align*}
where $\dfrac{\partial f}{\partial\vec{n}}=\nabla f\cdot\vec{n}$ represents the directional derivative along the normal vector $\vec{n}$ and $\rmd S$ denotes the surface integral over the spherical surface. Here we use the fact that $p_t(x)$ converges to $0$ at an exponential rate as $\l|x\r|$ approaches infinity, and the fact that
\begin{align*}
-\Delta\log p_t(x)=-\Tr(\nabla^2\log p_t(x))\in [0,dL(t)]  \,,
\end{align*}
which follows from Lemma~\ref{lem:GaoLt}.

\section{Numerical Studies}
\label{app:simulation}
We apply the five schemes to the posterior density of penalized logistic regression, defined by $p_0(\theta)\propto \exp(-f(\theta))$
with the potential function
\begin{align*}
	f(\theta)=\dfrac{\lambda}{2}\|\theta\|^2+\dfrac{1}{n_{\sf data}}\sum_{i=1}^{n_{\sf data}}\log(1+\exp(-y_ix_i^\top \theta)),
\end{align*}
where $\lambda>0$ denotes the tuning parameter. The data $\{x_i,y_i\}_{i=1}^{n_{\sf data}}$, composed of binary labels $y_i\in\{-1,1\}$ and features $x_i\in\mathbb{R}^d$ generated from $x_{i,j}\mathop{\sim}\limits^{iid}\mathcal{N}(0,\sigma^2)$.

\subsection{Calculation}
In this part, we derive explicit formulas for each coefficient term we need. First, the score function can be computed as
\begin{align*}
    \nabla\log p_0(\theta)
    &=-\left(\lambda \theta+\dfrac{1}{n_{\sf data}}\sum_{i=1}^{n_{\sf data}}\dfrac{-y_ix_i\exp(-y_ix_i^\top \theta)}{1+\exp(-y_ix_i^\top \theta)}\right)\\
    &=-\left(\lambda \theta+\dfrac{1}{n_{\sf data}}\sum_{i=1}^{n_{\sf data}}\dfrac{-y_ix_i}{1+\exp(y_ix_i^\top \theta)}\right).
\end{align*}
For simplicity, we denote the logistic sigmoid function $\sigma(u)=\dfrac{1}{1+e^{-u}}$, then
\begin{align*}
	\nabla\log p_0(\theta)=-\left(\lambda \theta+\dfrac{1}{n_{\sf data}}\sum_{i=1}^{n_{\sf data}}-y_ix_i\sigma(-y_ix_i^\top \theta)\right).
\end{align*}
Since $\sigma'(u)=\sigma(u)[1-\sigma(u)]$, we have
\begin{align*}
	\nabla^2\log p_0(\theta)
    =&-\left(\lambda I_d+\dfrac{1}{n_{\sf data}}\sum_{i=1}^{n_{\sf data}}y_i^2\sigma(-y_ix_i^\top \theta)\left[1-\sigma(-y_ix_i^\top \theta)\right]x_ix_i^\top\right)\\
	=&-\lambda I_d-\dfrac{1}{n_{\sf data}}\sum_{i=1}^{n_{\sf data}}\sigma(-y_ix_i^\top \theta)\left[1-\sigma(-y_ix_i^\top \theta)\right]x_ix_i^\top.
\end{align*}
As $x_ix_i^\top\succcurlyeq 0$, $\nabla^2\log p_0(\theta)\preccurlyeq-\lambda I_d$. We also have that $\sigma(1-y_ix_i^\top \theta)\in(0,1)$, then
\begin{align*}
	\nabla^2\log p_0(\theta)\succcurlyeq&-\lambda I_d-\dfrac{1}{4n_{\sf data}}\sum_{i=1}^{n_{\sf data}}x_ix_i^\top\\
    \succcurlyeq&-(\lambda+\dfrac{1}{n_{\sf data}}\lambda_{\max}(\sum_{i=1}^{n_{\sf data}}x_ix_i^\top))I_d\,.
\end{align*}
Therefore, 
\begin{align*}
	m_0=\lambda,\quad L_0=\lambda+\dfrac{1}{n_{\sf data}}\lambda_{\max}(\sum_{i}^{n_{\sf data}}x_ix_i^\top).
\end{align*}
Recall that the transition probability $p_{t|0}(\theta_t|\theta_0)=\phi(\theta_t;\mu_t,\Sigma_t)$, where $\mu_t=e^{-\frac{1}{2}t}\theta_0,\Sigma_t=(1-e^{-t})I_d$, and $\phi(\theta,\mu,\Sigma)$ denotes the probability density function of $\mathcal{N}(\mu,\Sigma)$, then we have
\begin{align*}
	p_t(\theta_t)&=\int_{\mathbb{R}^d}p_{t|0}(\theta_t|\theta_0)p_0(\theta_0)\rmd \theta_0\\
	&=\int_{\mathbb{R}^d}\dfrac{1}{\sqrt{(2\pi)^d|\Sigma_t|}}\exp(-\dfrac{1}{2}(\theta_t-\mu_t)^\top\Sigma_t^{-1}(\theta_t-\mu_t))p_0(\theta_0)\rmd \theta_0\\
	&=\int_{\mathbb{R}^d}\dfrac{1}{\left[2\pi(1-e^{-t})\right]^{d/2}}\exp(-\dfrac{1}{2(1-e^{-t})}\|\theta_t-e^{-\frac{1}{2}t}\theta_0\|^2)p_0(\theta_0)\rmd \theta_0\\
	&=\dfrac{1}{\left[2\pi(1-e^{-t})\right]^{d/2}}\mathbb{E}_{\theta_0\sim p_0}\left[\exp(-\dfrac{1}{2(1-e^{-t})}\|\theta_t-e^{-\frac{1}{2}t}\theta_0\|^2)\right]\,.
\end{align*}
Hence,
\begin{align*}
	\nabla p_t(\theta_t)&=\dfrac{1}{\left[2\pi(1-e^{-t})\right]^{d/2}}\mathbb{E}_{\theta_0\sim p_0}\left[\nabla\left(\exp(-\dfrac{1}{2(1-e^{-t})}\|\theta_t-e^{-\frac{1}{2}t}\theta_0\|^2)\right)\right]\\
	&=\dfrac{1}{\left[2\pi(1-e^{-t})\right]^{d/2}}\mathbb{E}_{\theta_0\sim p_0}\left[\exp(-\dfrac{1}{2(1-e^{-t})}\|\theta_t-e^{-\frac{1}{2}t}\theta_0\|^2)\cdot\dfrac{-(\theta_t-e^{-\frac{1}{2}t}\theta_0)}{1-e^{-t}}\right],\\
	\nabla^2 p_t(\theta_t)&=\dfrac{1}{\left[2\pi(1-e^{-t})\right]^{d/2}}\mathbb{E}_{\theta_0\sim p_0}\bigg[\exp(-\dfrac{1}{2(1-e^{-t})}\|\theta_t-e^{-\frac{1}{2}t}\theta_0\|^2)\\
    &\qquad\qquad\qquad\qquad\qquad\cdot\bigg(\dfrac{(\theta_t-e^{-\frac{1}{2}t}\theta_0)(\theta_t-e^{-\frac{1}{2}t}\theta_0)^\top}{(1-e^{-t})^2}-\dfrac{1}{1-e^{-t}}I_d\bigg)\bigg].
\end{align*}
We can approximate $p_t(\theta_t),\nabla p_t(\theta_t)$ and $\nabla^2 p_t(\theta_t)$ or even higher order derivative tensor of $p_t(\theta_t)$ by Monte Carlo method, therefore, we can compute score function and its high order derivative by
\begin{align*}
	\nabla\log p_t(\theta_t)=\dfrac{\nabla p_t(\theta_t)}{p_t(\theta_t)},\quad \nabla^2\log p_t(\theta_t)=\dfrac{\nabla^2 p_t(\theta_t)}{p_t(\theta_t)}-\dfrac{\nabla p_t(\theta_t)\nabla p_t(\theta_t)^\top}{p_t(\theta_t)^2}\,.
\end{align*}

\subsection{Implementation}
In the numerical studies, we set $T=10$, and the number of Monte Carlo iterations is chosen as the floor of $T/h$, where $h$ varies according to the step size indicated in the figure.

\end{document}